\documentclass[twoside,11pt]{article}
\usepackage[abbrvbib]{jmlr2e}
\usepackage{blindtext}

\usepackage{microtype}
\usepackage{graphicx}
\usepackage{subfigure}
\usepackage{booktabs} % for professional tables
\usepackage[table]{xcolor}

\usepackage{siunitx}  % Add this line to include the siunitx package
\usepackage{multicol}
\usepackage{amsmath}
\usepackage{amsfonts}
\usepackage{amssymb}
\usepackage{eurosym}
\usepackage{bm}
\usepackage{multirow}
\usepackage{soul}
\usepackage{tikz}
\usepackage{mdframed}

\ifdefined\nohyperref\else\ifdefined\hypersetup
  \definecolor{mydarkblue}{rgb}{0,0.08,0.45}
  \hypersetup{ %
    pdftitle={},
    pdfauthor={},
    pdfsubject={},
    pdfkeywords={},
    pdfborder=0 0 0,
    pdfpagemode=UseNone,
    colorlinks=true,
    linkcolor=mydarkblue,
    citecolor=mydarkblue,
    filecolor=mydarkblue,
    urlcolor=mydarkblue,
    pdfview=FitH}

  \ifdefined\isaccepted \else
    \hypersetup{pdfauthor={Anonymous Submission}}
  \fi
\fi\fi

\definecolor{mycyan}{RGB}{230, 255, 255}
\colorlet{mycyanopacity}{mycyan!20}

\surroundwithmdframed[
  backgroundcolor=mycyan,
  linecolor=black,
  innerleftmargin=10pt,
  innerrightmargin=10pt,
  innertopmargin=6pt,
  innerbottommargin=6pt
]{proposition}
\surroundwithmdframed[
  backgroundcolor=mycyan,
  linecolor=black,
  innerleftmargin=10pt,
  innerrightmargin=10pt,
  innertopmargin=6pt,
  innerbottommargin=6pt
]{definition}
\surroundwithmdframed[
  backgroundcolor=mycyan,
  linecolor=black,
  innerleftmargin=10pt,
  innerrightmargin=10pt,
  innertopmargin=6pt,
  innerbottommargin=6pt
]{theorem}
\surroundwithmdframed[
  backgroundcolor=mycyan,
  linecolor=black,
  innerleftmargin=10pt,
  innerrightmargin=10pt,
  innertopmargin=6pt,
  innerbottommargin=6pt
]{example}
\usepackage{algorithm}
\usepackage[noend]{algpseudocode}
\usepackage{enumitem}
\setlist[itemize,enumerate]{topsep=4pt,itemsep=0pt,leftmargin=*}

\usepackage{comment}

\usepackage{mdframed}
\definecolor{theoremcolor}{rgb}%{0.94, 0.94, 0.94}
{0.97, 0.97, 0.9}
\definecolor{examplecolor}{rgb}{1, 1, 1.0}
\mdfsetup{
    backgroundcolor=theoremcolor,
    linewidth=0pt,
}

\DeclareMathOperator*{\argmax}{argmax}
\DeclareMathOperator*{\argmin}{argmin}
\usepackage{lastpage}
\jmlrheading{26}{2025}{1-\pageref{LastPage}}{11/24; Revised
6/25}{10/25}{24-1961}{Saul Santos, Vlad Niculae, Daniel McNamee, André F.T. Martins}

\usepackage[capitalize,noabbrev]{cleveref}
\ShortHeadings{Hopfield-Fenchel-Young Networks}{Santos, Niculae, McNamee, Martins}
\firstpageno{1}

\begin{document}

\title{Hopfield-Fenchel-Young Networks:\\A Unified Framework for Associative Memory Retrieval}

\author{\name Saul Santos \email saul.r.santos@tecnico.ulisboa.pt \\
       \addr Instituto Superior Técnico, Universidade de Lisboa, Lisbon, Portugal\\
       Instituto de Telecomunicações, Lisbon, Portugal
       \AND
       \name Vlad Niculae\email v.niculae@uva.nl \\
       \addr Language Technology Lab,
       University of Amsterdam, The Netherlands
        \AND
       \name  Daniel McNamee \email daniel.mcnamee@research.fchampalimaud.org \\
       \addr Champalimaud Research,
       Lisbon, Portugal
        \AND
       \name André F. T. Martins \email andre.t.martins@tecnico.ulisboa.pt \\
       \addr Instituto Superior Técnico, Universidade de Lisboa, Lisbon, Portugal\\
       Instituto de Telecomunicações, Lisbon, Portugal\\
       ELLIS Unit Lisbon}

\editor{Mohammad Emtiyaz Khan}

\maketitle

\begin{abstract}
%Modern Hopfield networks have enjoyed recent interest due to their 
%connection to
%attention in transformers. 
%Our paper provides a unified framework for sparse Hopfield networks by establishing a link with Fenchel-Young losses. 
%The result is a new family of Hopfield-Fenchel-Young energies whose update rules are end-to-end differentiable \textit{sparse} transformations. 
%We reveal a connection between loss margins, sparsity, and \textit{exact} memory retrieval. 
%We further extend this framework to \textit{structured} Hopfield networks via the SparseMAP transformation, which can retrieve pattern associations instead of a single pattern. Our proposed approach enables the application of post-transformations to the generated Hopfield query, such as normalization and layer normalization, where the latter further bridges the gap with transformers. Inspired by neuroscience \andre{this doesn't sound very compelling}, we further derive algorithms \andre{I think it's not so much ``deriving algorithms'' but more applying and evaluating our Hopfield networks in these recall problems} for modeling memory recall paradigms \andre{add more details, e.g., types of recall}.
%Experiments on multiple instance learning and text rationalization demonstrate the usefulness of our approach \andre{we need to revamp this abstract a bit to differentiate more from the ICML paper}. 
Associative memory models, such as Hopfield networks and their modern variants, have garnered renewed interest due to advancements in memory capacity and connections with self-attention in transformers. 
In this work, we introduce a unified framework---\textit{Hopfield-Fenchel-Young networks}---which generalizes these models to a broader family of energy functions. 
Our energies are formulated as the difference between two Fenchel-Young losses: one, parameterized by a generalized entropy, defines the Hopfield scoring mechanism, while the other applies a post-transformation to the Hopfield output. 
By utilizing Tsallis and norm entropies, we derive end-to-end differentiable update rules that enable sparse transformations, uncovering new connections between loss margins, sparsity, and exact retrieval of single memory patterns. 
We further extend this framework to \textit{structured} Hopfield networks using the SparseMAP transformation, allowing the retrieval of  pattern associations rather than a single pattern. 
%By using specific combinations of Fenchel-Young losses and key results in convex duality we not only recover several well-known methods such as classic Hopfield networks and dense associative models but also enable the application of post-transformations such as $\ell_2$-normalization and layer normalization. 
Our framework unifies and extends traditional and modern Hopfield networks and provides an energy minimization perspective for widely used post-transformations like $\ell_2$-normalization and layer normalization---all through suitable choices of Fenchel-Young losses and by using convex analysis as a building block. 
Finally, we validate our Hopfield-Fenchel-Young networks on diverse memory recall tasks, including free and sequential recall. Experiments on simulated data, image retrieval, multiple instance learning, and text rationalization demonstrate the effectiveness of our approach.
\end{abstract}
\section{Introduction}
\label{sec:intro}
Hopfield networks are biologically plausible neural networks with associative memory capabilities \citep{amari1972learning,nakano1972associatron,hopfield1982neural}. Their attractor dynamics, which describe how the networks evolve toward stable states or patterns, make them suitable for modeling associative memory retrieval in humans and animals \citep{tyulmankov2021biological,whittington2021relating}. The limited storage capacity of classical Hopfield networks was recently overcome through the proposal of new energy functions. These energies were initially proposed for dense associative models by \citet{krotov2016dense} and \citet{demircigil2017model} and later expanded to continuous states by \cite{ramsauer2020hopfield}, resulting in exponential storage capacity and renewed interest in "modern" Hopfield networks. In particular, \citet{ramsauer2020hopfield} revealed connections to attention layers in transformers via an update rule linked to the convex-concave procedure (CCCP; \citealt{yuille2003concave}). However, this model only \textit{approximates} stored patterns, requiring careful temperature tuning to avoid converging to large metastable states that mix multiple input patterns. %highlighting the need for a more principled formulation of energy functions and update rules that guarantees exact retrieval and stable dynamics.
%The model's propensity to converge to large metastable states can be advantageous \andre{this contradicts the previous sentence}, especially in the presence of noise or inherent variability due to its capacity to represent multiple plausible patterns. In scenarios with uncertain or ambiguous data, this behavior enables the model to capture diverse interpretations, enhancing its resilience to noise and its capacity to generate contextually relevant and varied solutions.

A common element in many recent studies on Hopfield networks is connecting an energy function with a desired update rule, typically by modeling the network's temporal dynamics through differential equations, which are discretized to produce updates involving derivatives of Lagrangian terms in the energy function \citep{krotov2021large}. However, a comprehensive framework for formulating the equations that govern the dynamics of the entire spectrum of Hopfield networks is lacking—a gap we aim to fill by explicitly using convex analysis and Fenchel-Young duality as building blocks \citep{rockafellar1970convex,bauschke2017convex}.
This not only allows for designing new energy functions but also provides a way to understand functionalities in transformer architectures, like multi-head attention \citep{attention} and layer normalization \citep{ba2016layer}, as well as other normalization techniques \citep{nguyen2019transformers}.
Developing a generalized framework is crucial for a unified theoretical basis to understand and extend Hopfield networks. This unified view facilitates comparative analysis, reveals underlying principles that govern their dynamics, and potentially leads to improvements in associative memory design and functionality.%while also bridging classical energy-based models with modern deep learning architectures to enable expressive associative memory systems.

\paragraph{Main contributions.} The starting point of our paper is establishing a connection between Hopfield energies and \textbf{Fenchel-Young losses} \citep{blondel2020learning}. Namely, we consider energy functions expressed as the difference of two Fenchel-Young loss terms induced by convex functions $\Omega$ and $\Psi$. These two terms serve  distinct purposes: $\Omega$ contributes to the Hopfield scoring function, where the aim is to ``separate'' one pattern from the others, while $\Psi$ acts as a regularizer, relating to the post-transformation applied in the Hopfield updates. Our proposed Hopfield-Fenchel-Young energies recover as particular cases a wide range of associative memory models, such as the classical binary and continuous Hopfield networks \citep{hopfield1982neural}, polynomial and exponential dense associative memories \citep{krotov2016dense, demircigil2017model}, the modern Hopfield networks from \citet{ramsauer2020hopfield}, as well as their sparse counterparts \citep{hu2023sparse}. 
Furthermore, we show that the Fenchel-Young loss associated with the Hopfield scoring function, when induced by certain generalized entropy functions $\Omega$, can lead to \textbf{sparse update rules} which include as particular cases 
$\alpha$-entmax \citep{peters2019sparse}, $\gamma$-normmax \citep{blondel2020learning},   %used by adaptively sparse transformers 
%\citep{peters2019sparse,correia2019adaptively}
and {SparseMAP} \citep{niculae2018sparsemap}. The latter case allows general structural constraints to be incorporated in addition to sparsity, enabling the retrieval of \textbf{pattern associations}. We illustrate this with structural constraints which  %is applied to 
ensure the retrieval of $k$ patterns, as well as a sequential variant which promotes the $k$ patterns to be contiguous in a memorized sequence. 
One  distinctive property of our sparsity-inducing generalized entropies compared to the Hopfield layers of \citet{ramsauer2020hopfield} is that our resulting update rules pave the way for \textbf{exact retrieval}, leading to the emergence of sparse association among patterns while ensuring end-to-end differentiability, a property which relates to the existence of a separation margin in the corresponding Fenchel-Young losses. 
In addition, our formulation allows for post-transformations in the Hopfield updates via $\Psi$, such as different kinds of normalization, which may accelerate convergence and have other beneficial properties. 
While some of these post-transformations have been considered before \citep{krotov2021large}, their contribution to the energy has often been left implicit. We derive in this paper the explicit energy terms, using classical results from convex duality.  

Our endeavour aligns with the strong neurobiological motivation to seek new Hopfield energies capable of \textbf{sparse} and \textbf{structured} retrieval. 
Indeed, sparse neural activity patterns forming structured representations underpin core principles of cortical computation \citep{SimoncelliOlhausen2001,Tse2007,Palm2013}.
% This endeavour aligns with the strong neurobiological motivation to seek new Hopfield energies capable of \textbf{sparse} and \textbf{structured} retrieval.
% Sparse neural activity patterns
% % are observed in electrophysiological recordings from many brain areas across a variety of animal species and 
% form a core principle of cortical computation %due to their 
% %efficient coding properties 
% \citep{SimoncelliOlhausen2001,Palm2013}.
With respect to memory formation circuits, the sparse firing of neurons in the dentate gyrus, a distinguished region within the hippocampal network, underpins its proposed role in pattern separation during memory storage \citep{Yassa2011,Severa2017}. Evidence suggests that such sparsified %population 
activity aids in minimizing interference,  %between competing memory patterns 
% just prior to pattern completion via autoassociative dynamics downstream
however an integrative theoretical account linking sparse coding and attractor network functionality to clarify these empirical observations is lacking
\citep{Leutgeb2007a,Neunuebel2014}.

To sum up, our main contributions are:
\begin{itemize}
    \item We introduce \textbf{Hopfield-Fenchel-Young} energy functions as a  generalization of modern and classical Hopfield networks (\S\ref{sec:HFYE}). 
    \item We leverage  properties of Fenchel-Young losses which relate \textbf{sparsity} to \textbf{margins}, obtaining new theoretical results for %about the ability of sparse Hopfield networks to have 
    \text{exact} memory retrieval and proving exponential storage capacity in a stricter sense than prior work (\S\ref{sec:probabilistic}). 
    \item We propose  new \textbf{structured} Hopfield networks via the SparseMAP transformation,  which return \textbf{pattern associations} instead of single patterns. We show that SparseMAP has a structured margin, enabling exact retrieval of pattern associations (\S\ref{sec:structured}). 
    \item We use our framework in memory retrieval modeling problems (\S\ref{sec:memory_retrieval_modeling}).
\end{itemize}
Experiments on synthetic and real-world tasks (image retrieval, multiple instance learning, and text rationalization) showcase the usefulness of our proposed models using various kinds of sparse and structured transformations (\S\ref{sec:experiments}). An overview of the Hopfield scoring functions studied in this paper is shown in Figure~\ref{fig:overview}.%
\footnote{Our code is made available on \url{https://github.com/deep-spin/HFYN}.} %

\paragraph{Previous Paper.} Our work builds upon a previously published conference paper \citep{santos2024sparse} and a non-archival workshop paper \citep{martins2023sparse}, which we extend significantly in several ways. \citet{santos2024sparse} fix $\Psi(\cdot) = \frac{1}{2}\|\cdot\|^2$ and focus on $\Omega$ corresponding to sparsity-inducing generalized entropies. The current paper examines general $\Psi$, leading to more general Hopfield energies which are a difference of two Fenchel-Young losses; this step allows this construction to be unified with many other modern and classical Hopfield networks and enables the inclusion of a post-transformation step, such as $\ell_2$ and layer normalization (for which we derive, in \S\ref{sec:normalization}, an explicit energy minimization interpretation). We show how our expanded framework enables the creation of useful new Hopfield networks, with any arbitrary post-transformation defined by a Fenchel-Young loss induced by any convex function. Our theoretical proofs are extended to support this generalization, and we empirically evaluate these variants in multiple instance learning and memory retrieval benchmarks.
Overall, \S\ref{sec:HFYE} and \S\ref{sec:memory_retrieval_modeling} are completely new,  \S\ref{sec:experiments} contains many new experiments, and \S\ref{sec:probabilistic} and \S\ref{sec:structured} have new proofs due to the inclusion of $\Psi$. 

%Additionally, using our Hopfield-Fenchel-Young networks and some adaptations we model memory recall paradigms such as free and sequential recall, also not convered in the previous paper. 
%Finally, we conduct several additional  experiments in memory recall problems, such as free and sequential recall, also not covered in the previous paper, by experimenting with new constrained and penalized sparse transformations. 
%Finally, we conduct several additional  experiments in image retrieval, multiple instance learning and memory retrieval modeling.

\paragraph{Notation.} We denote by $\triangle_N$ the $(N-1)\textsuperscript{th}$-dimensional probability simplex, $\triangle_N := \{\bm{p} \in \mathbb{R}^N : \bm{p} \ge \mathbf{0}, \, \mathbf{1}^\top \bm{p} = 1\}$. 
The convex hull of a set $\mathcal{Y} \subseteq \mathbb{R}^M$ is $\mathrm{conv}(\mathcal{Y}) := \{\sum_{i=1}^N p_i \bm{y}_i \,:\, \bm{p} \in \triangle_N, \, \bm{y}_1, ..., \bm{y}_N \in \mathcal{Y}, \, N \in \mathbb{N}\}$. 
We have $\triangle_N = \mathrm{conv}(\{\bm{e}_1, ..., \bm{e}_N\})$, where $\bm{e}_i \in \mathbb{R}^N$ is the $i$\textsuperscript{th} basis (one-hot) vector. 
Given a convex function $\Omega: \mathbb{R}^N \rightarrow \bar{\mathbb{R}}$, where $\bar{\mathbb{R}} = {\mathbb{R}} \cup \{+\infty\}$, we denote its domain by $\mathrm{dom}(\Omega) := \{\bm{y} \in \mathbb{R}^N \,:\, \Omega(\bm{y}) < +\infty\}$ and its Fenchel conjugate by $\Omega^*(\bm{\theta}) = \sup_{\bm{y} \in \mathbb{R}^N} \bm{y}^\top \bm{\theta} - \Omega(\bm{y})$. 
We denote by $I_\mathcal{C}$ the indicator function of a convex set $\mathcal{C}$, defined as $I_\mathcal{C}(\bm{y}) = 0$ if $\bm{y} \in \mathcal{C}$, and $I_\mathcal{C}(\bm{y}) = +\infty$ otherwise. 
\begin{figure*}[t]
    \centering
    \includegraphics[width=0.7\textwidth]{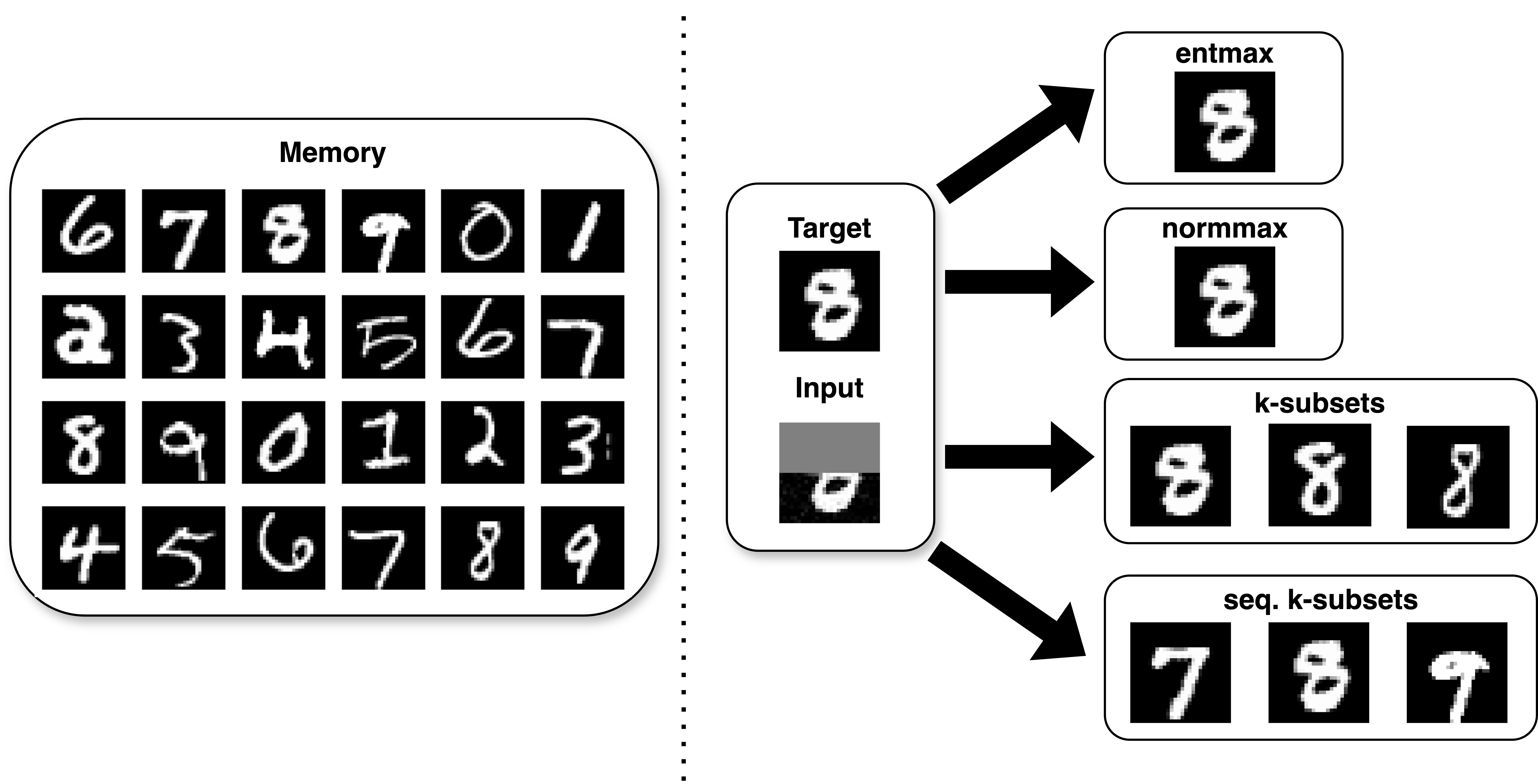}
    \caption{Overview of Hopfield scoring functions: sparse transformations (entmax and normmax) aim to retrieve the closest pattern to the query, and they have exact retrieval guarantees. Structured variants find pattern associations. The $k$-subsets transformation favors a mixture of the top-$k$ patterns, and sequential $k$-subsets favors contiguous retrieval.}
    \label{fig:overview}
\end{figure*}

\subsection*{Table of Contents}

\S\ref{sec:intro} \smallskip Introduction; \\
\S\ref{sec:background} \smallskip Background; \\
\S\ref{sec:HFYE} \smallskip Hopfield-Fenchel-Young Energies; \\
\S\ref{sec:probabilistic} \smallskip Sparse Hopfield Networks; \\
\S\ref{sec:structured} \smallskip Structured Hopfield Networks; \\
\S\ref{sec:memory_retrieval_modeling} \smallskip Mechanics of Memory Retrieval Modeling; \\
\S\ref{sec:experiments} \smallskip Experiments; \\
\S\ref{sec:related_work} \smallskip Related Work; \\
\S\ref{sec:conclusions} \smallskip Conclusions. \\

\section{Background}\label{sec:background}
\subsection{Associative Memories and Hopfield Networks}

In associative memories, data patterns are retrieved  based on similarity to a given query, rather than through an explicit address. When a noisy cue of the memories is provided as the query, the aim is to retrieve the most similar memory pattern. Hopfield networks \citep{amari1972learning,nakano1972associatron,hopfield1982neural} are neural models inspired by statistical physics, specifically by the Ising model  \citep{ising1925beitrag}, which describes a system of magnetic moments or ``spins" of particles that can be in one of two states $(\pm 1)$, interacting with each other to minimize the system's overall energy. In a classical Hopfield network \citep{hopfield1982neural}, states are binary and interact through synaptic weights analogous to the spin-spin couplings in the Ising model. When a query is given as input, the network iteratively adjusts to minimize its energy, leading to the emergence of stable states or attractors, which correspond to stored memory patterns. Pioneering works on classic Hopfield networks, as well as subsequent research by \citet{amit1985spin} and \citet{hertz1991introduction}, has demonstrated the ability of this approach to perform tasks like pattern retrieval and completion for binary data. 

%Recent advancements have extended the classic Hopfield networks to continuous states with significantly increased memory capacity, leading to modern Hopfield networks. To define these we can consider
Formally, let $\bm{X} \in \mathbb{R}^{N \times D}$ be a matrix whose rows hold a set of examples $\bm{x}_1, ..., \bm{x}_N \in \mathbb{R}^D$ (``memory patterns''), and let $\bm{q}^{(0)} \in \mathbb{R}^D$ be the query vector (or ``state pattern'').
%(called a ``state pattern'').. 
Hopfield networks iteratively
%builds
update $\bm{q}^{(t)} \mapsto \bm{q}^{(t+1)}$ for $t\in \{0, 1, ...\}$ according to a certain rule, eventually converging to a fixed point attractor state $\bm{q}^*$, hopefully corresponding
to one of the memorized examples.
This update rule corresponds to the minimization of an energy function, %of the form $$E(\bm{q}) = -\sum_{i=1}^N F(\bm{q}^\top \bm{x}_i),$$  where $F: \mathbb{R} \rightarrow \mathbb{R}$ is a function.
which for a classical binary Hopfield network \citep{hopfield1982neural} takes the form %use $F(u) = \frac{1}{2}\|u\|^2$, so the energy function can be written as
\begin{align}\label{eq:energy_hopfield_classic}
E(\bm{q}) = -\frac{1}{2} \| \bm{X} \bm{q}\|^2 =  -\frac{1}{2} \bm{q}^\top \bm{W} \bm{q},
\end{align}
where  $\bm{W} = \bm{X}^\top\bm{X} \in \mathbb{R}^{D \times D}$ is a weight matrix and 
$\bm{q} \in \{\pm 1\}^D$ is a binary vector, leading to the update rule $\bm{q}^{(t+1)} = \mathrm{sign}(\bm{W}\bm{q}^{(t)})$. 
In this classical construction, the stored patterns, as well as the queries, are assumed to be binary vectors in $\{\pm 1\}^D$. 
The discreteness of the update rule (due to the sign transformation)  makes this network capable of \textbf{exact retrieval} (\textit{i.e.}, it is able to retrieve memorized patterns perfectly, upon convergence). However, its main limitation is that it has only $N=\mathcal{O}(D)$ memory storage capacity. When this capacity is exceeded, the patterns start to interfere \citep{amit1985storing,mceliece1987capacity}, resulting in the retrieval of metastable states or spurious attractor points. % which do not correspond to memorized patterns.  

\subsection{Modern Hopfield Networks}

More recent work 
%% VN prsent tense
%sidestepped 
has sidestepped
the limitation above through alternative energy functions \citep{krotov2016dense,demircigil2017model}, paving the way for a class of models known as ``modern Hopfield networks'' with superlinear (often exponential) memory capacity. In \citet{ramsauer2020hopfield}, $\bm{q} \in \mathbb{R}^D$ is continuous and the following energy is used:
%\begin{align}
%    \label{eq:energy_hopfield}
%        E(\bm{q}) = -\mathrm{lse}(\beta, \bm{X}\bm{q}) + \frac{1}{2} \|\bm{q}\|^2  + \beta^{-1} \log N + \frac{1}{2}M^2,
%\end{align}
\begin{align}
    \label{eq:energy_hopfield}
        E(\bm{q}) = -\frac{1}{\beta} \log\sum_{i=1}^N \exp(\beta \bm{x}_i^\top \bm{q}) + \frac{1}{2} \|\bm{q}\|^2 + \mathrm{const.}
\end{align}
%where $M = \max_i \|\bm{x}_i\|$ and $\mathrm{lse}(\beta, \bm{\theta}) = \beta^{-1} \log\sum_{i=1}^N \exp(\beta \theta_i)$.

\citet{ramsauer2020hopfield} revealed an interesting relation between the updates in this modern Hopfield network and the attention layers in transformers. 
Namely, the minimization of the energy \eqref{eq:energy_hopfield} using the concave-convex procedure (CCCP; \citealt{yuille2003concave}), also known as difference-of-convex (DC) programming \citep{horst1999dc}, leads to the update rule
\begin{equation}\label{eq:update_rule}
    \begin{aligned}
        \bm{q}^{(t+1)} &= \bm{X}^\top \mathrm{softmax}(\beta \bm{X} \bm{q}^{(t)}).
    \end{aligned}
\end{equation}
When $\beta = \frac{1}{\sqrt{D}}$, each update matches
the computation performed in the attention layer of a transformer with a single attention head and identity projection matrices \citep[\S2]{ramsauer2020hopfield}.
This triggered interest in developing variants of Hopfield layers which can be used as drop-in replacements for multi-head attention layers \citep{hoover2023energy}.

While \citet{ramsauer2020hopfield} derived useful theoretical properties of these networks (including their exponential storage capacity under a relaxed notion of retrieval), the use of
%a softmax transformation
softmax
in %the update rule
\eqref{eq:update_rule}, along to the fact that these networks now operate on a continuous space, makes retrieval only approximate (\textit{i.e.}, the attractors are not the exact stored patterns, but only approximately close as $\beta \rightarrow \infty$), with some propensity for undesirable metastable states (states that mix multiple memory patterns). 
We overcome these
drawbacks in our work by showing that it is possible to work on a continuous space but still obtain exact retrieval (as in binary Hopfield networks), without sacrificing exponential storage capacity. 
We generalize the energies \eqref{eq:energy_hopfield_classic} and \eqref{eq:energy_hopfield}, as well as several other proposed formulations of Hopfield-like models, and we provide a unified treatment of sparse and structured
Hopfield networks along with a theoretical and empirical analysis.

\subsection{Regularized Argmax and Fenchel-Young Losses}
\label{sec:FYL}
Our construction and results follow from the properties of regularized argmax functions and Fenchel-Young losses \citep{blondel2020learning}, which we next review. 

%Let $\triangle_N := \{\bm{p} \in \mathbb{R}^N \,:\, \bm{p} \ge \mathbf{0}, \, \mathbf{1}^\top \bm{p} = 1\}$ denote the probability simplex, whose elements are probability vectors of length $N$.
Given a strictly convex function $\Omega: \mathbb{R}^N \rightarrow \bar{\mathbb{R}}$ with compact domain $\mathrm{dom}(\Omega)$,
the 
$\Omega$-regularized argmax transformation \citep{niculae2017regularized}, 
%\textbf{$\Omega$-regularized prediction map} ($\Omega$-RPM; \citealt{niculae2017regularized}), 
$\hat{\bm{y}}_\Omega : \mathbb{R}^N \rightarrow \mathrm{dom}(\Omega)$ is:
\begin{align}\label{eq:rpm}
    \hat{\bm{y}}_\Omega(\bm{\theta}) := \nabla \Omega^*(\bm{\theta}) = \argmax_{\bm{y} \in \mathrm{dom}(\Omega)} \bm{\theta}^\top \bm{y} - \Omega(\bm{y}),
\end{align} 
where $\nabla \Omega^*(\bm{\theta})$ denotes the gradient of the Fenchel conjugate $\Omega^*$ at $\bm{\theta}$.%
\footnote{Note that $\Omega^*$ is guaranteed to be everywhere differentiable under the assumption that $\mathrm{dom}(\Omega)$ is compact and $\Omega$ is strictly convex, which together ensure existence and uniqueness of a maximum in \eqref{eq:rpm}.} %
A trivial example of such a regularized argmax function \eqref{eq:rpm} is obtained when $\Omega(\bm{y}) = \frac{1}{2}\|\bm{y}\|^2$ with $\mathrm{dom}(\Omega) = \mathbb{R}^N$, which leads to the identity map $\bm{y}_\Omega(\bm{\theta}) = \bm{\theta}$. 
In general, we are interested in cases where $\mathrm{dom}(\Omega) \subsetneq \mathbb{R}^N$, for example the probability simplex $\mathrm{dom}(\Omega) = \triangle_N$ (studied in \S\ref{sec:probabilistic}) or a polytope (studied in \S\ref{sec:structured}). 
A famous instance of the former is the \textbf{softmax} transformation, obtained when the regularizer is the Shannon negentropy, $\Omega(\bm{y}) = \sum_{i=1}^N y_i \log y_i + I_{\triangle_N}(\bm{y})$. % (negative entropy).
Another 
%example
instance
is the \textbf{sparsemax} transformation, obtained when $\Omega(\bm{y}) = \frac{1}{2}\|\bm{y}\|^2 + I_{\triangle_N}(\bm{y})$ \citep{martins2016softmax}, and which corresponds to 
the Euclidean projection onto the probability simplex. 
In \S\ref{sec:probabilistic}, we analyze a wider set of transformations induced by generalized entropies which include these as particular cases. 

The \textbf{Fenchel-Young loss} induced by $\Omega$ 
\citep{blondel2020learning} is the function defined as
\begin{align}\label{eq:fy_loss}
    L_{\Omega}(\bm{\theta}, \bm{y}) := \Omega(\bm{y}) + \Omega^*(\bm{\theta}) - \bm{\theta}^\top \bm{y}.
\end{align}
In the trivial case above, where $\Omega(\bm{y}) = \frac{1}{2}\|\bm{y}\|^2$ with $\mathrm{dom}(\Omega) = \mathbb{R}^N$, we obtain $\Omega^*(\bm{\theta}) = \frac{1}{2}\|\bm{\theta}\|^2$, leading to the squared loss, $L_\Omega(\bm{\theta}, \bm{y}) = \frac{1}{2}\|\bm{\theta} - \bm{y}\|^2$. 
When $\Omega$ is the Shannon's negentropy defined above, 
we have $\Omega^*(\bm{\theta}) = \log\sum_{i=1}^N\exp(\theta_i)$,
%(see \eqref{eq:lse}),
and $L_\Omega$ is the \textbf{cross-entropy loss}, up to a constant independent of $\bm{\theta}$.  %\citep[\S 3.2]{blondel2020learning}.
Intuitively, Fenchel-Young losses quantify how ``compatible'' a score vector $\bm{\theta} \in \mathbb{R}^N$ (\textit{e.g.}, logits) is to a desired target $\bm{y} \in \mathrm{dom}(\Omega)$ (\textit{e.g.}, a probability vector). 
Fenchel-Young losses have important and useful properties, summarized below: 
\begin{proposition}[Properties of Fenchel-Young losses]
\label{prop:properties} %\citep[Prop.~2]{blondel2020learning}
Fenchel-Young losses $L_\Omega(\bm{\theta}, \bm{y})$ satisfy the following properties:
\begin{enumerate}
    \item They are non-negative, $L_\Omega(\bm{\theta}, \bm{y}) \ge 0$, with equality if and only if $\bm{y} = \hat{\bm{y}}_\Omega(\bm{\theta})$.
    \item They are convex in $\bm{\theta}$, and their gradient is $\nabla_{\bm{\theta}} L_\Omega(\bm{\theta}, \bm{y}) = -\bm{y} + \hat{\bm{y}}_\Omega(\bm{\theta})$.
\end{enumerate}
\end{proposition}
%\end{enumerate}
%For Tsallis negentropies $\Omega_\alpha$ with $\alpha > 1$, 
%a margin property holds \citep[Prop. 7]{blondel2020learning}:
%we also have the following \textbf{margin property} 
%\begin{align}\label{eq:margin}
%\begin{split}
%\forall i \in [N], \quad
%    L_{\Omega_\alpha}(\bm{\theta}, \bm{e}_i) = 0 &\Longleftrightarrow 
%    \hat{y}_{\Omega_\alpha}(\bm{\theta}) = \bm{e}_i \\
%    &\Longleftrightarrow {\theta}_i - \max_{j \ne i} {\theta}_j \ge (\alpha - 1)^{-1}.
%\end{split}
%\end{align}
%\andre{not sure if we should make this specific to Tsallis or not. Also, we should say somewhere that we will be using a structured version of FY losses later, where $\mathrm{dom}(\Omega)$ is a polytope, more general than the probability simplex.} 
%
%% VN: we can make this read more dramatically. "Another X" is too flat.
%
These properties are stated and proved in \citet[Proposition~2]{blondel2020learning}. 
For certain choices of $\Omega$, Fenchel-Young losses also have an important margin property, which we will elaborate on in \S\ref{sec:probabilistic} and which underpins our proposed Hopfield energies with exact retrieval. 

\section{Hopfield-Fenchel-Young Energies}
\label{sec:HFYE}
%\andre{
%    \begin{itemize}
%    \item The idea is to define here our energy as a difference of two FY losses (more general than in the workshop paper) and point out that it subsumes classical Hopfield, MHNs, sparse MHNs \cite{hu2023sparse}, and generalized sparse MHNs
%    \item Proposition with the CCCP updates, and interpret what they are doing as a pipeline of transformations (refer to \citet{millidge2022universal} here)
%    \item Focus first on FY losses over $\triangle$ and $\Omega$ with margin (including sparse and generalized sparse MHNs). Maybe also norm-entropies?
%    \item Proposition with margin results for that case, exact storage capacity in autoassociative mode, etc. 
%    \end{itemize}
%}

We now use $\Omega$-regularized argmax transformations and Fenchel-Young losses to define a new class of energy functions associated to modern Hopfield networks. 
Historically, the connection between sparse Hopfield energy functions and Fenchel conjugates was first made by \citet{hu2023sparse} and extended by \citet{martins2023sparse}, \citet{wu2023stanhop}, and  \citet{santos2024sparse}.

\subsection{Definition}\label{sec:definition_hfy}

%\begin{definition}[Hopfield-Fenchel-Young energy]
%Let $\Omega: \mathbb{R}^N \rightarrow \bar{\mathbb{R}}$ and $\Psi: \mathbb{R}^D \rightarrow \bar{\mathbb{R}}$ be convex functions, and $\bm{X} \in \mathbb{R}^{N \times D}$ be the matrix of memory patterns. 
%We define the generalized \textbf{Hopfield-Fenchel-Young energy} $E_{\Omega, \Psi}(\bm{q}; \bm{X})$ as a difference of two Fenchel-Young losses, up to a constant (on $\bm{q}$), namely 
%\begin{align}\label{eq:general_hfy_energy}
%    E_{\Omega, \Psi}(\bm{q}; \bm{X}) &= -\Omega^*(\bm{X}\bm{q}) + \Psi(\bm{q}) + \mathrm{constant}\nonumber\\
%    &= \underbrace{-L_\Omega(\bm{X}\bm{q}, \bm{u})}_{E_{\mathrm{concave}}(\bm{q})} + \underbrace{L_\Psi(\bm{X}^\top \bm{u}, \bm{q})}_{E_{\mathrm{convex}}(\bm{q})} + \mathrm{constant},
%\end{align}
%where $\bm{u} \in \mathrm{dom}(\Omega) \subseteq \mathbb{R}^N$ is arbitrary.\footnote{\andre{check this} Later it may be useful to define it as $\bm{u} = (\nabla \Omega^*)(\bm{0}) = \arg\min_{\bm{y}} \Omega(\bm{y})$ when this $\arg\min$ exists.} 
%This energy is defined for $\bm{q} \in \mathrm{dom}(\Psi) \subseteq \mathbb{R}^D$.  
%\end{definition}
%\andre{some things that are a bit annoying: (1) the constant is not the same in both rows; it's a bit weird to define the energy up to a constant. (2) it is not clear from the definition why we need FY losses, e.g. why not defining the energy functions using $\Omega$ and $\Psi$ only. (3) $\bm{u}$ being arbitrary may also be confusing.} 

Let $\Omega: \mathbb{R}^N \rightarrow \bar{\mathbb{R}}$ and $\Psi: \mathbb{R}^D \rightarrow \bar{\mathbb{R}}$ be convex functions, and $\bm{X} \in \mathbb{R}^{N \times D}$ be the matrix of memory patterns. 
We consider energy functions of the form
$E(\bm{q}) = -\Omega^*(\bm{X}\bm{q}) + \Psi(\bm{q})$ (up to a constant), defined for $\bm{q} \in \mathrm{dom}(\Psi) \subseteq \mathbb{R}^D$. These energy functions can be equivalently written as a difference of Fenchel-Young losses (cf.~\eqref{eq:fy_loss}):
\begin{align}\label{eq:general_hfy_energy}
    E(\bm{q}) 
    &= \underbrace{-L_\Omega(\bm{X}\bm{q}, \bm{u})}_{E_{\mathrm{concave}}(\bm{q})} + \underbrace{L_\Psi(\bm{X}^\top \bm{u}, \bm{q})}_{E_{\mathrm{convex}}(\bm{q})} \,\, + \,\,\mathrm{const.},
\end{align}
where $\bm{u} \in \mathrm{dom}(\Omega) \subseteq \mathbb{R}^N$ is an arbitrary baseline.\footnote{The value of $E(\bm{q})$ does not depend on the choice of $\bm{u}$; without loss of generality we define $\bm{u} = (\nabla \Omega^*)(\bm{0}) = \argmin_{\bm{y}} \Omega(\bm{y})$ assuming that this $\argmin$ exists. When $\mathrm{dom}(\Omega) = \triangle_N$ and $\Omega$ is a generalized negentropy---the scenario we study in \S\ref{sec:probabilistic}---this choice leads to $\bm{u} = \mathbf{1}/N$ being a uniform distribution.} %
We call \eqref{eq:general_hfy_energy} a \textbf{Hopfield-Fenchel-Young (HFY) energy}. 
We will see that energies of this form are general enough to encompass most previously proposed variants of Hopfield energies and to motivate new ones. 

The first thing to observe is that HFY energies decompose as the sum of a concave and a convex function of $\bm{q}$: %, since Fenchel-Young losses are always convex with respect to either of its arguments. 
The concavity of $E_{\mathrm{concave}}$ holds from the convexity of Fenchel-Young losses on their first argument, as established in Proposition~\ref{prop:properties}, and from the fact that the composition of a convex function with an affine map is convex \citep[\S 3.2]{boyd2004convex}.
The convexity of $E_{\mathrm{convex}}$ comes from the fact that Fenchel-Young losses are also convex on their second argument when $\Psi$ is strictly convex. 
%This property, combined with the strict convexity of \(\Omega\), is essential for ensuring convergence to a local minimum, as it guarantees that the Fenchel-Young losses are differentiable  \citep[Proposition 2]{blondel2020learning}. 
These two
terms compete when minimizing the energy   \eqref{eq:general_hfy_energy} with respect to $\bm{q}$:
\begin{itemize}
    \item Minimizing $E_{\mathrm{concave}}$ is equivalent to \textit{maximizing} $L_{\Omega}(\bm{X} \bm{q}; \bm{u})$, which pushes for state patterns $\bm{q}$ such that $\bm{X}\bm{q}$ is as far as possible from the baseline $\bm{u}$. We will see later that this will encourage $\bm{q}$ to be close to a single memory pattern.
    \item Minimizing $E_{\mathrm{convex}}$ serves as a regularization, encouraging $\bm{q}$ to stay close to $\bm{X}^\top \bm{u}$. More intuition about this regularization will be provided in \S\ref{sec:hfy_definition}. 
\end{itemize} 

\subsection{Update rule}

The next result, proved in Appendix~\ref{app:general_cccp}, leverages Proposition~\ref{prop:properties} to derive the Hopfield update rule for the energy function \eqref{eq:general_hfy_energy}, generalizing \citet[Theorem A.1]{ramsauer2020hopfield}.

\begin{proposition}[Update rule of HFY energies]\label{prop:cccp}
    Minimizing \eqref{eq:general_hfy_energy} using the CCCP algorithm \citep{yuille2003concave} leads to the updates
\begin{align}\label{eq:updates_general_hfy}
    \bm{q}^{(t+1)} &= %(\nabla \Psi)^{-1} \left( \bm{X}^\top \nabla \Omega^*(\bm{X}\bm{q}^{(t)})\right) %\nonumber\\
    %&= 
    (\nabla \Psi^*) \left( \bm{X}^\top \nabla \Omega^*(\bm{X}\bm{q}^{(t)})\right) %\nonumber\\
    %&= 
    = \hat{\bm{y}}_\Psi \left( \bm{X}^\top \hat{\bm{y}}_\Omega (\bm{X}\bm{q}^{(t)})\right).
\end{align}
%This algorithm will lead to a local minimum of the energy function. \andre{under suitable conditions; not sure we want to be very rigorous here}
\end{proposition} 
Note that the updates \eqref{eq:updates_general_hfy} involve the functions $\Omega$ and $\Psi$ only via the gradient maps $\nabla \Omega^*$ and $\nabla \Psi^*$, which are the regularized prediction functions $\hat{\bm{y}}_{\Omega}$ and $\hat{\bm{y}}_{\Psi}$ (cf.~\ref{eq:rpm}). 
Another way of looking into \eqref{eq:updates_general_hfy} is to unpack these updates in terms of composing four operations: similarity, separation, projection, and post-transformation. This can be expressed by
\begin{equation}
    \bm{q}^{(t+1)} = \underbrace{\text{normalize}}_{\substack{\text{Post-transformation}}} \Big( \underbrace{\bm{X}^\top}_{\text{Projection}} \cdot \underbrace{\text{sep}}_{\substack{\text{Separation}}} \big( \underbrace{\text{sim}(\bm{X}, \bm{q}^{(t)})}_{\substack{\text{Similarity}}} \big) \Big).
\end{equation}
The update rule \eqref{eq:updates_general_hfy} is obtained when the similarity operation is the dot product, the separation corresponds to $\hat{\bm{y}}_\Omega$, and the post-transformation corresponds to $\hat{\bm{y}}_\Psi$. This is consistent with the universal Hopfield network of \citet{millidge2022universal}, where the post-transformation is the identity function. The update rule \eqref{eq:updates_general_hfy} supports only the dot product as similarity function, although other similarities can be introduced as long as the gradient of the similarity function is also projected.

\begin{table}[t!]
\centering
\caption{Properties/examples of convex conjugates \citep[\S 3.3]{boyd2004convex}.}
\renewcommand{\arraystretch}{1.5}
\small
\begin{tabular}{@{}p{4cm}p{10cm}@{}}
\toprule
\textbf{Property} & \textbf{Expression / Description} \\ \midrule
\textbf{Biconjugate} & $\Omega^{**} = \Omega$, \quad \text{if $\Omega$ is closed and convex} \\ %\midrule
%\textbf{Fenchel-Young inequality} & $\Omega^*(\bm{y}) \geq \langle \bm{\theta}, \bm{y} \rangle - \Omega(\bm{y})$ \\ \midrule

%\textbf{Convexity} & $\Omega^*(\bm{\theta})$ is convex \\ \midrule
%\textbf{Subdifferential} & $\bm{\theta} \in \partial \Omega(\bm{y}) \iff \bm{y} \in \partial \Omega^*(\bm{\theta})$ \\ \midrule
%\textbf{Differentiability} & $\nabla \Omega(\bm{y}) = \bm{\theta} \quad \text{and} \quad \nabla \Omega^*(\bm{\theta}) = \bm{y}$ \\ \midrule
\textbf{Linear transformations} & If $\bm{A}$ is squared and non-singular, $(\Omega \circ \bm{A})^* = \Omega^* \circ \bm{A}^{-\top}$ \\ %\midrule
\textbf{Scaling} & For $a > 0$, $(a \Omega)^*(\bm{\theta}) = a \Omega^*\left(\frac{\bm{y}}{a}\right)$ \\ 
%\midrule
\textbf{Translation} & If $\Omega(\bm{y}) = \Psi(\bm{y} - \bm{b})$, then $\Omega^*(\bm{\theta}) = \bm{b}^\top \bm{\theta} + \Psi^*(\bm{\theta})$ \\ 
%\midrule
\textbf{Separable sum} & If $\Omega(\bm{y}) = \sum_i \Omega_i(y_i)$, then $\Omega^*(\bm{\theta}) = \sum_i \Omega^*_i(\theta_i)$ \\ %\midrule
\textbf{Infimal convolution} & $(\Omega_1 \square\, \Omega_2)^*(\bm{\theta}) = \Omega_1^*(\bm{\theta}) + \Omega_2^*(\bm{\theta})$ (see \eqref{eq:infimal_conv}) \\ 
%\midrule
\textbf{Dual norms} & If $\Omega(\bm{y}) = \frac{1}{p}\|\bm{y}\|_p^p$, then $\Omega^*(\bm{\theta}) = \frac{1}{q}\|\bm{\theta}\|_q^q$ with $\frac{1}{p} + \frac{1}{q} = 1$ \\ 
%\midrule
\textbf{Exponential} & If $\Omega(y) = \exp(y)$,  %with $\mathrm{dom}(\Omega) = \mathbb{R}$, 
then $\Omega^*(\theta) = \theta \log \theta - \theta + I_{\mathbb{R}_+}(\theta)$ %with $\mathrm{dom}(\Omega^*) = \mathbb{R}_+$ 
\\ %\midrule
\textbf{Negentropy} & If $\Omega(\bm{y}) = 
\sum_i y_i \log y_i + I_{\triangle^N}(\bm{y})$, then $\Omega^*(\bm{\theta}) = \log \sum_i \exp(\theta_i)$ \\ 
\bottomrule
\end{tabular}
\label{table:convex_conjugate_properties}
\end{table}

\subsection{Particular cases}\label{sec:particular_cases}

We show now that energies of the form \eqref{eq:general_hfy_energy} recover many known variants of Hopfield networks and suggest new ones. 
These variants are obtained through particular choices of the $\Omega$ and $\Psi$ functions, which lead to the corresponding conjugates $\Omega^*$ and $\Psi^*$. 
We start by examining key properties of convex duality that are instrumental for developing Hopfield energies. The convex conjugate, also known as the Legendre-Fenchel transform \citep{Fenchel1949}, is a cornerstone of convex analysis and optimization, allowing to convert complex optimization problems into more manageable dual forms. Table~\ref{table:convex_conjugate_properties} highlights several properties and examples which establish the building blocks for designing new classes of Hopfield networks, as described in the subsequent subsections. Table~\ref{tab:hopfield} provides a summary of the several examples arising from the simple construction presented in \S\ref{sec:definition_hfy} by leveraging the properties in Table~\ref{table:convex_conjugate_properties}. 

\begin{table*}[t]
    \caption{\textbf{Examples of HFY energies and corresponding update rules.} The top rows show classic Hopfield networks and dense associative memories (we show continuous variants; binary variants are a limit case when $\beta^{-1} \rightarrow 0^+$ in $\Psi(\bm{q})$). %We denote by $H_\mathrm{b}(\bm{y}) = -\sum_{i} \left(y_i \log y_i - (1-y_i)\log(1-y_i)\right)$ the Fermi-Dirac entropy (sum of independent binary entropies), and 
    In DAMs, we set $r^{-1} + s^{-1} = 1$ and denote by $\mathrm{spow}(\bm{\theta}, \alpha) = \mathrm{sign}(\bm{\theta}) \odot |\bm{\theta}|^
    \alpha$ the signed power function. 
    The middle rows show modern Hopfield networks (MHNs) associated to probabilistic prediction, where $\mathrm{dom}(\Omega) = \triangle_N$. Our paper focuses on sparse variants of these models (\S\ref{sec:probabilistic}). Finally, the last row generalizes the setups above to structured memories over a structured set $\mathcal{Y}$, also addressed in this paper (\S\ref{sec:structured}).}
    \vspace{.3cm}
    \label{tab:hopfield}
    \centering
    \resizebox{\textwidth}{!}{%
    %\small
    \begin{tabular}{llllll}
    \toprule
        & $\Omega(\bm{y})$ & $\mathrm{dom}(\Omega)$ & $\Psi(\bm{q})$  & $\mathrm{dom}(\Psi)$ & Update rule ($\bm{q}^{(t+1)} = ...$) \\
    \midrule
        Classic HNs  & $\frac{1}{2}\|\bm{y}\|^2$ & $\mathbb{R}^N$ & $-\frac{1}{\beta} H_\mathrm{b}\left(\frac{\mathbf{1}+\bm{q}}{2}\right)$ & $[-1,1]^D$ & $\tanh(\beta \bm{X}^\top \bm{X} \bm{q}^{(t)})$ \\[.1cm]
        \citep{hopfield1982neural} &&&&&\\
        Poly-DAMs  & $s^{-1}\|\bm{y}\|_s^s$ & $\mathbb{R}^N$ & $-\frac{1}{\beta} H_\mathrm{b}\left(\frac{\mathbf{1}+\bm{q}}{2}\right)$ & $[-1,1]^D$ & $\tanh(\beta \bm{X}^\top \mathrm{spow}(\bm{X} \bm{q}^{(t)}, r-1))$ \\[.1cm]
        \citep{krotov2016dense} &&&&&\\
        Exp-DAMs & $\bm{y}^\top (\log \bm{y} - \mathbf{1})_+$ %$\sum_{j=1}^N [y_j \log y_j - y_j]_+$ 
        & $\mathbb{R}_+^N$ & $-\frac{1}{\beta} H_\mathrm{b}\left(\frac{\mathbf{1}+\bm{q}}{2}\right)$ & $[-1,1]^D$ & $\tanh(\beta \bm{X}^\top \exp(\bm{X} \bm{q}^{(t)})$ \\[.1cm]
        \citep{demircigil2017model} &&&&&\\
        \midrule
        MHNs & $\frac{1}{\beta} \bm{y}^\top \log \bm{y}$ %$\beta^{-1}\sum_{i=1}^N y_i \log y_i$ 
        & $\triangle_N$ & $\frac{1}{2}\|\bm{q}\|^2$ & $\mathbb{R}^D$ & $\bm{X}^\top \mathrm{softmax}(\beta \bm{X}\bm{q}^{(t)})$\\[.1cm]
        \citep{ramsauer2020hopfield} &&&&&\\
        Sparse MHNs  & $\frac{1}{\beta} \frac{1}{2}\|\bm{y}\|^2$ & $\triangle_N$ & $\frac{1}{2}\|\bm{q}\|^2$ & $\mathbb{R}^D$ & $\bm{X}^\top \mathrm{sparsemax}(\beta \bm{X}\bm{q}^{(t)})$ \\[.1cm]
        \citep{hu2023sparse} &&&&&\\
        Entmax MHNs & $\frac{1}{\beta} \frac{\|\bm{y}\|^\alpha_\alpha - 1}{\alpha(\alpha-1)}$ & $\triangle_N$ & $\frac{1}{2}\|\bm{q}\|^2$ & $\mathbb{R}^D$ & $\bm{X}^\top \alpha\text{-}\mathrm{entmax}(\beta \bm{X}\bm{q}^{(t)})$ \\[.1cm]
        (\textbf{this work} and \citealt{wu2023stanhop}) &&&&&\\
        Normmax MHNs  & $\frac{1}{\beta} (\|\bm{y}\|_\gamma - 1)$ & $\triangle_N$ & $\frac{1}{2}\|\bm{q}\|^2$ & $\mathbb{R}^D$ & $\bm{X}^\top \gamma\text{-}\mathrm{normmax}(\beta \bm{X}\bm{q}^{(t)})$ \\[.1cm]
        (\textbf{this work}) &&&&&\\
        \midrule
        Structured MHNs & $\frac{1}{\beta} \frac{1}{2}\|\bm{y}\|^2$ & $\mathrm{conv}(\mathcal{Y})$ & $\frac{1}{2}\|\bm{q}\|^2$ & $\mathbb{R}^D$ & $\bm{X}^\top \mathrm{SparseMAP}(\beta \bm{X}\bm{q}^{(t)})$ \\[.1cm]
        (\textbf{this work}) &&&&&\\
    \bottomrule
    \end{tabular}
    }
\end{table*}

\subsubsection{Classical Hopfield networks}

%\paragraph{Classical Hopfield networks.} 
We denote the Fermi-Dirac entropy (sum of independent binary entropies) by 
%\begin{equation}
    $H_\mathrm{b}(\bm{y}) = -\sum_{i=1}^N \left(y_i \log y_i - (1-y_i)\log(1-y_i)\right)$, 
    %\quad \text{
where $\bm{y} \in [0,1]^N$. %}
%\end{equation}
Classical (continuous) Hopfield networks are recovered in \eqref{eq:general_hfy_energy} with
\begin{equation}
    \Omega(\bm{y}) = \frac{1}{2}\|\bm{y}\|^2 \quad \text{and} \quad
\Psi(\bm{q}) = -\beta^{-1} H_\mathrm{b}\left(\frac{\mathbf{1} + \bm{q}}{2}\right) + I_{[-1,1]^D}(\bm{q}),
\end{equation}
where $\beta^{-1}\ge 0$ is a temperature parameter. 
In this case, $\hat{\bm{y}}_\Omega$ is the identity and $\hat{\bm{y}}_\Psi$ is the {\tt tanh} transformation with temperature $\beta^{-1}$, leading to the update rule $\bm{q}^{(t+1)} = \tanh \left(\beta \bm{X}^\top \bm{X}\bm{q}^{(t)}\right)$. Binary Hopfield networks \citep{hopfield1982neural} appear as a limit case when $\beta \rightarrow +\infty$.%
\footnote{We can obtain the same result for the classical binary Hopfield network by letting $\Psi(\bm{q}) = I_{\|.\|_\infty \le 1}(\bm{q})$, in which case $\Psi^*(\bm{z}) = \|\bm{z}\|_1$ and $(\nabla \Psi^*)(\bm{z}) = \mathrm{sign}(\bm{z})$.}

\subsubsection{Dense associative memories (DAMs)}

%\paragraph{Dense associative memories (DAMs).} 

DAMs correspond to energy functions
%\begin{equation}
    $E(\bm{q}) = -\sum_{i=1}^N F(\bm{q}^\top \bm{x}_i)$. 
%\end{equation} 
Poly-DAMs use $F(z) = |z|^r$ and are constrained to $\bm{q} \in \{\pm 1\}^D$ \citep{krotov2016dense}. 
These energies can be equivalently written (up to a scaling factor) as 
\begin{equation}
    E(\bm{q}) = -r^{-1}\|\bm{X}\bm{q}\|_r^r + I_{[-1, 1]^D}(\bm{q}).
\end{equation}
This corresponds to choosing $\Omega(\bm{y}) = -s^{-1}\|\bm{y}\|_s^s$, where $r^{-1} + s^{-1} = 1$ ($\|.\|_r$ and $\|.\|_s$ are dual norms), and $\Psi(\bm{q})$ as in classical Hopfield networks. %$ = I_{[-1, 1]^D}(\bm{q})$. 
We have $(\nabla \Omega^*)(\bm{\theta}) = \mathrm{sign}(\bm{\theta})|\bm{\theta}|^{r-1} =: \mathrm{spow}(\bm{\theta}, r-1)$, where {\tt spow} denotes the signed power function. 
The update rule is
\begin{equation}
    \bm{q}^{(t+1)} = \tanh\left(\beta \bm{X}^\top \mathrm{spow}(\bm{X}\bm{q}^{(t)}, r-1) \right).
\end{equation} 
Likewise, the Exp-DAM of \citet{demircigil2017model}, which uses $F(z) = \exp(z)$, can be written as $E(\bm{q}) = -\mathbf{1}^\top \exp(\bm{X}\bm{q}) + I_{[-1,1]^D}(\bm{q})$. It corresponds to choosing
\begin{equation}
    \Omega(\bm{y}) = \sum_{j=1}^D [y_j \log y_j - y_j]_+ + I_{\mathbb{R}_+^N}(\bm{y}) \quad \text{and} \quad
\Psi(\bm{q}) = \beta^{-1} H_\mathrm{b}\left(\frac{1 + \bm{q}}{2}\right) + I_{[-1,1]^D}(\bm{q}).
\end{equation}
We have $(\nabla \Omega^*)(\bm{\theta}) = \exp(\bm{\theta})$. The update rule is $\bm{q}^{(t+1)} = \tanh\left(\beta \bm{X}^\top \exp(\bm{X}\bm{q}^{(t)}) \right)$. 

\subsubsection{Modern Hopfield networks (MHNs)} 

The energy \eqref{eq:energy_hopfield} of \citet{ramsauer2020hopfield} is recovered when $\Omega(\bm{y}) = \beta^{-1} \sum_{i=1}^N y_i \log y_i + I_{\triangle_N}(\bm{y})$, the Shannon negentropy with temperature $\beta^{-1}$, and when $\Psi(\bm{q}) = \frac{1}{2}\|\bm{q}\|^2$. 
In this case, $\hat{\bm{y}}_\Psi$ is the identity and $\hat{\bm{y}}_\Omega$ is the softmax transformation with temperature $\beta^{-1}$, leading to the update rule 
\begin{equation}
    \bm{q}^{(t+1)} = \bm{X}^\top \mathrm{softmax}(\beta \bm{X}\bm{q}^{(t)}).
\end{equation} 
The sparse MHNs of \citet{hu2023sparse} modify $\Omega$ to $\Omega(\bm{y}) = \frac{1}{2}\|\bm{y}\|^2 + I_{\triangle_N}(\bm{y})$, which leads to similar updates but where softmax is replaced by sparsemax. 
We present in \S\ref{sec:probabilistic} a generalization using the Tsallis $\alpha$-negentropy, which is also sparse for $\alpha>1$ and was proposed independently by \citet{martins2023sparse} and \citet{wu2023stanhop}. We also study the $\gamma$-normmax negentropy, proposed by \citet{santos2024sparse}, which also leads to sparse MHNs. 

\subsubsection{Hetero-associative memories} 

The convex function \(\Psi (\bm{q})\) can be used to induce outer projection operations, resulting in hetero-associativity through a matrix \(\bm{A}\), similar to the hetero-associative memories described by \citet{millidge2022universal} and  \citet{wu2024uniform}. To do that, \(\Psi (\bm{q})\) is defined as \(\Psi (\bm{q}) = \frac{1}{2} \bm{q}^\top \bm{A}^{-1} \bm{q}\), where $\bm{A}$ is a symmetric and positive definite matrix, which leads to the gradient map \((\nabla \Psi^*)(\bm{z}) = \bm{A} \bm{z} \). When \(\bm{z} = \bm{X}^\top \hat{\bm{y}}_\Omega(\bm{X}\bm{q})\), we are effectively using the matrix \(\bm{A}\) to project the associative space into a hetero-associative memory. %The matrix \(\bm{A}\) must be symmetric and full rank to ensure it is positive definite, which guarantees the invertibility of \(\bm{A}\) and the convexity of the function \(\Psi(\bm{q})\). 
This approach follows closely the outer projection considered in \citet{ramsauer2020hopfield}'s Hopfield layers. %However, the constraint on \(\bm{A}\) in our case ensures that \(\Psi(\bm{q})\) remains convex and $\bm{A}$ invertible.

\subsubsection{Structured energies} 

The framework of MHNs can be extended to the case where $\mathrm{dom}(\Omega)$ is a polytope, which is more general than the simplex $\triangle_N$. A vertex of this polytope might indicate an association among memory patterns informed by some desired structure. We develop this scenario in \S\ref{sec:structured}. 

\subsubsection{Normalization operations} \label{sec:normalization}

The function $\Psi$ can also be used to induce a post-normalization operation, as hinted in the classical and dense associative cases highlighted above. Post-normalization was discussed by \citet{millidge2022universal}. 
One way to do this is to define $\Psi(\bm{q}) = I_{\|.\| \le r}(\bm{q})$, whose Fenchel conjugate is $\Psi^*(\bm{z}) = r\|\bm{z}\|$ and has gradient map $(\nabla \Psi^*)(\bm{z}) = \frac{r\bm{z}}{\|\bm{z}\|}$, which corresponds to \textbf{$\ell_2$-normalization}. This normalization technique was explored by \citet{krotov2021large}. An alternative  \textbf{layer normalization} approach was proposed by \cite{hoover2023energy}, but without explicitly deriving the underlying energy term. 

%In the latter
%\begin{equation}
%    \Psi(\bm{q}) = I_S(\bm{q}) \quad  \text{and} \quad S = \left\{ \bm{q}  \,\,:\,\, \|\bm{q} - \bm{\delta}\| \le \eta \sqrt{D} \,\, \wedge \,\,  \bm{1}^\top (\bm{q} - \bm{\delta})=0 \right\},
%\end{equation} 
%\andre{add proof; we need to be more assertive stating that this is an original result we are deriving in our paper} 

%with $\eta$ and $\bm{\delta}$ serving as trainable parameters, leading to gradient map 
%\begin{equation}
%    (\nabla \Psi^*)(\bm{z})=\eta \frac{\bm{z}- \mu_{\bm{z}}}{\sqrt{\frac{1}{D}\sum_i(z_i-\mu_{\bm{z}})^2}} + \bm{\delta},
%\end{equation}
%where $\mu_{\bm{z}} := \frac{1}{D} %\mathbf{1}^\top \bm{z}$. 

We establish next a result, proved in Appendix~\ref{sec:proof_prop_normalization}, which derives explicitly the energy term (via $\Psi(\bm{q})$) which gives rise to the layer normalization post-transformation in HFY networks. 
To the best of our knowledge, this result has never been explicitly derived before.

\begin{proposition}[Layer normalization]\label{prop:normalization}
    Consider the layer normalization map 
\begin{equation}\label{eq:layernorm}
    \mathrm{LayerNorm}(\bm{z}; \eta, \bm{\delta}) := \eta \frac{\bm{z}- \mu_{\bm{z}}}{\sqrt{\frac{1}{D}\sum_i(z_i-\mu_{\bm{z}})^2}} + \bm{\delta},
\end{equation}
where $\eta > 0$ and $\bm{\delta} \in \mathbb{R}^D$ are arbitrary parameters, and where $\mu_{\bm{z}} := \frac{1}{D} \mathbf{1}^\top \bm{z}$. 
If we choose 
\begin{equation}
    \Psi(\bm{q}) := I_S(\bm{q}) \quad  \text{with} \quad S := \left\{ \bm{q}  \,\,:\,\, \|\bm{q} - \bm{\delta}\| \le \eta \sqrt{D} \,\, \wedge \,\,  \bm{1}^\top (\bm{q} - \bm{\delta})=0 \right\},
\end{equation}
then we have that 
$\hat{\bm{y}}_\Psi(\bm{z}) = (\nabla \Psi^*)(\bm{z}) = \mathrm{LayerNorm}(\bm{z}; \eta, \bm{\delta})$. 
\end{proposition}

%The proof of this result is in Appendix~\ref{sec:proof_prop_normalization}. 

\subsubsection{Sum of energy functions and infimal convolutions} 

It may be convenient to consider sums of energy functions, as done recently by \citet{hoover2023energy}. Let $\Omega_1$ and $\Omega_2$ be two functions and suppose we would like to have $\Omega^* = \Omega_1^* + \Omega_2^*$. The corresponding function $\Omega$ is the infimal convolution $\Omega = \Omega_1 \square\, \Omega_2$ \citep[\S 12]{bauschke2017convex}, defined as 
\begin{equation}\label{eq:infimal_conv}
    (\Omega_1 \square\, \Omega_2)(\bm{y}) := \inf_{\bm{z} \in \mathbb{R}^N} \left[ \Omega_1(\bm{y} - \bm{z}) + \Omega_2(\bm{z}) \right].
\end{equation}
This leads to the desired convex conjugate $\Omega^* = (\Omega_1 \square\, \Omega_2)^* = \Omega_1^* + \Omega_2^*$, and therefore to the sum of the gradient maps $\hat{\bm{y}}_\Omega = \hat{\bm{y}}_{\Omega_1} + \hat{\bm{y}}_{\Omega_2}$.

\bigskip

%We next focus on sparse Hopfield energies and study their properties by making a connection to margins in Fenchel-Young losses.

\section{Sparse Hopfield Networks}\label{sec:probabilistic}
In this section we assume that $\mathrm{dom}(\Omega) = \triangle_N$ (the probability simplex). We now use the general result derived in Proposition~\ref{prop:cccp} to define a particular instance of sparse energy functions for modern Hopfield networks. 
Then, we study their properties by making a connection to margins in Fenchel-Young losses.

\subsection{Generalized negentropies}
\label{sec:GE}
A \textbf{generalized (neg)entropy} \citep{grunwald2004generalized,Amigo2018GeneralizedEntropies}, formally defined below, provides a flexible framework for measuring disorder or uncertainty. % in a distribution $\bm{y} \in \triangle_N$.

\begin{definition}
\label{def:generalized_negent}
   \(\Omega : \triangle_{N} \to \mathbb{R}\)
   is a generalized negentropy iff it satisfies the conditions
   \begin{enumerate}
      \item Zero negentropy: $\Omega(\bm{y})=0$
      if \(\bm{y}\) is a one-hot vector, \textit{i.e.},
      \(\bm{y}=\bm{e}_i\) for any \(i \in \{1,...,N\}\).
      \item Strict convexity:
      \(\Omega\left((1-\lambda)\bm{y} + \lambda \bm{y}'\right)
      < (1-\lambda)\Omega(\bm{y}) + \lambda \Omega(\bm{y}')
      \) for $\lambda \in \,\, ]0,1[$ and $\bm{y}, \bm{y}' \in \triangle_N$ with $\bm{y} \ne \bm{y}'$.  
      \item Permutation invariance:
      \(\Omega(\bm{Py})=\Omega(\bm{y})\) for any permutation matrix \(\bm{P}\)
      (\textit{i.e.}, a square matrix with a single 1 in each row and each column, and zero elsewhere).
   \end{enumerate}
\end{definition} 
This definition implies that $\Omega \le 0$ and that $\Omega$ is minimized when $\bm{y} = \mathbf{1}/N$, the uniform distribution \citep[Proposition~4]{blondel2020learning}, which justifies the name ``generalized negentropy.'' 
We next discuss choices of $\Omega$ which lead to sparse alternatives to softmax. 

\subsection{Sparse transformations}
\begin{figure*}[t]
    \centering
    \includegraphics[width=1\columnwidth]{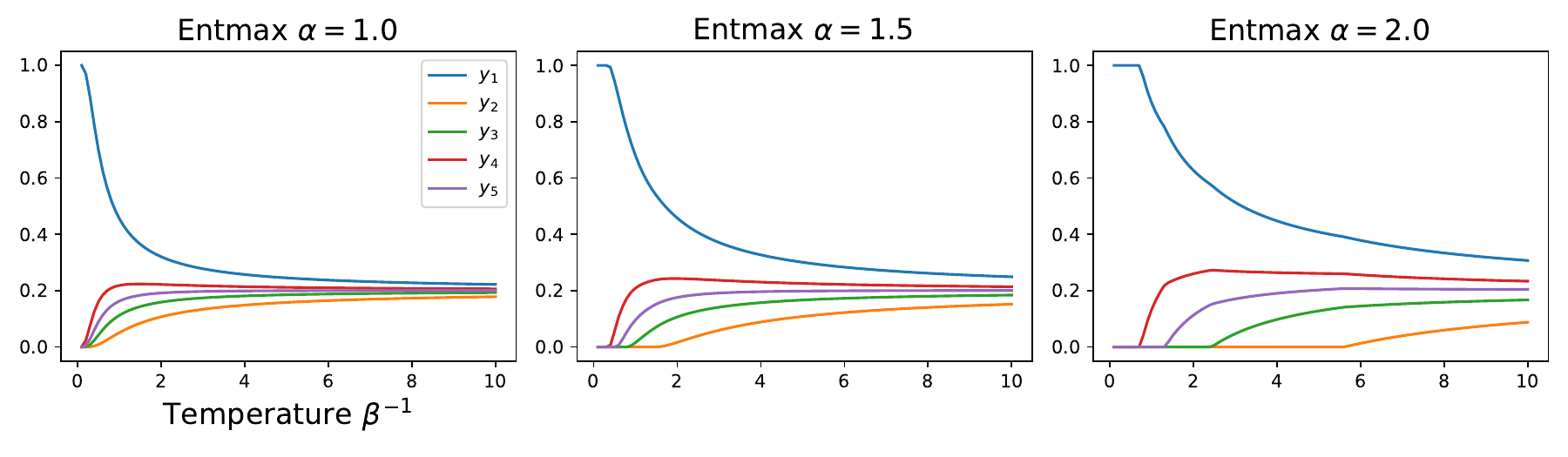}\\
    \includegraphics[width=1\columnwidth]{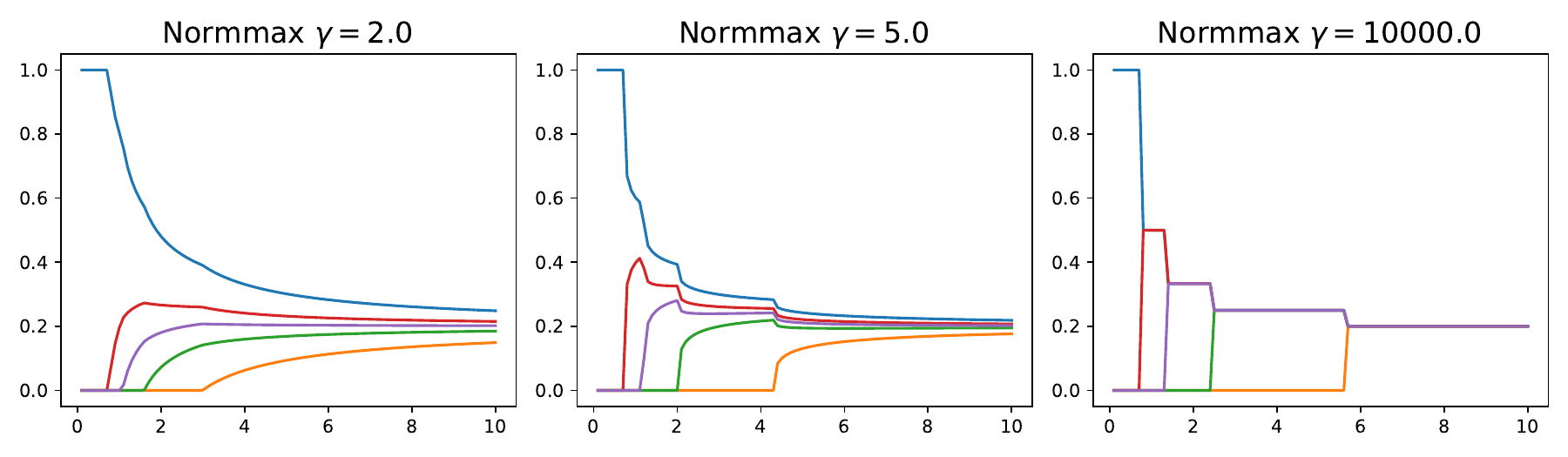}\\
    \vspace{0.3cm}
    \includegraphics[width=1\columnwidth]{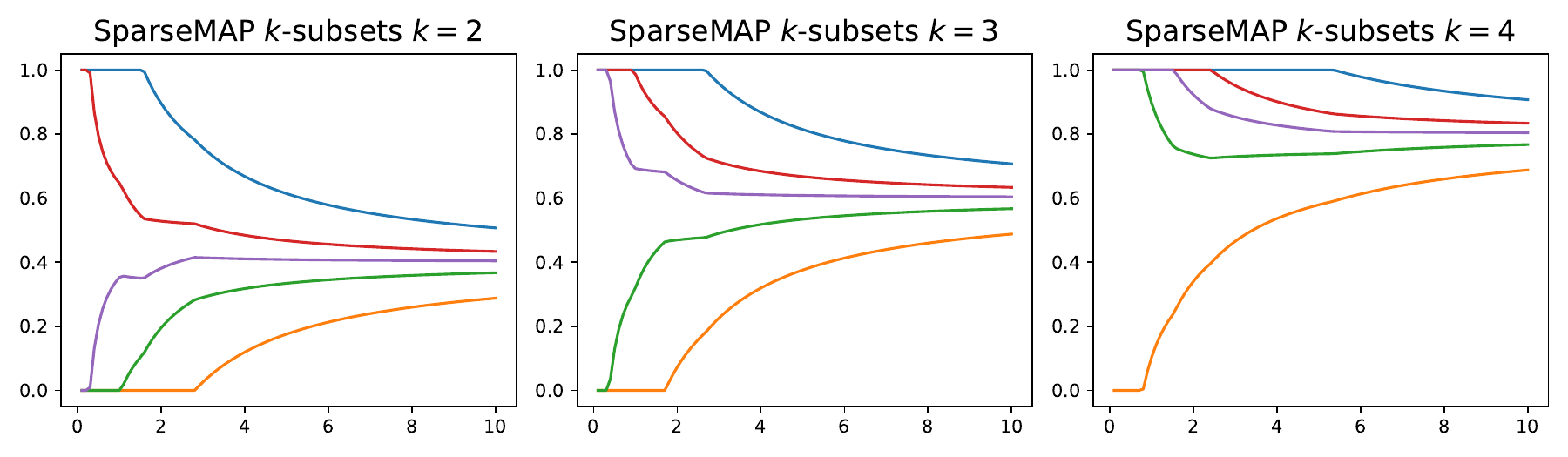}
    \caption{Sparse and structured transformations used in this paper and their regularization path. In each plot, we show $\hat{\bm{y}}_\Omega(\beta \bm{\theta}) = \hat{\bm{y}}_{\beta^{-1}\Omega}(\bm{\theta})$ as a function of the temperature $\beta^{-1}$ where $\bm{\theta} = [1.0716, -1.1221, -0.3288, 0.3368, 0.0425]^\top$.}
    \label{fig:sparse_transformations}
\end{figure*} 
In \S\ref{sec:FYL} we saw sparsemax, which is an example of a sparse transformation.
In fact, the softmax and sparsemax transformations 
%One instance of regularized prediction function \eqref{eq:rpm} is the \textbf{softmax} transformation, obtained when the regularizer is the Shannon negentropy, $\Omega(\bm{y}) = \sum_{i=1}^N y_i \log y_i + I_{\triangle_N}(\bm{y})$. % (negative entropy).
%Another 
%%example
%instance
%is the \textbf{sparsemax} transformation, obtained when $\Omega(\bm{y}) = \frac{1}{2}\|\bm{y}\|^2 + I_{\triangle_N}(\bm{y})$ \citep{martins2016softmax}.
%The sparsemax transformation is 
%corresponds to 
%the Euclidean projection onto the probability simplex.
%Softmax and sparsemax 
are both particular cases of a broader family of \textbf{$\alpha$-entmax transformations} \citep{peters2019sparse}, parametrized by a scalar $\alpha \ge 0$ (called the entropic index). These transformations correspond to the following choice of regularizer $\Omega$, called the \textbf{Tsallis $\alpha$-negentropy} \citep{tsallis1988possible}:
%\begin{align}\label{eq:tsallis}
    %\Omega_\alpha(\bm{p}) = \frac{1}{\alpha(\alpha-1)}\left(-1 + \sum_{i=1}^N p_i^\alpha\right).
%    \Omega^T_\alpha(\bm{y}) := \frac{-1 + \|\bm{y}\|_\alpha^\alpha}{\alpha(\alpha-1)} + I_{\triangle_N}(\bm{y}).
%\end{align}
\begin{align}\label{eq:tsallis}
    \Omega^T_\alpha(\bm{y}) = \begin{cases}
    \frac{-1 + \|\bm{y}\|_\alpha^\alpha}{\alpha(\alpha-1)} + I_{\triangle_N}(\bm{y}), & \text{if } \alpha \neq 1 \\
    \sum_i y_i \log y_i, & \text{if } \alpha = 1.
    \end{cases}
\end{align}
%Softmax and sparsemax correspond to $\alpha=1$ and $\alpha=2$, respectively. 
When $\alpha = 1$,
the Tsallis $\alpha$-negentropy
\(\Omega_\alpha^T\)
reduces to the Shannon's negentropy, leading to the softmax transformation. When $\alpha=2$, it becomes the Gini negentropy, which equals the $\ell_2$-norm (up to a constant), leading to the sparsemax transformation \citep{martins2016softmax}.
%We write $\alpha\text{-}\mathrm{entmax}(\bm{\theta}) := \hat{\bm{p}}_{\Omega_\alpha}(\bm{\theta})$ for the $\alpha$-entmax transformation. 

Another example is the  \textbf{norm $\gamma$-negentropy} \citep[\S4.3]{blondel2020learning}, 
\begin{align}\label{eq:norm_entropy}
    \Omega^{N}_{\gamma}(\bm{y}) := -1 + \|\bm{y}\|_\gamma + I_{\triangle_N}(\bm{y}), 
\end{align}
which, when $\gamma \rightarrow +\infty$, is called the Berger-Parker dominance index \citep{may1975patterns}, widely used in ecology to measure the diversity of a species within a community. 
We call the resulting transformation \textbf{$\gamma$-normmax}.  
While the Tsallis and norm negentropies have similar expressions and the resulting transformations both tend to be sparse, they have important differences, as suggested in Figure~\ref{fig:sparse_transformations}: 
normmax favors distributions closer to uniform over the selected support. 
We will come back to these properties in the subsequent sections. 

The examples above are all instances of transformations induced by generalized negentropies \citep{blondel2020learning}: 
%These conditions imply that $\Omega \le 0$ and that $\Omega(\bm{y})$ is minimized when $\bm{y}=\mathbf{1}/N$ is the uniform distribution \citep[Prop.~4]{blondel2020learning}. 
Tsallis negentropies \eqref{eq:tsallis} for $\alpha \ge 1$ and  norm negentropies \eqref{eq:norm_entropy} for $\gamma>1$ both satisfy the properties stated in \S\ref{sec:GE}. 

%Let $\Omega^*$ be the convex conjugate of $\Omega$,
%\begin{align}
%    $\Omega^*(\bm{\theta}) = \sup_{\bm{y} \in \mathrm{dom}(\Omega)} \bm{\theta}^\top \bm{y} - \Omega(\bm{y}).$
%\end{align}
%The $\Omega$-RPM in \eqref{eq:rpm} equals the gradient map of $\Omega^*$,
%$\hat{\bm{y}}_\Omega(\bm{\theta}) = \nabla \Omega^*(\bm{\theta})$.
%Note also that we have $\Omega^*(\bm{\theta}) = \bm{\theta}^\top \hat{\bm{y}}_\Omega(\bm{\theta}) - \Omega(\hat{\bm{y}}_\Omega(\bm{\theta}))$. 

\subsection{Sparsity and margins}

As seen in \S\ref{sec:FYL}, convex regularizers $\Omega$ can be used to define not only a regularized argmax transformation, but also a Fenchel-Young loss. 
What happens when $\Omega$ is a sparsity-inducing generalized negentropy? 
We next show that, 
in addition to the general properties mentioned in Proposition~\ref{prop:properties}, Fenchel-Young losses induced by such negentropies also satisfy a \textbf{margin} property,
which, as we shall see,  plays a pivotal role in the convergence and storage capacity of 
the class of Hopfield networks to be studied in \S\ref{sec:probabilistic} and \S\ref{sec:structured}. 

\begin{definition}[Margin]\label{def:margin} A loss function $L(\bm{\theta}; \bm{y})$ has a \textbf{margin} if there is a finite $m \ge 0$ such that
\begin{align}\label{eq:margin}
\begin{split}
\forall i \in [N], \quad
    L(\bm{\theta}, \bm{e}_i) = 0     &\Longleftrightarrow {\theta}_i - \max_{j \ne i} {\theta}_j \ge m.
\end{split}
\end{align}
The smallest such $m$ is called the margin of $L$. 
If $L_\Omega$ is a Fenchel-Young loss, 
\eqref{eq:margin} is equivalent to $\hat{\bm{y}}_{\Omega}(\bm{\theta}) = \bm{e}_i$. 
\end{definition} 
A famous example of a loss with a margin of 1 is the hinge loss of support vector machines. 
On the other hand, the cross-entropy loss does not have a margin, as suggested in Figure~\ref{fig:margins}: it never reaches exactly zero for \( \bm{\theta} = [s, 0] \in \mathbb{R}^2 \), unlike the $\alpha$-entmax and $\gamma$-normmax losses.
\begin{figure*}[t]
    \centering
    \includegraphics[width=1\textwidth]{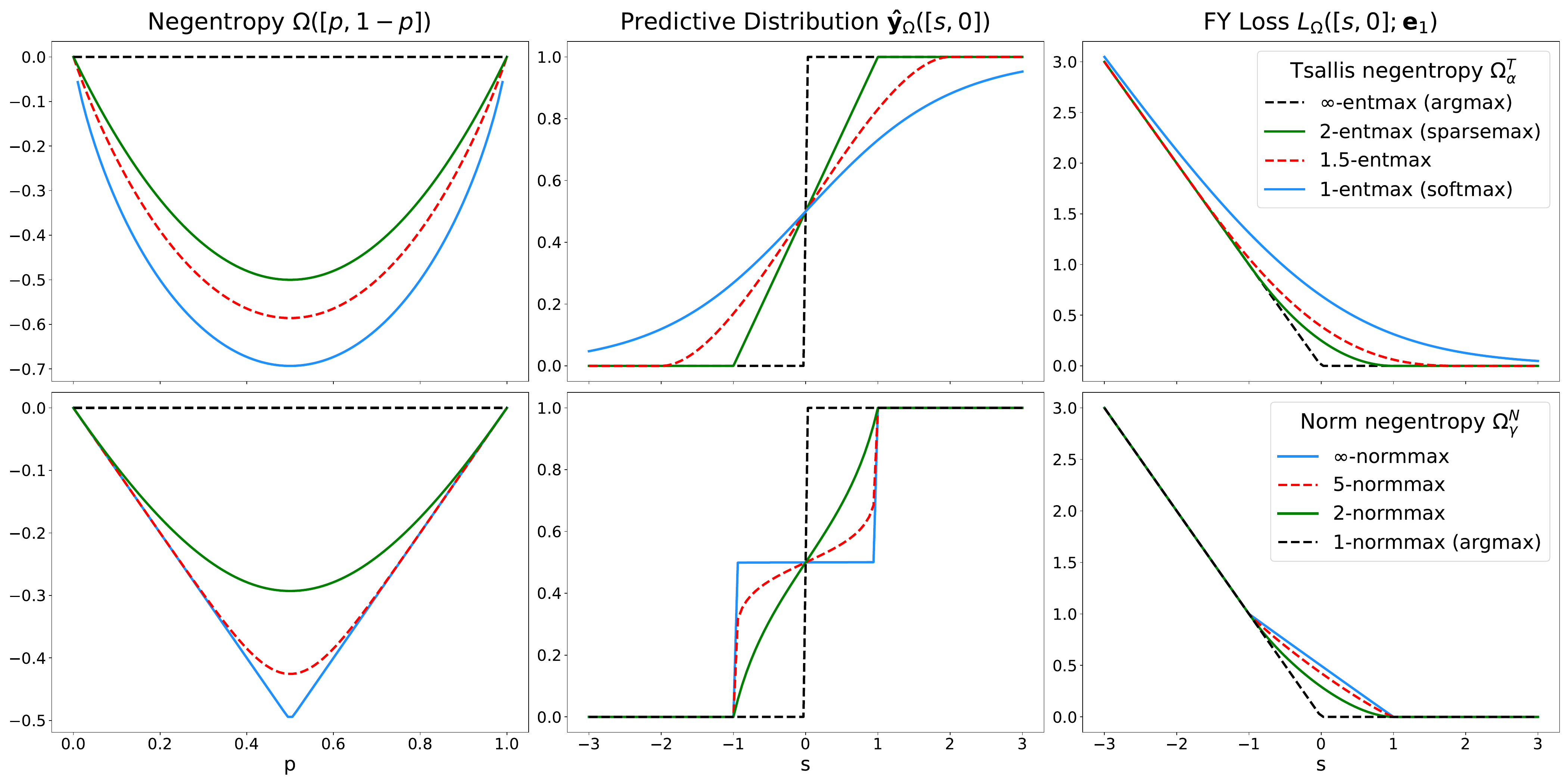}
\caption{Examples of \textbf{generalized entropies} (left) are presented alongside their corresponding \textbf{prediction distributions} (middle) and \textbf{Fenchel-Young losses} (right) for the binary case. Here, \( \bm{y} = [p, 1 - p] \in \triangle_2 \), \( \bm{\theta} = [s, 0] \in \mathbb{R}^2 \), and $\bm{e}_1$ is the one-hot vector for the first class. Unlike softmax, which never reaches exactly zero and consequently does not have a margin, all other distributions shown in the center can exhibit sparse support.}
    \label{fig:margins}
\end{figure*}We then have the following result, proved by \citet{blondel2020learning}: 
\begin{proposition}[Margin Properties of Tsallis and Norm-Entropies]
\label{prop:margins}
Tsallis and norm negentropies have the following margins:
\begin{enumerate}
    \item \textbf{Tsallis negentropies} $\Omega^T_\alpha$ with $\alpha > 1$ lead to a loss $L_{\Omega^T_\alpha}$ with a margin $m = (\alpha - 1)^{-1}$.
    \item \textbf{Norm-negentropies} $\Omega^N_\gamma$ with $
    \gamma > 1$ lead to a loss $L_{\Omega^N_\gamma}$  with a margin $m = 1$, independent of $\gamma$.
\end{enumerate}
\end{proposition}  
%\begin{proposition}[Sufficient and Necessary Conditions for a Margin] \label{prop:conditions_margin}
%Let $\Omega$ be a generalized negentropy, % and be uniformly separable, i.e., $\Omega(\bm{p}) = \sum_{i=1}^{d}  h(p_i)$ for some function $h$. 
%then the sufficient and necessary conditions for the existence of a separation margin in $L_\Omega$ are:
%\begin{enumerate}
%    \item $\partial(\Omega)(\bm{p}) \neq \emptyset$ for any $\bm{p} \in \mathbb{R}^d$;
%    \item The mapping $\nabla(\Omega)^*$ covers the full simplex, i.e., $\nabla(\Omega)^*(\mathbb{R}^d) = \Delta^d$;
%\end{enumerate}
%\end{proposition}
In fact, the sparsity of $\hat{\bm{y}}_{\Omega}$ is a sufficient condition for $L_\Omega$ having a margin \citep[Proposition 6]{blondel2020learning}.
By leveraging these established facts from prior propositions, we will prove that our class of Hopfield networks is capable of exact retrieval.

\subsection{Sparse Hopfield networks: Definition and update rule}\label{sec:hfy_definition}

In this section, we study a specialization of the HFY energy defined in \S\ref{sec:HFYE} by assuming that the regularizer $\Omega$ has domain $\mathrm{dom}(\Omega) = \triangle_N$ and that it is a generalized negentropy (see Definition \ref{def:generalized_negent}). 
We assume also that $\Psi(\bm{q}) = \frac{1}{2}\|\bm{q}\|^2$, and we use $\bm{u} = \frac{\mathbf{1}}{N}$ as the baseline. Using  $\tilde{\Omega}(\bm{\theta}) = \Omega(\beta \bm{\theta})$, where $\beta^{-1} > 0$ is a temperature parameter, the HFY energy \eqref{eq:general_hfy_energy} becomes the following, up to an extra constant:
\begin{equation}\label{eq:hfy_energy}
    E(\bm{q}) = \underbrace{-\beta^{-1} L_\Omega(\beta \bm{X} \bm{q}; \mathbf{1}/{N})}_{E_{\mathrm{concave}}(\bm{q})} + \underbrace{\frac{1}{2} \|\bm{q} - \bm{\mu}_{\bm{X}}\|^2 + \frac{1}{2}(M^2 - \|\bm{\mu}_{\bm{X}}\|^2)}_{{E_{\mathrm{convex}}(\bm{q})}},
\end{equation}
where $\bm{\mu}_{\bm{X}} := \bm{X}^\top \mathbf{1}/N \in \mathbb{R}^D$ is the empirical mean of the patterns and $M := \max_i \|\bm{x}_i\|$. 
%\andre{Note that $E(\bm{q}) = -\beta^{-1} L_\Omega(\beta \bm{X} \bm{q}; \mathbf{1}/{N}) - \bm{m}_{\bm{X}}^\top \bm{q} + 1/2 \|\bm{q}\|^2 + 1/2M^2$, where $\bm{m}_{\bm{X}} = X^\top \mathbf{1}/N$, i.e.,
%$E(\bm{q}) = -\beta^{-1} L_\Omega(\beta \bm{X} \bm{q}; \mathbf{1}/{N}) + 1/2 \|\bm{q} - \bm{m}_{\bm{X}}\|^2 + 1/2(M^2 - \|\bm{m}_{\bm{X}}\|^2)$.}
This energy extends that of \citet{ramsauer2020hopfield} in \eqref{eq:energy_hopfield}, which is recovered when $\Omega$ is Shannon's negentropy. %
%\footnote{Up to constants, for this choice of $\Omega$ our convex-concave decomposition is the same as 
%\citet{ramsauer2020hopfield}.
% when $\Omega$ is Shannon's negentropy.
%} % 
The $E_{\text{convex}}$ and $E_{\text{concave}}$
terms compete when minimizing %the energy %function  
\eqref{eq:hfy_energy}: % with respect to $\bm{q}$:
\begin{itemize}
    \item Minimizing $E_{\mathrm{concave}}$ is equivalent to \textit{maximizing} $L_{\Omega}(\beta \bm{X} \bm{q}; \mathbf{1}/{N})$, 
    %% VN fit in a line
    %which pushes for 
    which pushes
    $\bm{q}$ to be as far as possible from a uniform average and closer to a single pattern.
    \item Minimizing $E_{\mathrm{convex}}$ serves as proximity regularization, encouraging $\bm{q}$ to stay close to $\bm{\mu}_{\bm{X}}$.
\end{itemize} 
The next result is a consequence of our Proposition~\ref{prop:cccp},  generalizing \citet[Lemma A.1, Theorem A.1]{ramsauer2020hopfield}. The bounds on the energy are derived in Appendix~\ref{app:general_cccp}. 
\begin{proposition}[Update rule of sparse HFY energies]\label{prop:bounds_cccp}
    Let the query $\bm{q}$ be in the convex hull of the rows of $\bm{X}$, i.e., $\bm{q} = \bm{X}^\top \bm{y}$ for some $\bm{y} \in \triangle_N$. Then, the energy \eqref{eq:hfy_energy} satisfies
%% inline equation:
\(0 \le E(\bm{q}) \le \min\left\{2M^2, \,\, -\beta^{-1}\Omega(\mathbf{1}/N) + \frac{1}{2}M^2\right\}\).
%% out-of-line equation
%    $$0 \le E(\bm{q}) \le \min\left\{2M^2, \,\, -\beta^{-1}\Omega(\mathbf{1}/N) + \frac{1}{2}M^2\right\}.$$
Furthermore, minimizing \eqref{eq:hfy_energy} with the CCCP algorithm \citep{yuille2003concave} leads to the updates
\begin{align}\label{eq:updates_hfy}
    \bm{q}^{(t+1)} = \bm{X}^\top \hat{\bm{y}}_\Omega(\beta \bm{X} \bm{q}^{(t)}).
\end{align}
\end{proposition}

\paragraph{Entmax and normmax.} When $\Omega=\Omega^T_\alpha$ (the Tsallis $\alpha$-negentropy \eqref{eq:tsallis}), the update \eqref{eq:updates_hfy} corresponds to the adaptively sparse transformer of \citet{correia2019adaptively}. 
The $\alpha$-entmax transformation can be computed %in linear time 
efficiently with sort or top-$k$ algorithms
for $\alpha \in \{1.5, 2\}$; for other values of $\alpha$, an efficient bisection algorithm was proposed by \citet{peters2019sparse}. 
The case $\alpha=2$ (sparsemax) corresponds to the sparse modern Hopfield network proposed by \citet{hu2023sparse}. When $\alpha > 1$ and $\alpha \ne 2$, we recover the $\alpha$-entmax Hopfield network introduced by \citet{martins2023sparse}, \citet{wu2023stanhop}, and  \citet{santos2024sparse}. When $\Omega=\Omega^N_\gamma$ (the norm $\gamma$-negentropy \eqref{eq:norm_entropy}), we obtain the $\gamma$-normmax transformation. This transformation is more challenging since $\Omega^N_\gamma$ is not separable, but Appendix~\ref{sec:normmax} presents a bisection algorithm which works for any $\gamma>1$.  %\andre{maybe explain here that we can also solve this problem by composing entmax with renormalization}. 

%in this paper an expression and an efficient bisection algorithm which works for any $\alpha>1$, described in Appendix~\ref{sec:normmax}.  

%The result above generalizes the energy bounds and update rule of \citet[Lemma A.1 and Theorem A.1]{ramsauer2020hopfield}.
%In particular, when $\Omega$ is the Tsallis $\alpha$-entropy, the regularized predition function is the $\alpha$-entmax transformation and we obtain a correspondence with the adaptively sparse transformer proposed by \citep{correia2019adaptively}.

%\andre{show that there is global and local convergence, try to generalize the results of \citep{ramsauer2020hopfield}, compute the memory capacity, etc.}

\paragraph{$\ell_2$-normalization and layer normalization.} 
It is possible to extend the idea above to incorporate a post-transformation as described in \S\ref{sec:normalization}. We discuss this scenario in \S\ref{sec:extension_normalization}.

\subsection{Properties: Margins, sparsity, and exact retrieval}\label{sec:hfy_margins}

Prior work on modern Hopfield networks \citep[Def.~1]{ramsauer2020hopfield} 
defines pattern storage and retrieval in an {\it approximate} sense: they assume a small neighbourhood around each pattern $\bm{x}_i$ containing an attractor $\bm{x}_i^*$, such that if the initial query $\bm{q}^{(0)}$ is close enough, the Hopfield updates will converge to $\bm{x}_i^*$, leading to a retrieval error of $\|\bm{x}_i^* - \bm{x}_i\|$. For this error to be small, a large $\beta$ may be necessary. 
We consider here a stronger definition of \textbf{exact retrieval}, where the attractors \textit{coincide} with the actual patterns (rather than being nearby). 
Our main result is that 
\textbf{zero retrieval error} is possible in HFY networks as long as the corresponding Fenchel-Young loss has a \textbf{margin} 
%% VN: I don't think cf is correct here since there is no contrast.
%(cf.~Def.~\ref{def:margin}). 
(Def.~\ref{def:margin}). 
%We show now a general result connecting Fenchel-Young losses with margins (cf.~Definition~\ref{def:margin}) and \textbf{exact retrieval} in Hopfield-Fenchel-Young networks. 
Given that $\hat{\bm{y}}_\Omega$ being a sparse transformation is a sufficient condition for $L_\Omega$ having a margin \citep[Proposition 6]{blondel2020learning}, this is a general statement about sparse transformations.%
\footnote{At first sight, this might seem to be a surprising result,  given that both queries and patterns are continuous. The reason why exact convergence is possible hinges crucially on sparsity.} 

\begin{definition}[Exact retrieval]\label{def:exact_retrieval}
    A pattern $\bm{x}_i$ is \textbf{exactly retrieved} for  query $\bm{q}^{(0)}$ iff there is a finite number of steps $T$ such that  iterating \eqref{eq:updates_hfy} %starting with $\bm{q}^{(0)}$ 
    leads to $\bm{q}^{(T')} = \bm{x}_i$ $\forall T' \ge T$. 
\end{definition}

The following result gives sufficient conditions for exact retrieval with $T=1$ given that patterns are well separated and that the query is sufficiently close to the retrieved pattern. 
It establishes the \textbf{exact autoassociative} property of this class of HFY networks: if all patterns are slightly perturbed, the Hopfield dynamics are able to recover the original patterns exactly.  
Following \citet[Def.~2]{ramsauer2020hopfield}, we define the separation of pattern $\bm{x}_i$ from data as $\Delta_i = \bm{x}_i^\top \bm{x}_i - \max_{j \ne i} \bm{x}_i^\top \bm{x}_j$.

\begin{proposition}[Exact retrieval in a single iteration]\label{prop:separation}
    Assume $L_\Omega$ has margin $m$, 
    and let $\bm{x}_i$ be a  pattern outside the convex hull of the other patterns.
    Then
    \begin{enumerate}
    \item 
    $\bm{x}_i$ is a stationary point of the energy \eqref{eq:hfy_energy} iff $\Delta_i \ge m{\beta^{-1}}$. %In particular, for the energy function of \citet{ramsauer2020hopfield} ($\alpha=1$ and $\beta>0$), no such memory pattern is a stationary point (although it might be arbitrarily close to one).
    \item In addition, if the initial query $\bm{q}^{(0)}$ satisfies ${\bm{q}^{(0)}}^\top (\bm{x}_i -\bm{x}_j) \ge m{\beta^{-1}}$ for all $j \ne i$, then the update rule \eqref{eq:updates_hfy} converges to $\bm{x}_i$ exactly in one iteration. 
    \item Moreover, if the patterns are normalized, $\|\bm{x}_i\| = M$ for all $i$, and well-separated with $\Delta_i \ge m{\beta^{-1}} + 2M\epsilon$, then any $\bm{q}^{(0)}$ $\epsilon$-close to $\bm{x}_i$ ($\|\bm{q}^{(0)} - \bm{x}_i\| \le \epsilon$) will converge to $\bm{x}_i$ in one iteration. 
    \end{enumerate}
\end{proposition}

%\andre{This can be captured through the idea of perturbations: if all patterns are $\epsilon$-perturbed (meaning each pattern $\bm{x}_i$ is moved to some location $\tilde{\bm{x}}_i$ such that $\|\tilde{\bm{x}}_i - \bm{x}_i\| \le \epsilon$) can the Hopfield dynamics recover the original $\bm{x}_i$ from $\tilde{\bm{x}}_i$? Mention autoassociative property.}

%\begin{proposition}\label{prop:separation}
%    Assume $\Omega=\Omega_\alpha$ with $\alpha>1$, 
%    and let $\bm{x}_i$ be a memory pattern outside the convex hull of the other memory patterns.
%    Then, $\bm{x}_i$ is a stationary point of the energy \eqref{eq:hfy_energy} iff $\Delta_i \ge \frac{1}{(\alpha-1)\beta}$. %In particular, for the energy function of \citet{ramsauer2020hopfield} ($\alpha=1$ and $\beta>0$), no such memory pattern is a stationary point (although it might be arbitrarily close to one).
%    In addition, if the initial query $\bm{q}_0$ satisfies $\bm{q}_0^\top (\bm{x}_i -\bm{x}_j) \ge \frac{1}{(\alpha-1)\beta}$ for all $j \ne i$, then the update rule \eqref{eq:updates_hfy} converges to $\bm{x}_i$ exactly in one iteration. Moreover, if the patterns are normalized, $\|\bm{x}_i\| = M$ for all $i$, and $\Delta_i \ge \frac{1}{(\alpha-1)\beta} + 2M\epsilon$ \andre{I fixed this}, then any $\bm{q}_0$ $\epsilon$-close to $\bm{x}_i$ ($\|\bm{q}_0 - \bm{x}_i\| \le \epsilon$) will converge to $\bm{x}_i$ in one iteration. 
%\end{proposition}

The proof is in Appendix~\ref{sec:proof_prop_separation}. 
%This result shows that for sparse transformations \textbf{the memory patterns can be stationary points of the energy} \eqref{eq:hfy_energy}. 
For the Tsallis negentropy case $\Omega = \Omega^T_\alpha$ with $\alpha>1$ (the sparse case), 
we have $m = (\alpha-1)^{-1}$ (cf.~Def.~\ref{def:margin}), leading to the condition $\Delta_i \ge \frac{1}{(\alpha-1)\beta}$. 
This result is stronger than that of \citet{ramsauer2020hopfield} for their energy (which is ours for $\alpha=1$), according to which memory patterns are only $\epsilon$-close to stationary points, where a small $\epsilon = 
\mathcal{O}(\exp(-\beta))$ requires a low temperature (large $\beta$). %The proof (Appendix~\ref{sec:proof_prop_separation}) relies on the margin property stated in \eqref{eq:margin}.
It is also stronger than the retrieval error bound recently derived by \citet[Theorem 2.2]{hu2023sparse} for the case $\alpha=2$, which has an additive term involving $M$ and therefore does not provide conditions for exact retrieval. 
For the normmax negentropy case $\Omega = \Omega^N_\gamma$ with $\gamma>1$, we have $m=1$, so the condition above becomes $\Delta_i \ge \frac{1}{\beta}$.

%\andre{I added this part, this is inspired by \citep{ramsauer2020hopfield} Theorems A3 and Lemma A13.}

Given that exact retrieval is a stricter definition, one may wonder whether requiring it sacrifices storage capacity. 
Reassuringly, the next result, inspired but stronger than \citep[Theorem~A.3]{ramsauer2020hopfield}%
\footnote{In fact, there is a mistake in the proof of Theorem A.3 by \citet{ramsauer2020hopfield} regarding their bound on ``equidistant patterns on the sphere''. 
We fix their proof in Appendix~\ref{sec:proof_prop_storage} by making a connection with existing bounds in sphere packing and kissing numbers, an active problem in combinatorics \citep{chabauty1953resultats,conway2013sphere}.} %
and proved in our Appendix~\ref{sec:proof_prop_storage}, shows that HFY networks with exact retrieval also have exponential storage capacity.  

\begin{proposition}[Storage capacity with exact retrieval]\label{prop:storage}
Assume patterns are optimally placed on the sphere of radius $M$. 
The HFY network can store and exactly retrieve $N = \mathcal{O}((2/\sqrt{3})^D)$ 
%$N = 2^{2(D-1)}$ 
patterns in one iteration under an $\epsilon$-perturbation as long as 
$M^2 > 2m{\beta^{-1}}$ and $\epsilon \le \frac{M}{4} - \frac{m}{2\beta M}$. 

Assume patterns are randomly placed on the sphere with uniform distribution. 
Then, with probability $1-p$, the HFY network can store and exactly retrieve $N = \mathcal{O}(\sqrt{p} \zeta^{\frac{D-1}{2}})$ %=  \sqrt{\frac{2p}{\kappa_{D-1}}} \zeta^{\frac{D-1}{2}} 
patterns in one iteration under a $\epsilon$-perturbation 
%where $\zeta \in \left]1, {\arccos^{-1} \left(1 - m{\beta^{-1} M^{-2}}\right)}\right[$,  
if
\begin{equation}
    \epsilon \le \frac{M}{2} \left(1 - \cos \frac{1}{\zeta}\right) - \frac{m}{2\beta M}.
\end{equation}
\end{proposition}

\subsection{Extension of previous results for extra normalization step}\label{sec:extension_normalization}

We now extend the previous results to the scenario where a post-transformation $\hat{\bm{y}}_\Psi$ is applied, such as $\ell_2$-normalization or layer normalization, as described in \S\ref{sec:normalization}. 
In this scenario, the update rule \eqref{eq:updates_hfy} is replaced by 
\begin{align}\label{eq:updates_hfy_normalized}
    \bm{q}^{(t+1)} = \hat{\bm{y}}_\Psi\left( \bm{X}^\top \hat{\bm{y}}_\Omega(\beta \bm{X} \bm{q}^{(t)}) \right).
\end{align}
We consider the image set induced by this transformation $\mathrm{im} (\hat{\bm{y}}_\Psi) := \left\{\hat{\bm{y}}_\Psi(\bm{z}) \,:\, \bm{z} \in \mathbb{R}^D\right\}$. 
For $\ell_2$-normalization, $\hat{\bm{y}}_\Psi(\bm{z}) = \frac{r\bm{z}}{\|\bm{z}\|}$, this image set is a $(D-1)\textsuperscript{th}$ sphere of radius $r$, and for layer normalization, $\hat{\bm{y}}_\Psi(\bm{z}) = \mathrm{LayerNorm}(\bm{z}; \eta, \bm{\delta})$ \eqref{eq:layernorm}, 
it is the $(D-2)\textsuperscript{th}$-dimensional  set $\left\{ \bm{q}  \,\,:\,\, \|\bm{q} - \bm{\delta}\| = \eta \sqrt{D} \,\, \wedge \,\,  \bm{1}^\top (\bm{q} - \bm{\delta})=0 \right\}$. 

\begin{proposition}\label{prop:extension_normalization}
Assume $\Psi$ is chosen so that the transformation $\hat{\bm{y}}_\Psi$ is idempotent, \textit{i.e.},   $\hat{\bm{y}}_\Psi(\hat{\bm{y}}_\Psi(\bm{z})) = \hat{\bm{y}}_\Psi(\bm{z})$ for all $\bm{z} \in \mathbb{R}^D$. 
Then, if all patterns $\bm{x}_i$ satisfy $\bm{x}_i \in \mathrm{im}(\hat{\bm{y}}_\Psi)$, we have that all results in Propositions~\ref{prop:separation}--\ref{prop:storage}, which concern convergence in one iteration, also hold for the Hopfield updates \eqref{eq:updates_hfy_normalized}. 
\end{proposition}
\begin{proof}
Propositions~\ref{prop:separation}--\ref{prop:storage} guarantee that $\bm{X}^\top \hat{\bm{y}}_\Omega(\beta \bm{X} \bm{q}^{(0)}) = \bm{x}_i$ for some $i \in [N]$; 
since $\bm{x}_i \in \mathrm{im}(\hat{\bm{y}}_\Psi)$ and $\hat{\bm{y}}_\Psi$ is idempotent, the subsequent post-transformation in \eqref{eq:updates_hfy_normalized} will not change the result, ensuring that $\bm{q}^{(t+1)} = \bm{x}_i$. 
\end{proof}

The idempotence condition is satisfied for both the $\ell_2$-normalization and layer normalization transformations mentioned above. 
The condition $\bm{x}_i \in \mathrm{im}(\hat{\bm{y}}_\Psi)$ is satisfied if the patterns  in $\bm{X}$ are pre-normalized with the same post-transformation $\hat{\bm{y}}_\Psi$ that is applied to the queries. 
The conditions in Propositions~\ref{prop:separation}--\ref{prop:storage} which require $\|\bm{x}_i\| = M$ are satisfied, in the $\ell_2$-normalization case, by $r=M$, and in the layer normalization case by $\eta = \frac{M}{\sqrt{D}}$ and $\bm{\delta} = \mathbf{0}$. 

Even in situations where the initial query $\bm{q}^{(0)}$ is not sufficiently close to a pattern to obtain convergence in one step, the inclusion of a post-transformation step under the conditions of Proposition~\ref{prop:extension_normalization} can speed up convergence by projecting the query to a smaller subspace $\mathrm{im}(\hat{\bm{y}}_\Psi)$ where the patterns are contained. This will be illustrated in \S\ref{sec:hopfield_dynamics}.

\section{Structured Hopfield Networks}\label{sec:structured}
%\andre{
%    \begin{itemize}
%        \item Briefly motivate the problem of inducing pattern associations as a structured problem?
%        \item Talk a bit about structured prediction, marginal polytopes, avoiding too much notation. Not sure if we need to talk about higher order potentials (though they will be needed for the SequenceBudget factor if we use it.
%        \item Present SparseMAP as a transformation that generalizes sparsemax and can do this. Show it corresponds to one instance of the Hopfield-Fenchel-Young energy
%        \item Example with $k$-subsets problem, equivalent to the budget problem; maybe another example with SequenceBudget?
%        \item Show the result from \citet{blondel2020learning} which says that SparseMAP has a structured margin under some assumptions, which hold for our problems of interest
%        \item Proposition showing the exact coverage / retrieval of pattern associations for structured and sparse MHNs.
%    \end{itemize}
%}

In \S\ref{sec:probabilistic}, we considered the case where $\bm{y} \in \mathrm{dom}(\Omega) = \triangle_N$, the scenario studied by \citet{ramsauer2020hopfield} and \citet{hu2023sparse}. 
Since $\triangle_N = \mathrm{conv}(\mathcal{Y})$ with $\mathcal{Y} = \{\bm{e}_1, ..., \bm{e}_N\}$, we can see the domain of $\Omega$ as the convex relaxation of the set $\mathcal{Y}$ of pattern indicators. 

We now go one step further and consider the more general \textbf{structured} case, where  $\mathrm{dom}(\Omega)$ is a polytope. More specifically, we assume that 
$\bm{y} \in \mathrm{dom}(\Omega) := \mathrm{conv}(\mathcal{Y})$ is a vector of ``marginals'' associated to some given structured set $\mathcal{Y}$. 
This structure can reflect \textbf{pattern associations} that we might want to induce when querying the Hopfield network with $\bm{q}^{(0)}$. 
Possible structures include \textbf{$k$-subsets} of memory patterns, potentially leveraging \textbf{sequential memory structure}, tree structures,  matchings, etc. 
In these cases, the set of pattern associations we can form is combinatorial, hence it can be considerably larger %(e.g., exponentially larger) 
than the number $N$ of memory patterns.  

\subsection{Unary scores and structured constraints}\label{sec:structured_unary}

As before, we assume $N$ is the number of patterns stored in the memory. 
Let us start with a simple scenario where there is a predefined set of binary structures $\mathcal{Y} \subseteq \{0, 1\}^N$ and $N$ unary scores $\bm{\theta} \in \mathbb{R}^N$, one for each memory pattern. 
We assume that we may have $|\mathcal{Y}| \gg N$ in general. 
 %(e.g., exploiting sequential/temporal structure). 
%Let $\mathcal{Y} \subseteq \{0, 1\}^N$ be a set of binary vectors indicating the underlying set of structures, and 
In what follows, we let $\mathrm{dom}(\Omega) = \mathrm{conv}(\mathcal{Y}) \subseteq [0,1]^N$ denote its convex hull, called the \textbf{marginal polytope} associated with the structured set $\mathcal{Y}$ \citep{wainwright2008graphical}. 
Later, in \S\ref{sec:structured_high_order}, we generalize this framework to accommodate higher-order interactions modeling soft interactions among patterns.

\begin{example}[$k$-subsets]\label{ex:ksubsets}
We may be interested in retrieving subsets of $k$ patterns, e.g., to take into account a $k$-ary relation among patterns or to perform top-$k$ retrieval. In this case, we can define, for $k \in [N]$, $$\mathcal{Y} := \left\{\bm{y} \in \{0, 1\}^N \,:\, \mathbf{1}^\top \bm{y} = k\right\}.$$ 
If $k=1$, we get $\mathcal{Y} = \{\bm{e}_1, ..., \bm{e}_N\}$ and $\mathrm{conv}(\mathcal{Y}) = \triangle_N$, recovering the  scenario in \S\ref{sec:probabilistic}. 
For larger $k$, $|\mathcal{Y}| = {N \choose k} \gg N$. 
With a simple rescaling, the marginal polytope $\mathrm{conv}(\mathcal{Y})$ is equivalent to the capped probability simplex  described by \citet[\S 7.3]{blondel2020learning}. 
\end{example}Given unary scores $\bm{\theta} \in \mathbb{R}^N$, %the score of a structure $\bm{y} \in \mathcal{Y}$ is decomposed as $\bm{\theta}^\top \bm{y}$. 
%the problem of 
%obtaining 
the structure with the largest score %amounts to solving the maximization problem 
is 
\begin{align}\label{eq:MAP}
\bm{y}^* = \argmax_{\bm{y} \in \mathcal{Y}} \bm{\theta}^\top \bm{y} = \argmax_{\bm{y} \in \mathrm{conv}(\mathcal{Y})} \bm{\theta}^\top \bm{y},
\end{align}
where the last equality comes from the fact that $\mathrm{conv}(\mathcal{Y})$ is a polytope, therefore the maximum is attained at a vertex. 
The solution of \eqref{eq:MAP} is often called the \textbf{maximum a posteriori (MAP)} assignment. 
As in \eqref{eq:rpm}, we consider a regularized prediction version of this problem via a convex regularizer $\Omega : \mathrm{conv}(\mathcal{Y}) \rightarrow \mathbb{R}$:
\begin{align}\label{eq:sparsemap}
\hat{\bm{y}}_\Omega(\bm{\theta}) := \argmax_{\bm{y} \in \mathrm{conv}(\mathcal{Y})} \bm{\theta}^\top \bm{y} - \Omega(\bm{y}). 
\end{align}
By choosing $\Omega(\bm{y}) = \frac{1}{2}\|\bm{y}\|^2 + I_{\mathrm{conv}(\mathcal{Y})}(\bm{y})$, 
we obtain \textbf{SparseMAP}, which can be seen as a relaxation of MAP and a structured version of sparsemax \citep{niculae2018sparsemap}. 
The SparseMAP transformation \eqref{eq:sparsemap} can be computed efficiently via an active set algorithm, as long as an algorithm is available to compute the MAP in \eqref{eq:MAP}, as shown in  
\citet[\S 3.2]{niculae2018sparsemap} 
We will make use of this efficient algorithm in our experiments in \S\ref{sec:experiments}.

\subsection{General case: factor graph, high order interactions}\label{sec:structured_high_order}

We consider now the more general case where there might be soft interactions among patterns, for example due to temporal dependencies, hierarchical  structure, etc. 
In general, these interactions can be expressed as a bipartite \textbf{factor graph} $(V, F)$, where $V = \{1, ..., N\}$ are variable nodes (associated with the patterns) and $F \subseteq 2^V$ are factor nodes representing the interactions \citep{kschischang2001factor}. 

A structure can be represented as a bit vector $\bm{y} = [\bm{y}_V; \bm{y}_F]$, where $\bm{y}_V$ and $\bm{y}_F$ indicate configurations of variable and factor nodes, respectively. Each factor $f \in F$ is linked to a subset of variable nodes $V_f \subseteq V$. 
We assume each variable $v \in V$ can take one of $N_v$ possible values, and we denote by $\bm{y}_v \in \{0,1\}^{N_v}$ a one-hot vector indicating a value for this variable. 
Likewise, each factor $f \in F$ has $N_f$ possible configurations, with $N_f = \prod_{v \in V_f} N_v$, and we associate to it a one-hot vector $\bm{y}_f \in \{0,1\}^{N_f}$ indicating a configuration for that factor. 
The global configuration of the factor graph is expressed through the bit vectors $\bm{y}_V = [\bm{y}_v \,:\, v \in V] \in \{0, 1\}^{N_V}$ and $\bm{y}_F = [\bm{y}_f \,:\, f \in F] \in \{0, 1\}^{N_F}$, with $N_V = \sum_{v \in V} N_v$ and $N_F = \sum_{f \in F} N_f$. A particular structure is expressed through the bit vector $\bm{y} = [\bm{y}_V; \bm{y}_F] \in \{0, 1\}^{N_V + N_F}$. 
Finally, we define the set of \textbf{valid structures} $\mathcal{Y} \subseteq \{0, 1\}^{N_V + N_F}$---this set contains all the bit vectors which correspond to valid structures, which must satisfy consistency between variable and factor assignments, as well as any additional hard constraints. 
%The convex hull of this set, $\mathrm{conv}(\mathcal{Y})$, is called the \textbf{marginal polytope} associated to the factor graph $(V, F)$ \citep{wainwright2008graphical}.  

We associate \textbf{unary scores} $\bm{\theta}_V = [\bm{\theta}_v \,:\, v \in V] \in \mathbb{R}^{N_V}$ to  configurations of variable nodes and \textbf{higher-order scores} $\bm{\theta}_F = [\bm{\theta}_f \,:\, f \in F] \in \mathbb{R}^{N_F}$ to configurations of factor nodes. We denote $\bm{\theta} = [\bm{\theta}_V; \bm{\theta}_F] \in \mathbb{R}^{N_V + N_F}$. 
The MAP inference problem is exactly as in \eqref{eq:MAP} where the objective can be written as $\bm{\theta}^\top\bm{y} = \bm{\theta}_V^\top \bm{y}_V + \bm{\theta}_F^\top \bm{y}_F$. 
%now
%\begin{align}\label{eq:map}
%\bm{y}^* &= \argmax_{\bm{y} \in \mathcal{Y}} \bm{\theta}^\top \bm{y} = \argmax_{\bm{y} \in \mathrm{conv}(\mathcal{Y})} \bm{\theta}_V^\top \bm{y}_V + \bm{\theta}_F^\top \bm{y}_F.
%\end{align}
As above, we consider regularized variants of \eqref{eq:MAP} via a convex regularizer $\Omega: \mathrm{conv}(\mathcal{Y}) \rightarrow \mathbb{R}$. 
\textbf{SparseMAP} corresponds to $\Omega(\bm{y}) = \frac{1}{2} \|\bm{y}_V\|^2 + I_{\mathrm{conv}(\mathcal{Y})}(\bm{y})$ (note that only the unary variables are quadratically regularized), which leads to the problem \eqref{eq:sparsemap}.  
The active set algorithm of \citet{niculae2018sparsemap} applies also to this general case,  requiring only a MAP oracle to solve \eqref{eq:MAP}. 

\begin{example}[sequential $k$-subsets]\label{ex:seqksubsets} 
Consider the $k$-subset problem of Example~\ref{ex:ksubsets} but now with a sequential structure. 
This can be represented as a pairwise factor graph $(V, F)$ where $V = \{1, ..., N\}$ and $F = \{(i, i+1)\}_{i=1}^{N-1}$. %We have $N_v = 2$ for every variable node $v \in V$ and $N_f = 4$ for every factor node $f \in F$. 
The budget constraint forces exactly $k$ of the $N$ variable nodes to have the value $1$. %so we must have $
%\sum_{v \in V} \bm{y}_v = [N-k; k]$. 
The set $\mathcal{Y}$ contains all bit vectors satisfying these constraints as well as consistency among the variable and factor assignments. 
\begin{itemize}
\item To each $i \in V$ we assign unary ``emission'' scores $\bm{\theta}_i = [0, s_i] \in \mathbb{R}^2$, corresponding to the states ``off'' and ``on'', respectively. A positive $s_i$ expresses a preference for the state ``on'' and a negative $s_i$ for the state ``off''. 
\item To each factor (edge) $(i, i+1) \in F$ we associate Ising higher-order (pairwise) ``transition'' scores $\bm{\theta}_{(i, i+1)} = [0, 0, 0, t] \in \mathbb{R}^4$, corresponding to state pairs ``off-off'', ``off-on'', ``on-off'', and ``on-on'', respectively.  
To promote consecutive memory items to be both or neither retrieved, we can 
define attractive ``transition'' scores by choosing $t>0$. 
%in addition to the unary ``emission'' scores $\bm{\theta}_i = [0, s_i]$. 
\end{itemize}
The MAP inference problem for this model can be solved with dynamic programming in runtime $\mathcal{O}(Nk)$, and the SparseMAP transformation can be computed with the active set algorithm \citep{niculae2018sparsemap} by iteratively calling this MAP oracle. 
\end{example}

\subsection{Structured Fenchel-Young losses and margins}

Fenchel-Young losses are applicable to the structured case outlined in this section by choosing a regularizer with domain $\mathrm{dom}(\Omega) = \mathrm{conv}(\mathcal{Y})$, instead of the probability simplex $\triangle_N$. 
In the sequel, we focus on the SparseMAP regularizer $\Omega(\bm{y}) = \frac{1}{2}\|\bm{y}_V\|^2 + I_{\mathrm{conv}(\mathcal{Y})}(\bm{y})$. 
The notion of margin in Def.~\ref{def:margin} can  be extended to the structured case \citep[Def.~5]{blondel2020learning}:

\begin{definition}[Structured margin]\label{def:structured_margin}
A loss $L(\bm{\theta}; \bm{y})$ has a \textbf{structured margin} if $\exists 
0 \le  m < \infty
%m \ge 0
$ 
such that   
%s.t.\
$\forall \bm{y} \in \mathcal{Y}$:
\begin{align*}
\bm{\theta}^\top \bm{y} \ge \max_{\bm{y}' \in \mathcal{Y}} \left( \bm{\theta}^\top \bm{y}' + \frac{m}{2}\|\bm{y} - \bm{y}' \|^2 \right) \,\, \Rightarrow \,\, L(\bm{\theta}; \bm{y}) = 0.
\end{align*} 
The smallest such $m$ is called the margin of $L$.
\end{definition}
Note that the definition of of margin in Def.~\ref{def:margin} is recovered when $\mathcal{Y} = \{\bm{e}_1, ..., \bm{e}_N\}$. Note also that, since we are assuming  $\mathcal{Y} \subseteq \{0, 1\}^{N_V + N_F}$, the term $\|\bm{y} - \bm{y}' \|^2$ is a Hamming distance, which counts how many bits need to be flipped to transform $\bm{y}'$ into $\bm{y}$. A well-known example of a loss with a structured separation margin is the structured hinge loss \citep{taskar2003max,tsochantaridis2005large}. 

We show below that the SparseMAP loss has a structured margin (our result, proved in Appendix~\ref{sec:proof_prop_sparsemap_margin}, extends that of \citet{blondel2020learning}, who have shown this only for structures without high order interactions):
\begin{proposition}\label{prop:sparsemap_margin}
Let $\mathcal{Y} \subseteq \{0, 1\}^{N_V + N_F}$ be contained in a sphere, i.e., for some $r>0$, $\|\bm{y}\| = r$ for all $\bm{y} \in \mathcal{Y}$. 
Then:
\begin{enumerate}
\item Without high order interactions, the SparseMAP loss has a structured margin $m=1$. 
\item If there are high order interactions and, for some $r_V$ and $r_F$ with $r_V^2 + r_F^2 = r^2$, we have $\|\bm{y}_V\| = r_V$ and $\|\bm{y}_F\| = r_F$ for any $\bm{y} = [\bm{y}_V; \bm{y}_F]\in \mathcal{Y}$, then the SparseMAP loss has a structured margin $m \le 1$.
\end{enumerate}
\end{proposition}
The assumptions above are automatically satisfied with the factor graph construction in \S\ref{sec:structured_high_order}, with $r_V^2 = |V|$, $r_F^2 = |F|$, and $r^2 = |V| + |F|$. 
For the $k$-subsets example, we have $r^2 = k$, and for the sequential $k$-subsets example, we have $r_V^2 = N$, $r_F^2 = N-1$, and $r^2 = 2N-1$.

\subsection{Guarantees for retrieval of pattern associations}

%\andre{this part is redundant and needs to be revised} 
%The MHN leads to convex combinations of patterns, and in that sense can be seen as a structured form of sparsemax where marginal polytope is $\mathcal{M} = \{\bm{X}\bm{p} : \bm{p} \in \triangle_N\}$. 
%However, we can also define a combinatorial problem with arbitrary marginal polytope $\mathcal{M} \subseteq [0,1]^N$. 
%Some examples:
%\begin{itemize}
    %\item convex hull of binary vectors with exactly $K$ ones (when $K=1$ this is the simplex, already studied)\footnote{With a simple rescaling, this is equivalent to the $k$-subsets problem whose marginal polytope is the capped probability simplex, described in \citep[\S 7.3]{blondel2020learning} }
   %\item we could define a graph over patterns and consider marginal pollytope induced by this graph -- e.g. a similarity graph, a hierarchical memory, or a chain (sequential) graph representing a temporal sequence of patterns. In both cases, a positive edge potential could encourage neighboring patterns to be both retrieved.
%\end{itemize}

We now consider a \textbf{structured HFY network} using SparseMAP. Following the same logic as Propositions~\ref{prop:cccp} and \ref{prop:bounds_cccp}, we obtain the following update rule:
\begin{align}\label{eq:sparsemap_hfy_update_rule}
    \bm{q}^{(t+1)} = \bm{X}^\top \mathrm{SparseMAP}(\beta \bm{X}\bm{q}^{(t)}).
\end{align} 
%
%One particular structure that is considered in this work is the exact K ones budget whose structure consists of vectors of the same size of the number of queries that need to have exact K ones.
%\cite{niculae2020lp} introduce structured differentiable layers that break down a given problem into more manageable subproblems, instantiated as local factors that need to align when overlapped.
%\paragraph{Structured separation margin.} We use the definition from \citep[Def. 5]{blondel2020learning}, where the structured separation margin is the smallest $m$ such that
%\begin{align}
%\bm{\theta}^\top \bm{y} \ge \max_{\bm{y'} \in \mathcal{Y}} \left( \bm{\theta}^\top \bm{y}' + \frac{m}{2}\|\bm{y} - \bm{y'}\|^2\right) \,\, \Rightarrow \,\, L(\bm{\theta}, \bm{y}) = 0.  
%\end{align}
%This is equivalent to the standard separation margin when $\mathcal{Y} = \{\bm{e}_1, \ldots, \bm{e}_N\}$. 
%It is shown  \citep[Prop. 8 and following example]{blondel2020learning} that, assuming that all elements of $\mathcal{Y}$ have fixed norm, the SparseMAP has a structured margin of $m=1$. 
%This property is satisfied by the exactly-$K$ problem, where elements of $\mathcal{Y}$ have norm $\sqrt{K}$. 
In this structured case, we aim to  retrieve not individual patterns but pattern associations of the form $\bm{X}^\top \bm{y}$, where $\bm{y} \in \mathcal{Y}$. Naturally, when $\mathcal{Y} = \{\bm{e}_1, ..., \bm{e}_N\}$, we recover the usual patterns, since $\bm{x}_i = \bm{X}^\top \bm{e}_i$. 
We define the separation of pattern association $\bm{y}_i \in \mathcal{Y}$ from data as $\Delta_i = \bm{y}_i^\top \bm{X} \bm{X}^\top \bm{y}_i - \max_{j \ne i} \bm{y}_i^\top \bm{X} \bm{X}^\top \bm{y}_j$. 
The next proposition, proved in Appendix~\ref{sec:proof_prop_stationary_single_iteration_sparsemap}, states conditions for exact convergence in a single iteration, generalizing Proposition~\ref{prop:separation}. 

\begin{proposition}[Exact structured retrieval]\label{prop:stationary_single_iteration_sparsemap}
Let $\Omega(\bm{y})$ be the SparseMAP regularizer and assume the conditions of Proposition~\ref{prop:sparsemap_margin} hold. 
Let $\bm{y}_i \in \mathcal{Y}$ be such that  
$\Delta_i \ge \frac{D_i^2}{2\beta}$, where $D_i = \max \|\bm{y}_i - \bm{y}_j\| \le 2r$. Then, $\bm{X}^\top \bm{y}_i$ is a stationary point of the Hopfield energy. 
In addition, if ${\bm{q}^{(0)}} ^\top \bm{X}^\top (\bm{y}_i - \bm{y}_j) \ge\frac{D_i^2}{2\beta}$ for all $j\ne i$, then the update rule 
\eqref{eq:sparsemap_hfy_update_rule} converges to the pattern association $\bm{X}^\top \bm{y}_i$ in one iteration. 
Moreover, %if the patterns are normalized with $\|\bm{x}_i\|=M$ for all $i$, and 
if $$\Delta_i \ge \frac{D_i^2}{2\beta} + \epsilon \min \{\sigma_{\max}(\bm{X})D_i, MD_i^2\},$$   
where $\sigma_{\max}(\bm{X})$ is the spectral norm of $\bm{X}$ and $M = \max_k \|\bm{x}_k\|$, then any $\bm{q}^{(0)}$ $\epsilon$-close to  $\bm{X}^\top \bm{y}_i$  will converge to $\bm{X}^\top \bm{y}_i$ in one iteration.
\end{proposition}

Note that the bound above on $\Delta_i$ includes as a particular case the unstructured bound in Proposition~\ref{prop:separation} applied to sparsemax (entmax with $\alpha=2$, which has margin $m = 1/(\alpha-1) = 1$), since for $\mathcal{Y} = \triangle_N$ we have $r = 1$ and $D_i = \sqrt{2}$, which leads to the condition $\Delta_i \ge \beta^{-1} + 2M\epsilon$. 

For the particular case of the $k$-subsets problem (Example~\ref{ex:ksubsets}), we have $r = \sqrt{k}$ and $D_i=\sqrt{2k}$, leading to the condition $\Delta_i  \ge \frac{k}{\beta} + 2Mk\epsilon$. This recovers sparsemax when $k=1$. 

For the sequential $k$-subsets problem in Example~\ref{ex:seqksubsets}, we have $r = 2N-1$. Noting that any two distinct $\bm{y}$ and $\bm{y}'$ differ in at most $2k$ variable nodes, and since each variable node can affect 6 bits (2 for $\bm{y}_V$ and 4 for $\bm{y}_F$), the Hamming distance between  $\bm{y}$ and $\bm{y}'$ is at most $12k$, therefore we have $D_i = \sqrt{12k}$, which leads to the condition $\Delta_i  \ge \frac{6k}{\beta} + 12Mk\epsilon$.  

\begin{comment}
\andre{comment on why this also works with the normalization} 

\paragraph{Extra normalization step.} 
Similarly to the sparse HFY case, for which Proposition~\ref{prop:extension_normalization} shows that convergence in one iteration still holds when  idempotent post-transformations (such as $\ell_2$-normalization and layer normalization) are applied, a similar results holds too for the structured case. 
\end{comment}

\section{Mechanics of Memory Retrieval Modeling}
\label{sec:memory_retrieval_modeling}

Memory is both a foundational paradigm in human cognitive psychology and a core focus of systems neuroscience in animal models \citep{Eichenbaum2017}. Ongoing research aims to integrate these levels of investigation into a comprehensive account of memory processing in the brain. From this integrative perspective, we describe how the framework described in \S\ref{sec:probabilistic} and \S\ref{sec:structured} provides a theoretical platform for generatively modeling memory retrieval while retaining formal guarantees regarding memory capacity and convergence. We examine these models in the context of recall paradigms used to study human memory.

Recall tasks for humans are crucial in elucidating the structure of memory and retrieval processes \citep{Tulving1972}. The most basic paradigm, associative recall, involves the learning of paired items, where one item serves as a cue to retrieve the associated item which is typically a corrupted or partial version of the target item. \textbf{Sequential recall} requires learning a sequence of items in a specific order. During retrieval, the initial item  serves as a cue, and the individual must recall the subsequent items in the exact order of presentation. \textbf{Free recall} involves the learning of a list of items presented without a specific sequence. During retrieval, individuals use a contextual cue to recall the items in any order, though all items must eventually be retrieved. 
In computational neuroscience, models have been proposed with varying degrees of biological precision which suggest that auto-associative and hetero-associative attractor dynamics in the hippocampal formation subserve memory recall in these paradigms \citep{Naim2020, Boboeva2021}. However, such models operate in a simplified setting focusing on binary, orthogonalized, and relatively low-dimensional memory patterns without broader theoretic guarantees for memories of the scale and complexity encoded by humans.

In order to address this gap, and inspired by this computational cognitive neuroscience approach for investigating memory recall, we apply our framework in deriving efficient algorithms for modeling memory retrieval with the latter two specified paradigms.
\subsection{Free recall}
Consider the sparsemax transformation presented in \S\ref{sec:background}:
\begin{align}\label{eq:sparsemax}
    \text{sparsemax}(\bm{\theta}) := \argmax_{\bm{y} \in \triangle_N} \bm{\theta}^\top \bm{y} - \frac{1}{2}\|\bm{y}\|^2.
\end{align}
\begin{algorithm}[t]
\small
\caption{Free recall with constrained sparsemax. %We apply an Hopfield update rule with an enforced upper bound on probabilities. 
$\bm{X} \in \mathbb{R}^{N \times D}$ represents the memory and $\bm{q} \in \mathbb{R}^D$ denotes the query, initialized as an arbitrary cue. $T$ is the number of inner associative Hopfield iterations. $N$ is the number of memory patterns, which equals the number of outer iterations during the recall process.}
\label{alg:csparsemax_free_recall}
\begin{algorithmic}[1]
\Require $\bm{X}$, $\bm{q}$, $\beta$, $T$
\State $\bm{u} \gets \mathbf{1}_N$ \Comment{Initialize upper bounds}
\For{$i \gets 1$ to $N$}
    \For{$j \gets 1$ to $T$}
        \State $\bm{\theta} \gets \bm{X} \bm{q}$ \Comment{Scores}
        \State $\bm{p} \gets \text{csparsemax}(\beta \bm{\theta}; \bm{u})$
        \State $\bm{q} \gets \bm{X}^\top \bm{p}$ \Comment{Hopfield update}
    \EndFor
    \State $\bm{u} \gets \bm{u} - \bm{p}$ \Comment{Upper bound probabilities}
    %\State $\bm{u} \gets \text{clamp}(\bm{u}, \text{min}=0.)$
\EndFor
\end{algorithmic}
\end{algorithm}
\begin{algorithm}[t]
\small
\caption{Free recall with penalized $\alpha$-entmax. The parameter $\lambda$ corresponds to the penalty applied to the moving average while $\tau$ corresponds to the decay rate.}
\label{alg:penalized_free_recall}
\begin{algorithmic}[1]
\Require $\bm{X}$, $\bm{q}$, $\lambda$, $\tau$, $\beta$, $T$
\State $\bm{a} \gets \mathbf{0}_N$ \Comment{Initial average}
\For{$i \gets 1$ to $N$}
    \State $\bm{p} \gets \alpha\text{-entmax}\left ( \beta  \left ( \bm{X} \bm{q} - \lambda \bm{a} \right) \right)$ \Comment{Penalized probabilities}
    \State $\bm{a} \gets \tau\bm{p} + (1 - \tau) \bm{a}$ \Comment{Exponentially weighted average}
    \State $\bm{q} \gets \bm{X}^\top \bm{p}$ \Comment{Outer Hopfield update}
    \For{$j \gets 1$ to $T$}
        \State $\bm{p} \gets \alpha\text{-entmax}(\beta\bm{X} \bm{q})$ 
        \State $\bm{q} \gets \bm{X}^\top \bm{p}$ \Comment{Inner Hopfield update}
    \EndFor
\EndFor
\end{algorithmic}
\end{algorithm}
These projections often reach the boundary of the simplex, resulting in a sparse probability distribution. However, they are not well-suited for modeling recall paradigms like free recall, as these models primarily retrieve the memory closest to the query without keeping a record of previously attended memories. This behaviour leads to potential repetitions and failure to attend to all distinct memories. One way to address this issue is to set an upper bound on the maximum probability which can be assigned to memories have already been attended to.

In this work, inspired by the concept of \textbf{constrained sparsemax} \citep{malaviya2018sparse}, we show that a modified version of sparse HFY networks can be used for modeling the free recall memory paradigm. Formally, constrained sparsemax is defined as:
\begin{align}\label{eq:csparsemax
}
\text{csparsemax}(\bm{\theta}; \bm{u}) := \argmax_{\bm{y} \in \triangle_N} \bm{\theta}^\top \bm{y} - \frac{1}{2}\|\bm{y}\|^2 \quad \text{s.t.} \quad \bm{y} \leq \bm{u},
\end{align}
where $\bm{u} \in [0, 1]^N$ is a vector of upper bounds. This transformation closely resembles the \text{sparsemax} function but introduces bounded probabilities defined by $\bm{u}$. \citet{malaviya2018sparse} developed efficient forward and backward propagation algorithms for this transformation, making it practical for various applications. It can be useful for tasks such as modeling the free recall memory paradigm with Hopfield models. This process, described in Algorithm~\ref{alg:csparsemax_free_recall}, involves an inner loop of associative recall and an outer loop where the constrained sparsemax transformation is used for keeping track of the attended memories by upper-bounding how much probability mass can be given to patterns that have already been attended to.

%Given the sparse nature of constrained sparsemax, Alg.~\ref{alg:csparsemax_free_recall} may exhibit an overly perfectionist behavior when the query exactly matches one of the memories\dan{i dont understand. exact retrieval is the problem? or its repeated memories? it seems Alg 2 is aimed at addressing repetitivity in the the base sparsemax algo with history penalty while Alg 1 (c-sparsemax) avoids repetition by making an explicit list of memories and enforcing prob limits?} as sparsemax induces exact retrieval, as shown in \S\ref{sec:probabilistic}. 
%This attempt to model free recall can diverge from human behavior, since humans can usually avoid repeating previously recalled memories.
%which often involves forgetting.
%Incorporating insights  \Saul{Daniel!!!!!!}\dan{I think this should be conceptualized at the mechanistic level as neural adaptation rather than a forgetting mechanism at the cognitive level. forgetting usually means the pattern is lost forever, while here we want to avoid it being reselected in the current free recall trial. i discuss this in the entorhinal-hippocampal context here in this commented ref} from neuroscience, we can introduce mechanisms that mimic the natural inhibitory process of neural adaptation.
% another possible mechanism is ``modulated inhibition'', see here https://www.frontiersin.org/journals/computational-neuroscience/articles/10.3389/fncom.2015.00149/full however this requires extrinsic input so I think neural adaptation is more parsimonious in the current context

An alternative option (Algorithm~\ref{alg:penalized_free_recall}) is to introduce a penalty mechanism which reduces the scores of attended patterns in subsequent Hopfield iterations. This penalty mechanism discourages the selection of previously chosen memories, aiming to model the non-repetitive functionality of human memory processing during free recall.
%an imperfect human memory process that involves forgetting. 
A potential penalty is the exponentially weighted average, which induces forgetting by dynamically balancing the influence of current and past penalties, encouraging diversity in selections. 

\subsection{Sequential recall}
\begin{algorithm}[t]
\small
\caption{Sequential recall using SparseMAP with sequential $2$-subsets. $t>0$ denotes the transition score and $\omega \ge 1$ is a coefficient which promotes sequential order by boosting the emission score of the last recalled pattern.}
\label{alg:penalized_seq_recall}
\begin{algorithmic}[1]
\Require $\bm{X}$, $\bm{q}$, $\lambda$, $\tau$, $\omega$, $t$, $\beta$, $T$
\State $\bm{a} \gets \mathbf{0}_N$
\For{$i \gets 1$ to $N$}
    \State $\bm{y} \gets \textsc{sequential\_k\_subsets}\left(\beta \left( \bm{X} \bm{q} - \lambda \bm{a} \right), k=2, t\right)$\Comment{Sequential $2$-subsets}
    \State $\bm{q} \gets \bm{X}^\top \bm{y} - \bm{q}$\Comment{Outer Hopfield update}
    \For{$k \gets 1$ to $T$}
        \State $\bm{p} \gets \alpha \text{-entmax}(\bm{X} \bm{q})$\Comment{Inner Hopfield update}
        \State $\bm{q} \gets \bm{X}^\top \bm{p}$
    \EndFor
    \State $\bm{a} \gets \tau (\bm{y} - \omega \bm{p}) + (1 - \tau) \bm{a}$ \Comment{Exponentially weighted average}
\EndFor
\end{algorithmic}
\end{algorithm}
We now consider the sequential recall paradigm, deriving an algorithm inspired by the penalized free recall approach (Algorithm~\ref{alg:penalized_free_recall}), but which leverages the structured Hopfield networks presented in \S\ref{sec:structured}. 
We consider the sequential $k$-subsets model described in Example~\ref{ex:seqksubsets} with large transition scores (\textit{i.e.}, choosing a large $t>0$), so that we strongly encourage the retrieval of consecutive memory patterns. 
This structured transformation is used in the outer loop, operating with $k=2$, which promotes sequential top-$2$ retrieval. This encourages retrieving a pattern association involving two adjacent memory patterns, the cue (associated with the initial query) and the succeeding pattern. %where the structured Hopfield output may be a combination of the initial query (single memory) with either the preceding or succeeding memory pattern, with the latter being the desired intermediate step in the current outer iteration. 

The algorithm is presented as Algorithm~\ref{alg:penalized_seq_recall}. At each step in the outer loop, the structured Hopfield network with the sequential $k$-subsets model is first queried with $\bm{q}$ (which we would like to be close to some pattern, $\bm{q} \approx \bm{x}_i$) and returns a pattern association $\bm{y}$---ideally, this is a two-hot vector indicating the index of the cue pattern and the index of an adjacent pattern, e.g., $\bm{x}_{i+1}$; that is, $\bm{y} \approx \bm{x}_i + \bm{x}_{i+1}$. 
In reality it can be a fractional vector satisfying $\mathbf{1}^\top \bm{y} = 2$. 
In the ideal scenario we have $\bm{X}^\top \bm{y} \approx \bm{x}_i + \bm{x}_{i+1} \approx \bm{q} + \bm{x}_{i+1}$. 
(Note that in this structured Hopfield update we use a penalty similar to the one used in the penalized free recall algorithm, which we will come back to later.) 
Then, we subtract the query $\bm{q}$ from $\bm{X}^\top \bm{y}$---in the ideal scenario above, this should be close to $\bm{q} + \bm{x}_{i+1} - \bm{q} = \bm{x}_{i+1}$. 
This becomes the query to the inner Hopfield loop, which we expect to yield the attractor $\bm{x}_{i+1}$---the next pattern to be recalled, and the cue for the next step.
The penalties are updated with the difference $\bm{y} - \omega \bm{p}$, where we expect $\bm{y} \approx \bm{e}_i + \bm{e}_{i+1}$ and $\bm{p} \approx \bm{e}_{i+1}$. If $\omega = 1$, this difference would be close to $\bm{e}_i$, the indicator of the $i\textsuperscript{th}$ pattern, which will be penalized in subsequent iterations to avoid being retrieved twice. 
By choosing $\omega$ slightly greater than one, we also add a small bonus to the $(i+1)\textsuperscript{th}$ pattern, which we would like to be retrieved as part of the pattern association in the next step---this avoids memory jumps. 
%To achieve this, the function operates over scores, produced with the current query (or current memory), penalized by an exponentially weighted average. The goal of this penalization is to lower the score corresponding to the first memory triggered by the sequential $2$-subsets, thus penalizing, in the next outer iteration, the retrieval of the association made by the intended memory output of the current outer iteration and the previous memory item. This exponentially weighted average also adds a bonus (by setting $\omega > 1$) to the outputted and intended memory item which promotes the retrieval of associations that include this pattern, thus penalizing memory jumps. 
%The resulting output (or next memory item) is obtained by subtracting the current query (or next memory item) from the Hopfield outer output computed via $2$-subsets. This subtraction may remove the first memory of the retrieved pattern association, potentially leading to the retrieval of the next memory pattern. The resulting algorithm can be found in Alg.~\ref{alg:penalized_seq_recall} and an illustration of this method can be seen in Fig.~\ref{fig:seq_recall_image}. 
%An illustrative experiment with this method is shown in Fig.~\ref{fig:seq_recall_image} for the MNIST, CIFAR10, and ImageNet datasets. \andre{add 1-2 sentences commenting the experiment}
\subsection{Empirical evaluation}
We evaluate the algorithms derived in the current section using three datasets: MNIST \citep{lecun1998gradient}, CIFAR10 \citep{krizhevsky2009learning}, and Tiny ImageNet \citep{le2015tiny}. We use a maximum of 20 Hopfield (or inner) iterations. For the penalized free and sequential recall, we employ a penalty of $\lambda=10^9$ and a decay rate of $\tau=0.001$. For the sequential recall algorithm, we use a transition score of $10^8$ with the $k=2$ for the sequential $k$-subsets and $\omega=1.1$. The metric used to evaluate the algorithms is the unique memory ratio, which measures the proportion of distinct memories recalled. 
\begin{figure*}[t]
    \centering
    \includegraphics[width=1\textwidth]{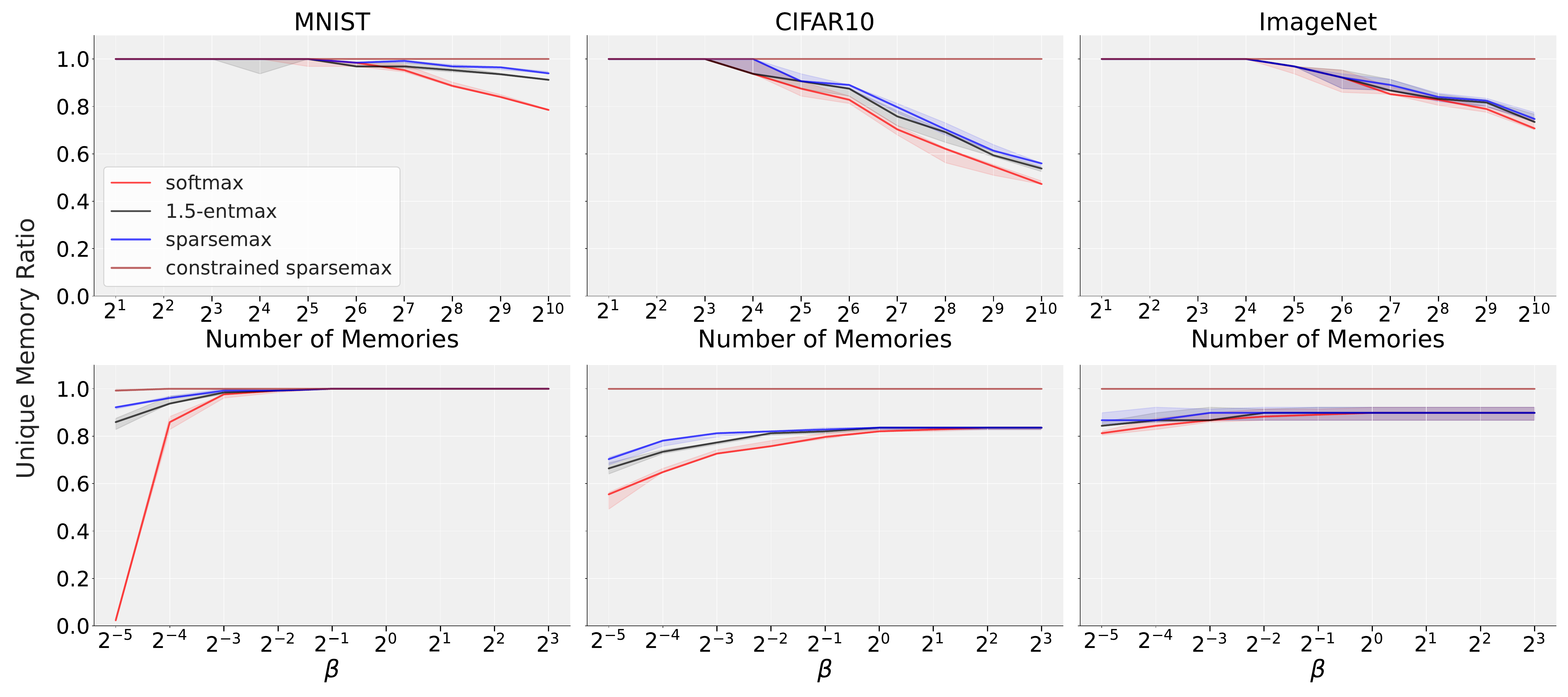}
    \caption{(Top) Memory capacity in terms of unique, non-repeated memories using various free recall methods for different numbers of stored memories with \(\beta = 0.1\).
(Bottom) Unique memory ratio as a function of \(\beta\) for a memory size of 128. Plotted are the medians over 5 runs with different memories and the interquartile range.}
    \label{fig:free_recall}
\end{figure*}
\begin{figure*}[t]
    \centering
    \includegraphics[width=0.7\textwidth]{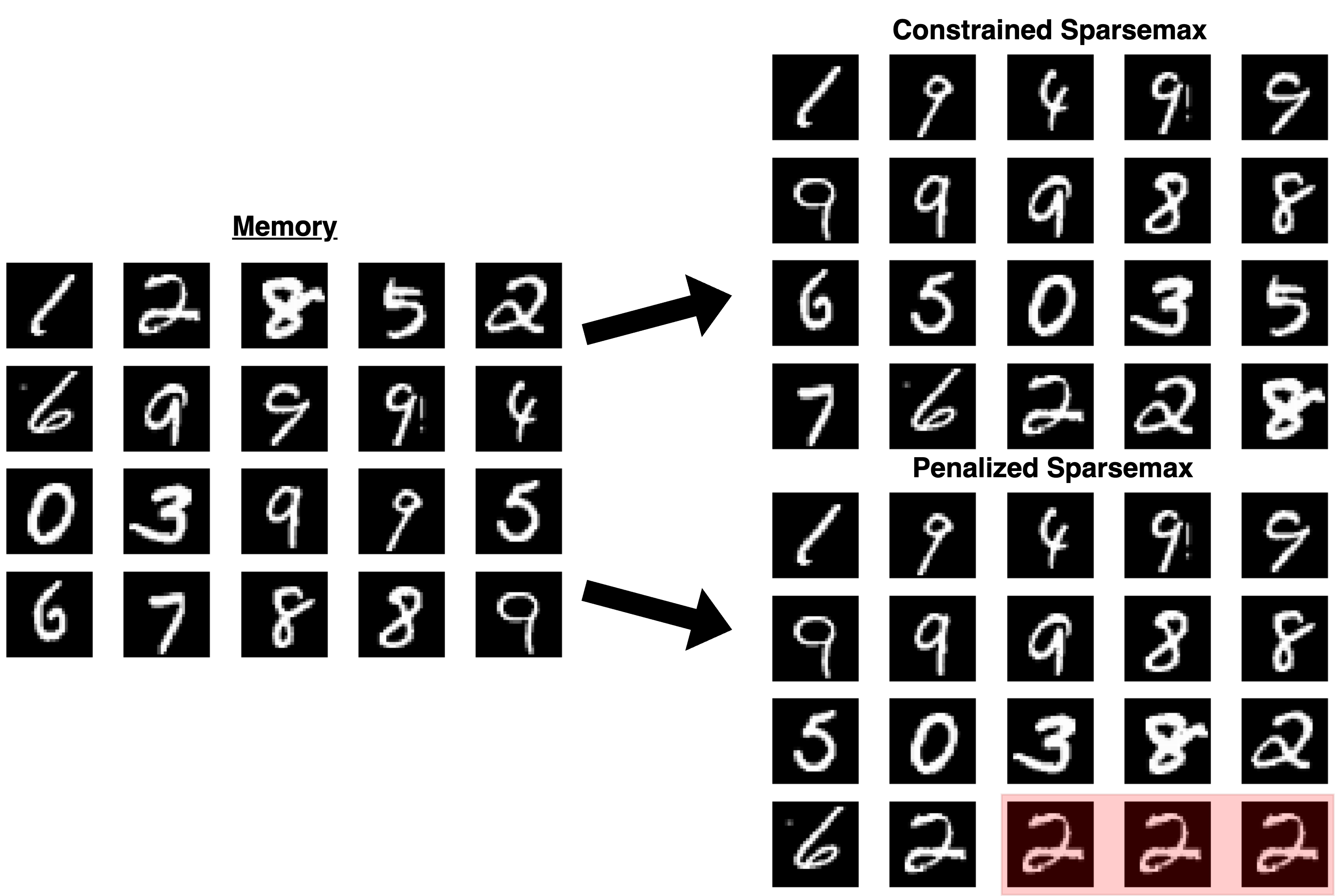}
    \caption{Simulation of \textbf{free recall} using our two methods on MNIST \citep{lecun1998gradient}: \textbf{constrained sparsemax} (Algorithm~\ref{alg:csparsemax_free_recall}) and \textbf{penalized sparsemax} (Algorithm~\ref{alg:penalized_free_recall}). For both methods, we set the number of Hopfield iterations to $T=5$. In the penalized free recall method, we apply a penalty of $\lambda = 10^{8}$ and a decay rate of $\tau=0.001$. In both case, we set  $\beta=0.1$. 
    %To enhance visualization and interpretability, we set $\beta=0.1$, allowing the differences between the methods to be more apparent, although both remain accurate. 
    Red highlight corresponds to repeated memories.}
    \label{fig:free_recall_image}
\end{figure*}
We illustrate the effectiveness of the constrained and penalized free recall methods in Figure~\ref{fig:free_recall}. As expected, performance decreases as the number of stored memory items increases, with this effect being particularly noticeable for softmax due to its dense nature. Performance also improves with higher values of $\beta$, as the transformation becomes sparser. Despite this behavior, constrained sparsemax shows near-optimal performance across different numbers of memories and $\beta$ values. The penalized $\alpha$-entmax transformations, which are biologically more plausible, work effectively for smaller memory sizes, since they are able to forget the already attended memories through the penalty using the exponentially weighted average, but they degrade as the number of memories increases (with the sparse transformations performing better than softmax). Figure \ref{fig:free_recall_image} which shows the perfect behavior of constrained sparsemax and the brief repetitiveness of penalized sparsemax. 
\begin{figure*}[t]
    \centering
    \includegraphics[width=1\textwidth]{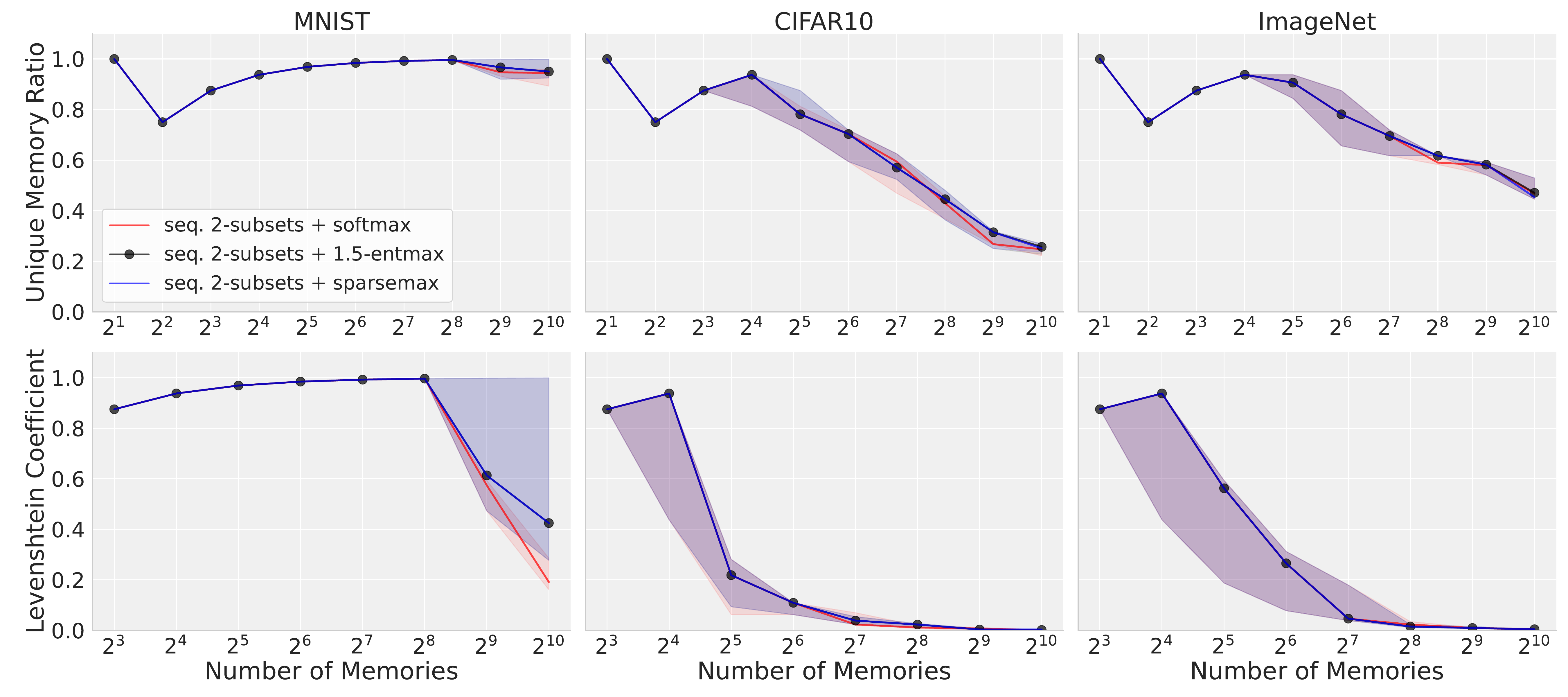}
    \caption{(Top) Memory capacity in terms of unique, non-repeated memories using the sequential recall for different numbers of stored memories with \(\beta = 0.1\).
(Bottom) Levenshtein coefficient as a function of number of memories. Plotted are the medians over 5 runs with different memories and the interquartile range.}
    \label{fig:seq_recall}
\end{figure*}
\begin{figure*}[t]
    \centering
    \includegraphics[width=0.7\textwidth]{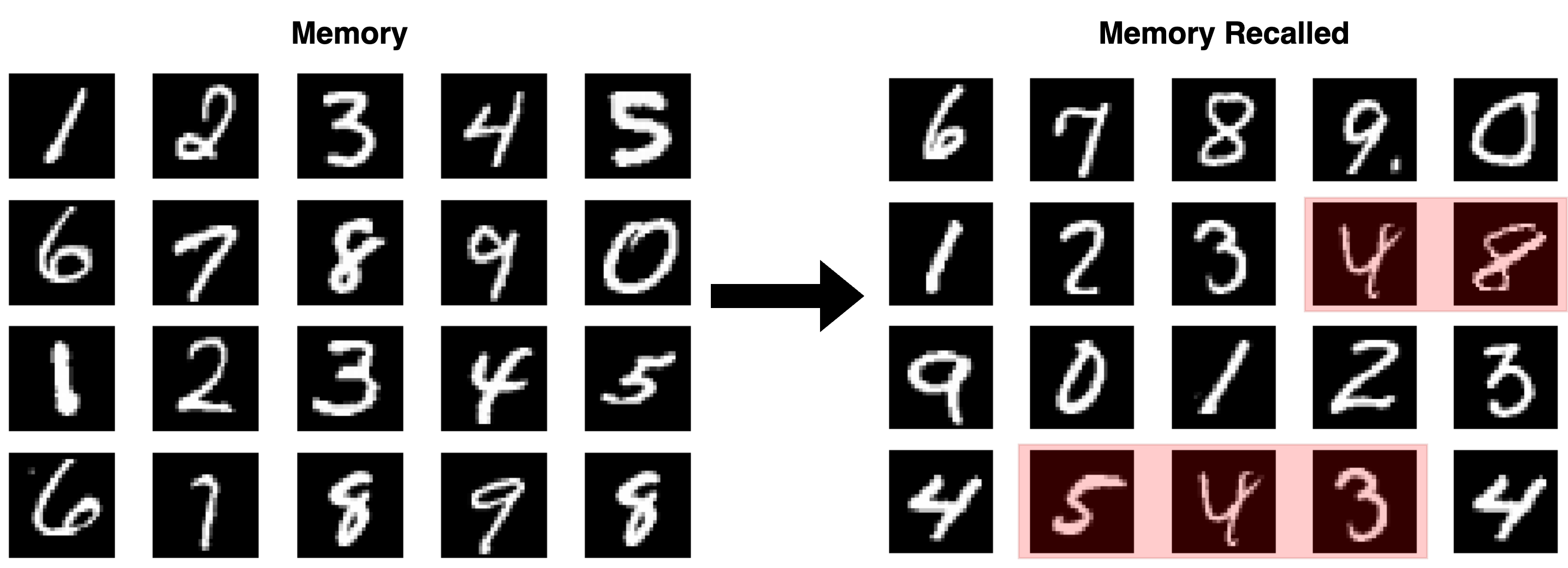}
    \caption{We illustrate \textbf{sequential recall} using Algorithm \ref{alg:penalized_seq_recall} on the MNIST dataset \citep{lecun1998gradient} for sparsemax. We set the transition scores to $t=10^5$, the number of inner iterations to $T=100$, the penalty coefficient to $\lambda=10^9$, and the decay rate to $\tau=0.001$. We set $\beta=0.1$ and $\omega=1.1$.  Red highlight corresponds to memory jumps. }
    \label{fig:seq_recall_image}
\end{figure*}
In Figure~\ref{fig:seq_recall}, a similar behavior is observed for the sequential recall paradigm in the first row, where the $\alpha$-entmax methods show competitive performance, for $\alpha \in \{1, 1.5, 2\}$. In the second row, we evaluate the quality of the generated sequence by measuring the Levenshtein coefficient as a function of the number of memories. This coefficient is computed as $1-\frac{D}{C}$, where $D$ is the Levenshtein distance and $C$ the sequence length. Even with the inclusion of the parameter $\omega$, the method still exhibits a tendency to jump between positions in memory, especially for larger memory sizes, which leads to the generation of multiple subsequences rather than reconstructing the full original sequence, as can be seen in Figure~\ref{fig:seq_recall_image}. Indeed, such ``jumpy'' dynamics are reminiscent of superdiffusive forms of hippocampal replay observed when animals thought to be reflective of parsimonious algorithms for sampling large memory structures \citep{McNamee2021}. This behavior results in fragmented outputs where the model captures and rearranges parts of the sequence as distinct units. Such ``block" jumps, where the model effectively skips over certain parts of the sequence, are not adequately handled by the Levensthein distance and other known metrics. As expected, we observe a decrease in performance as the number of memories increases, despite empirical verification that the individual elements within the subsequence blocks are retrieved in the correct order. Nonetheless, softmax tends to be worse than the remaining methods, as expected. %, due to its dense behavior.

\section{Experiments}\label{sec:experiments}
We now present experiments using both synthetic and real-world datasets to validate our theoretical findings in \S\ref{sec:probabilistic} and \S\ref{sec:structured}. These experiments demonstrate the practical benefits of our $\hat{\bm{y}}_\Omega$ functions, which result in sparse and structured Hopfield networks, and our $\hat{\bm{y}}_\Psi$ functions, which enable post-transformations like normalization and layer normalization.

\subsection{Metastable state distributions in MNIST}
\begin{table*}[t]
\small
    \centering
    \caption{Distribution of metastable state %cardinalities 
    (in $\%$) in MNIST. The %entire 
    training set %MNIST dataset 
    is memorized and the %entire 
    test set is used as queries.}
    \label{tab:metastable}
    \vspace{-0.1cm}
    \resizebox{\textwidth}{!}{%
    \begin{tabular}{c|S[table-format=2.1] S[table-format=1.1] S[table-format=1.1] S[table-format=1.1] S[table-format=1.1] S[table-format=1.1]  S[table-format=2.1]  S[table-format=1.1]| S[table-format=1.1] S[table-format=1.1] S[table-format=1.1] S[table-format=1.1] S[table-format=2.1] S[table-format=1.1] S[table-format=1.1] S[table-format=1.1] S[table-format=1.1] S[table-format=1.1]}
    \toprule
    Metastable & \multicolumn{8}{c}{$\beta = 0.1$} & \multicolumn{8}{c}{$\beta = 1$} \\
     State& \multicolumn{3}{c}{$\alpha$-entmax} & \multicolumn{2}{c}{$\gamma$-normmax} & \multicolumn{3}{c|}{$k$-subsets} & \multicolumn{3}{c}{$\alpha$-entmax} & \multicolumn{2}{c}{$\gamma$-normmax} & \multicolumn{3}{c}{$k$-subsets} \\
     Size & {1} & {1.5} & \multicolumn{1}{c|}{2} & {2} & \multicolumn{1}{c|}{5} & {2} & {4} & \multicolumn{1}{c|}{8} & {1} & {1.5} & \multicolumn{1}{c|}{2} & {2} & \multicolumn{1}{c|}{5} & {2} & {4} & {8} \\
    \midrule
        
    1 & 3.5 & 69.2 & 88.1 & 81.4 & 51.4& 0.0 & 0.0 & 0.0 & 97.8 & 99.9 & 100.0 & 100.0 &99.8 & 0.0 & 0.0 & 0.0 \\
    2 & 2.1 & 8.6 & 5.2 & 6.7 & 31.4& 87.3 & 0.0 & 0.0 & 0.9 & 0.1 & 0.0 & 0.0 & 0.2&99.9 & 0.0 & 0.0 \\
    3 & 1.6 & 3.9 & 2.6 & 1.9 & 7.0& 6.1 & 0.0 & 0.0 & 0.4 & 0.0 & 0.0 & 0.0 & 0.0&0.1 & 0.0 & 0.0 \\
    4 & 1.2 & 2.3 & 1.6 & 1.0 & 2.1& 2.5 & 80.0 & 0.0 & 0.3 & 0.0 & 0.0 & 0.0 &0.0 &0.0 & 99.3 & 0.0 \\
    5 & 1.2 & 1.6 & 1.1 & 0.9 &1.5& 2.0 & 11.9 & 0.0 & 0.2 & 0.0 & 0.0 & 0.0 & 0.0&0.0 & 0.7 & 0.0 \\
    6 & 0.9 & 0.9 & 0.8 & 0.5 &1.5& 1.1 & 4.4 & 0.0 & 0.1 & 0.0 & 0.0 & 0.0 & 0.0& 0.0&  0.1 & 0.0 \\
    7 & 1.1 & 0.6 & 0.4 & 0.4 & 1.3&0.6 & 2.1 & 0.0 & 0.1 & 0.0 & 0.0 & 0.0 &0.0 & 0.0 & 0.0 & 0.0 \\
    8 & 0.8 & 0.6 & 0.1 & 0.8 & 1.0&0.2 & 1.0 & 60.0 & 0.1 & 0.0 & 0.0 & 0.0 & 0.0& 0.0 & 0.0 & 95.0 \\
    9 & 1.0 & 0.3 & 0.0 & 0.5 &0.8& 0.1 & 0.4 & 26.0 & 0.1 & 0.0 & 0.0 & 0.0 & 0.0& 0.0 & 0.0 & 4.7 \\
    10 & 1.1 & 0.1 & 0.0 & 0.5 &0.6& 0.0 & 0.1 & 9.2 & 0.1 & 0.0 & 0.0 & 0.0 & 0.0& 0.0 & 0.0 & 0.2 \\
    10$^+$ & 85.5 & 11.9 & 0.1 & 5.4 & 1.4 & 0.1 & 0.0 & 4.8 & 0.1 & 0.0 & 0.0 & 0.0 &0.0 & 0.0 & 0.0 & 0.1 \\
        \bottomrule
    \end{tabular}
    }
\end{table*}
We start by investigating how often our Hopfield networks converge to metastable states, an important aspect for understanding the network's dynamics. To elucidate this, we examine  $\hat{\bm{y}}_\Omega(\beta\bm{X}\bm{q}^{(t)})$ for the MNIST dataset \citep{lecun1998gradient}, probing the number of nonzeros in these vectors. We set a threshold $>0.01$ for the softmax method (1-entmax). For the sparse transformations we do not need a threshold, since they have exact retrieval. 

%We meticulously count and present the distributions of these probabilities in Table \ref{tab:metastable}. 
Results in Table~\ref{tab:metastable} suggest that $\alpha$-entmax is capable of retrieving single patterns for higher values of $\alpha$. Despite $\gamma$-normmax's ability to induce sparsity, we observe that as $\gamma$ increases, the method tends to stabilize in small but persistent metastable states. This behavior aligns with theoretical expectations, as it favors a uniform distribution over some patterns. On the other hand, 
SparseMAP with $k$-subsets 
%shows a capability 
is capable of retrieving sparse pattern associations of $k$ patterns, as expected. For $\beta = 1$, we observe that all methods yield sparse distributions, which can be attributed to the inherently sparse nature of the MNIST dataset, where the majority of the pixels are background (represented by zeros), resulting in a high degree of sparsity in the data.
\subsection{Hopfield dynamics and basins of attraction}\label{sec:hopfield_dynamics}

\begin{figure*}[t]
  \centering
  \includegraphics[width=0.5\textwidth]{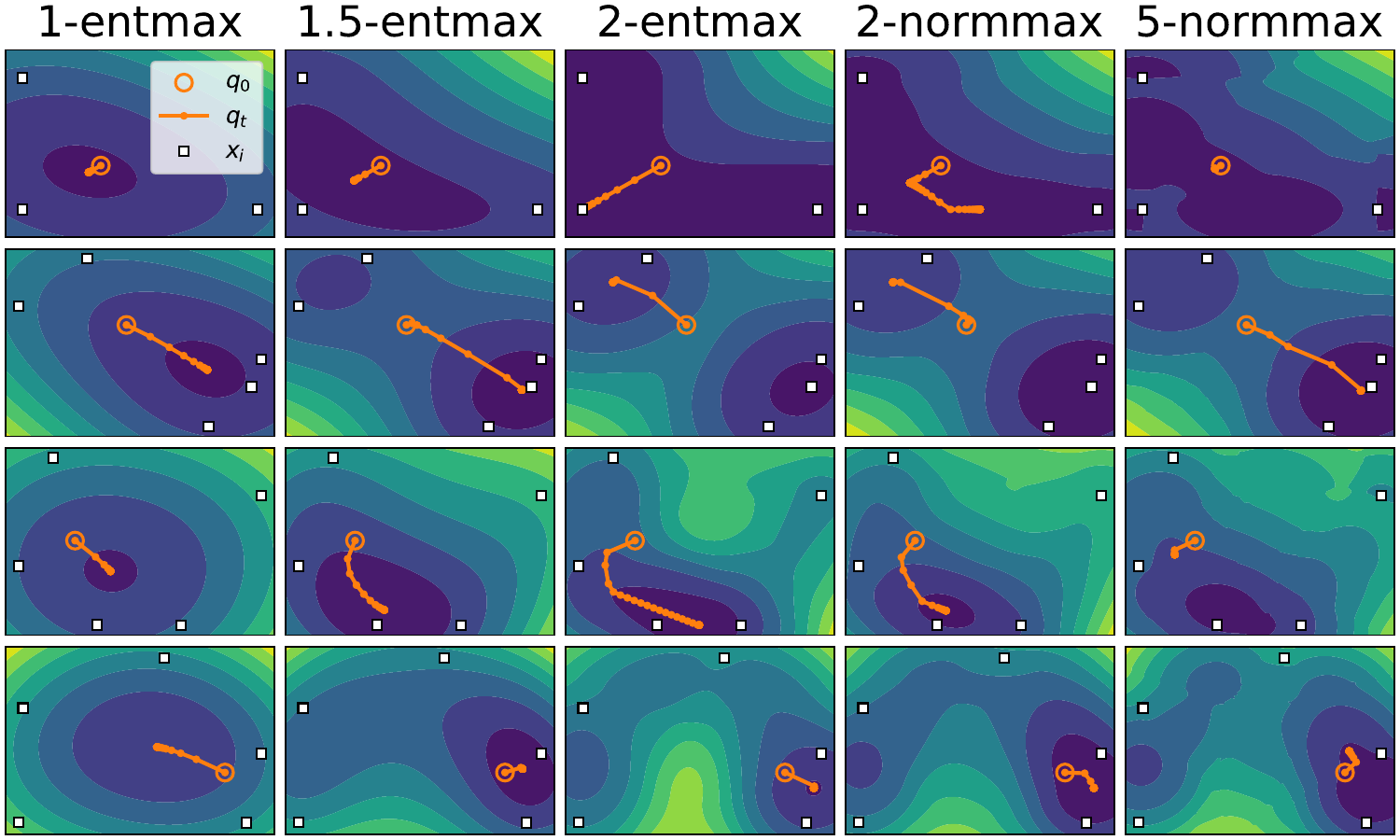}\hspace{2pt}%
  \includegraphics[width=0.49\textwidth]{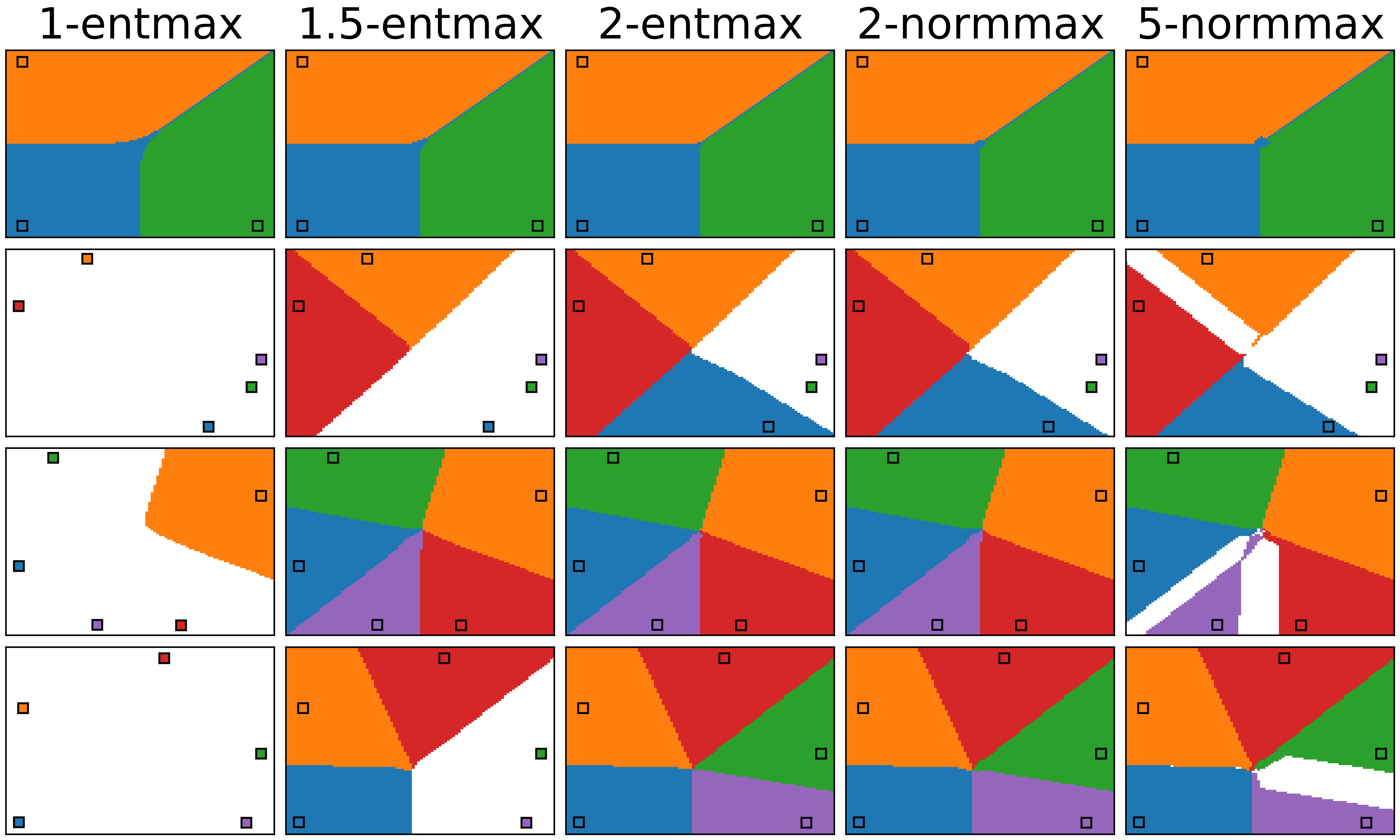}
  \caption{Left: contours of the energy function and optimization trajectory of the CCCP iteration ($\beta = 1$) for $\hat{\bm{y}}_\Psi(\bm{z}) = \bm{z}$. Right: attraction basins associated with each pattern ($\beta=10$; a larger $\beta$ is needed to allow for the $1$-entmax to get $T$-close to a single pattern). White sections converge to a metastable state; for $\alpha = 1$ we allow a tolerance of $T = .01$).}
  \label{fig:overall}
\end{figure*}

\begin{figure*}[t]
  \centering
  \includegraphics[width=0.5\textwidth]{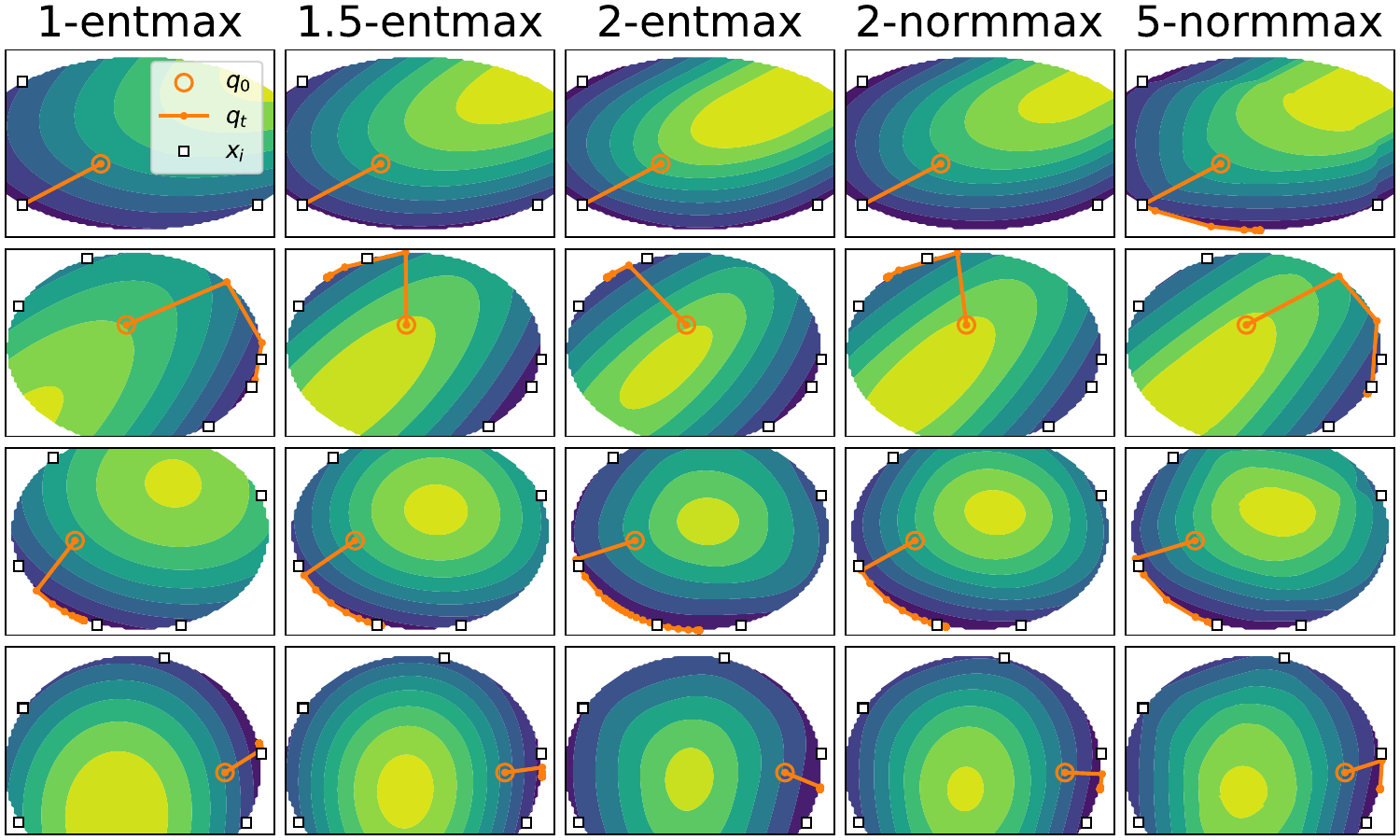}\hspace{2pt}%
  \includegraphics[width=0.49\textwidth]{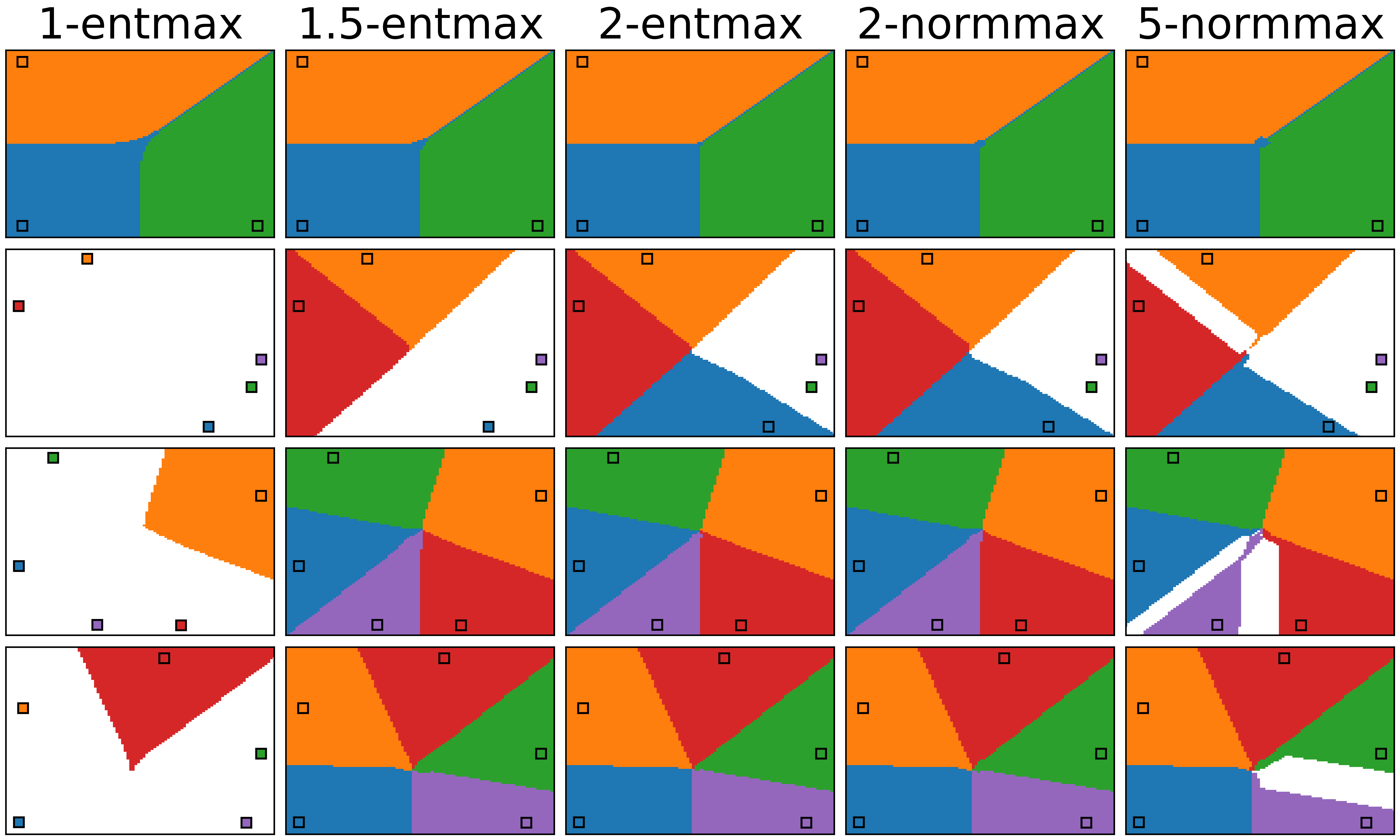}
  \caption{Left: contours of the energy function and optimization trajectory of the CCCP iteration ($\beta = 1$) for $\hat{\bm{y}}_\Psi(\bm{z}) = \frac{\bm{z}}{||\bm{z}||}$. The white regions correspond to infinite energies, due to the hard constraints. Right: attraction basins associated with each pattern ($\beta=10$).}
  \label{fig:overall11}
\end{figure*}

\begin{figure*}[t]
  \centering
  \includegraphics[width=0.5\textwidth]{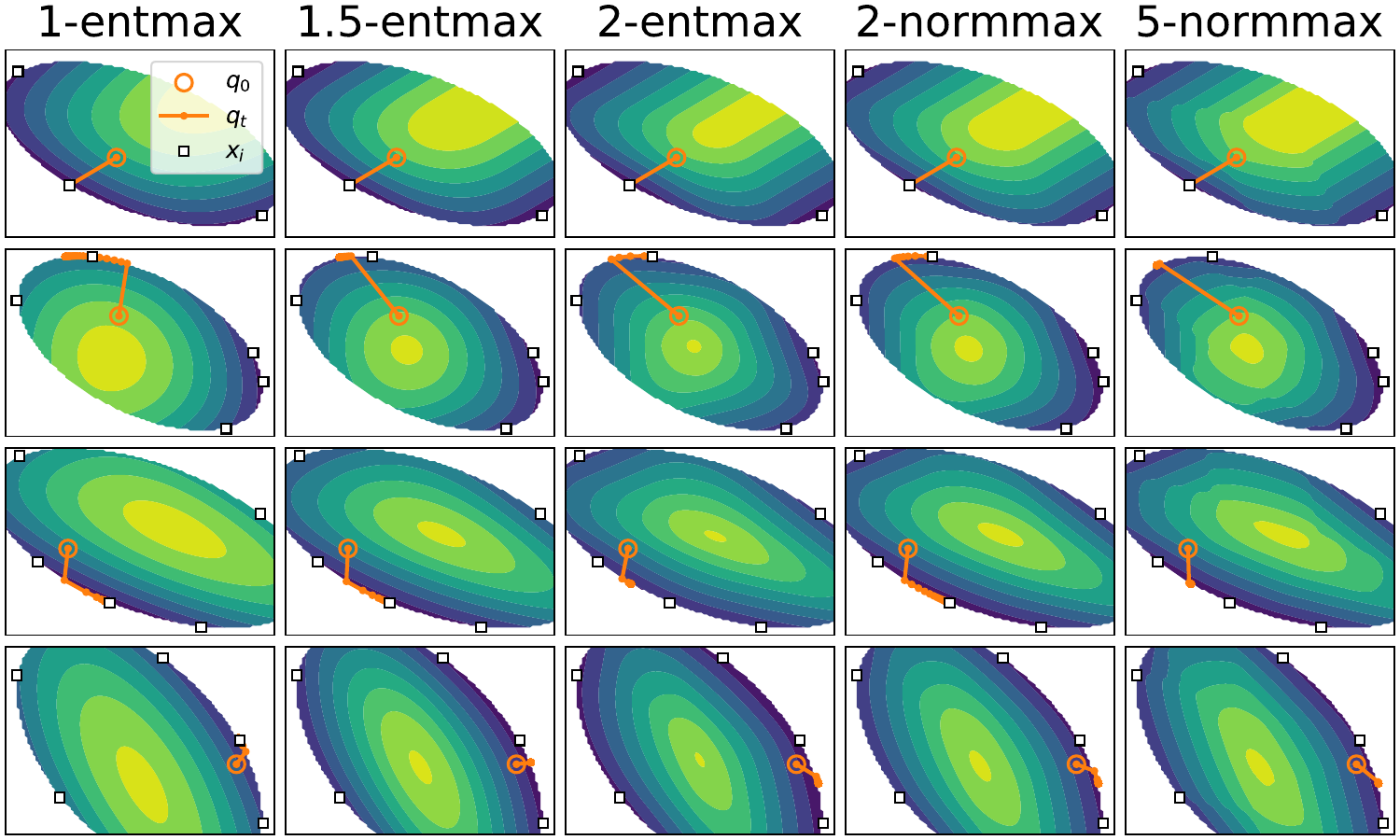}\hspace{2pt}%
  \includegraphics[width=0.49\textwidth]{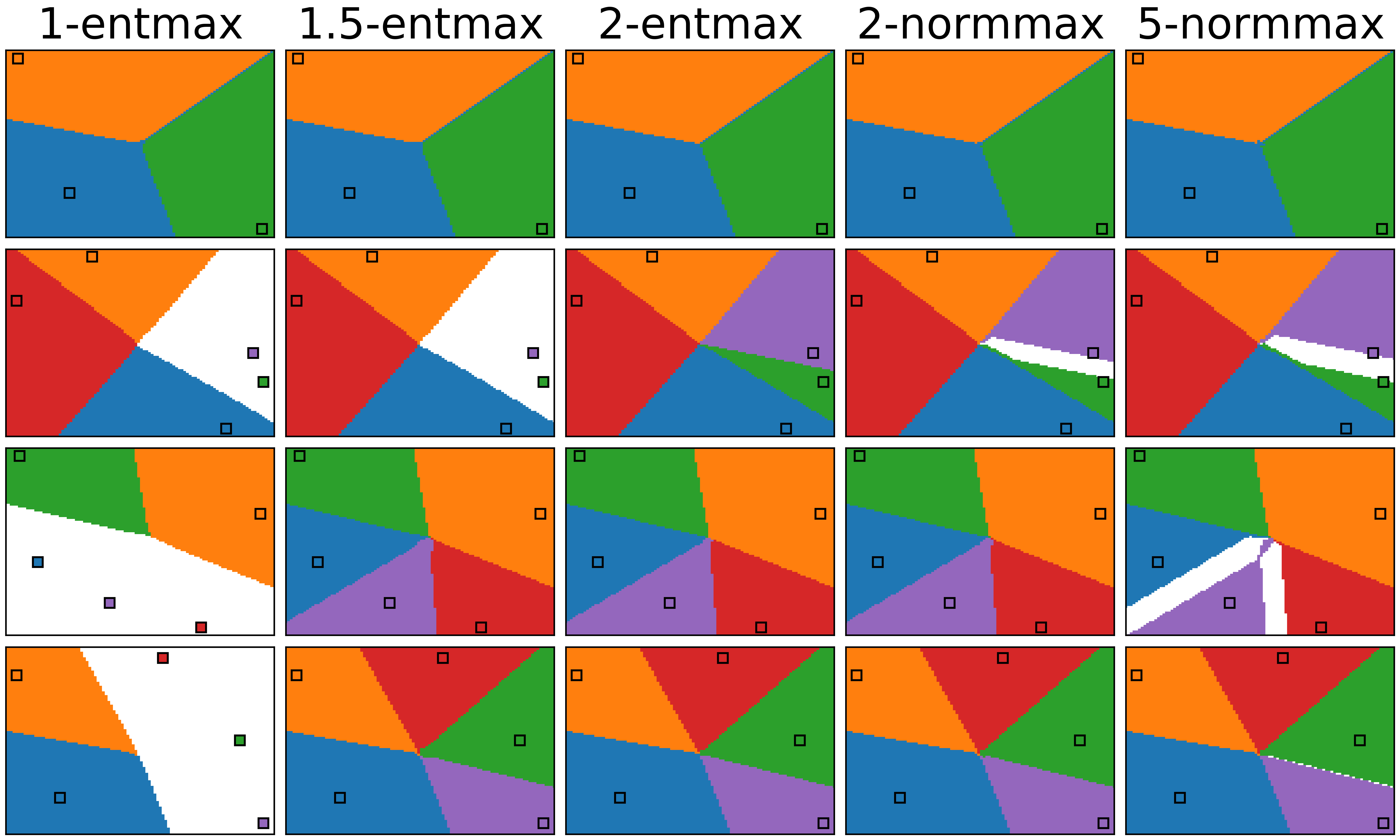}
  \caption{Left: contours of the energy function and optimization trajectory of the CCCP iteration ($\beta = 1$) for $\hat{\bm{y}}_\Psi(\bm{z}) = (\bm{z} - \mu_{\bm{z}})/\sqrt{\sigma_{\bm{z}}^2 + \epsilon}$. Here we show trajectories for 3D data points and contours according with the $\Psi(\bm{q})$ restrictions mentioned in \S\ref{sec:HFYE}, where the contours lie in the plane $z=-x-y$, intersected with the sphere of radius $\sqrt{D-1}$. White regions correspond to infinite energy. Right: attraction basins associated with each pattern ($\beta=10$).}
  \label{fig:overall2}
\end{figure*}

Figures \ref{fig:overall}, \ref{fig:overall11} and \ref{fig:overall2} illustrate the optimization trajectories and basins of attraction across different queries and artificially generated memory pattern configurations for two families of sparse transformations: $\alpha$-entmax and $\gamma$-normmax. We use the post-transformations $\hat{\bm{y}}_\Psi(\bm{z}) = \bm{z}$ (identity), $\hat{\bm{y}}_\Psi(\bm{z}) = \frac{\bm{z}}{||\bm{z}||}$ (normalization) and $\hat{\bm{y}}_{\Psi}(\bm{z}) = (\bm{z} - \mu_{\bm{z}})/\sqrt{\sigma_{\bm{z}}^2 + \epsilon}$ (layer normalization), respectively, which were covered in \S\ref{sec:particular_cases}. We used $\epsilon =10^{-8}$, as is commonly done in layer normalization for numerical stability. We use $\alpha \in \{1, 1.5, 2\}$ for $\alpha$-entmax and $\gamma \in \{2, 5\}$ for $\gamma$-normmax (where we apply the bisection algorithm described in Appendix~\ref{sec:normmax}). 

In Figure~\ref{fig:overall}, as $\alpha$ increases, $\alpha$-entmax converges more often to a single pattern, whereas $\gamma$-normmax tends to converge towards an attractor which is a uniform average of some patterns. This behavior is also observable in the basins of attraction (right plot), where larger values of $\alpha$ result in fewer regions converging to metastable states. In Figure~\ref{fig:overall11}, it is observed that the converged patterns consistently align along the circle with infinite energy outside. This observation is in line with expectations, considering the function $\hat{\bm{y}}_\Psi$, derived from $\Psi(\bm{q}) = I_{\|.\| \le r}(\bm{q})$, with $r=1$, that reflects the projected space constraint performed by normalization. In this figure, the local minima of the energy function tend to cluster around a set of memories, whereas in the basins of attraction, the trends closely resemble those in Figure~\ref{fig:overall}, with the exception of the softmax and 1.5-entmax cases, where more attraction areas are present. In Figure~\ref{fig:overall2}, we generate synthetic 3D data and plot the contours and trajectories obtained through the Hopfield update rules (CCCP iterations). 
In this case, we have 
%\begin{equation}
    $\Psi(\bm{q}) = I_S(\bm{q})$, where $S = \{ \bm{q} \mid \| \bm{q} \| \leq \sqrt{D-1} \text{ and } \bm{1}^\top \bm{q} = 0 \}$,
%\end{equation}
which corresponds to layer normalization, 
and therefore we can represent 
a query $\bm{q} = (q_1, q_2, q_3)$ in the 2D plane through coordinates $(q_1, q_2)$, with $q_3 = -q_1-q_2$. 
After the first iteration, the points converge to the conditions specified by this indicator function, \textit{i.e.}, \(\bm{1}^\top \bm{q} = 0\) plane and \(\sqrt{D-1}\) radius sphere.%
\footnote{The \(D-1\) term arises because the layer normalization operation uses an unbiased standard deviation.} %
This is a special case of the scenario discussed in \S\ref{sec:HFYE} where no trainable parameters are present. Similarly to normalization, trajectories tend to converge to clusters of memories, and the basins of attraction exhibit a greater number of attraction areas. % compared to the methods previously considered.

\subsection{Retrieval capacity}

We next assess the ability of HFY networks to handle growing quantities of stored memories (Figure~\ref{fig:1}), and noise (Figure~\ref{fig:3}), across various choices of $\hat{\bm{y}}_\Omega$ and $\hat{\bm{y}}_\Psi$, including the specific cases in Table~\ref{tab:hopfield}. 
%In certain instances, this analysis simplifies to recognized associative memory models, as illustrated in Table~\ref{tab:hopfield}.
\begin{figure*}[t]
    \centering
    \includegraphics[width=1\textwidth]{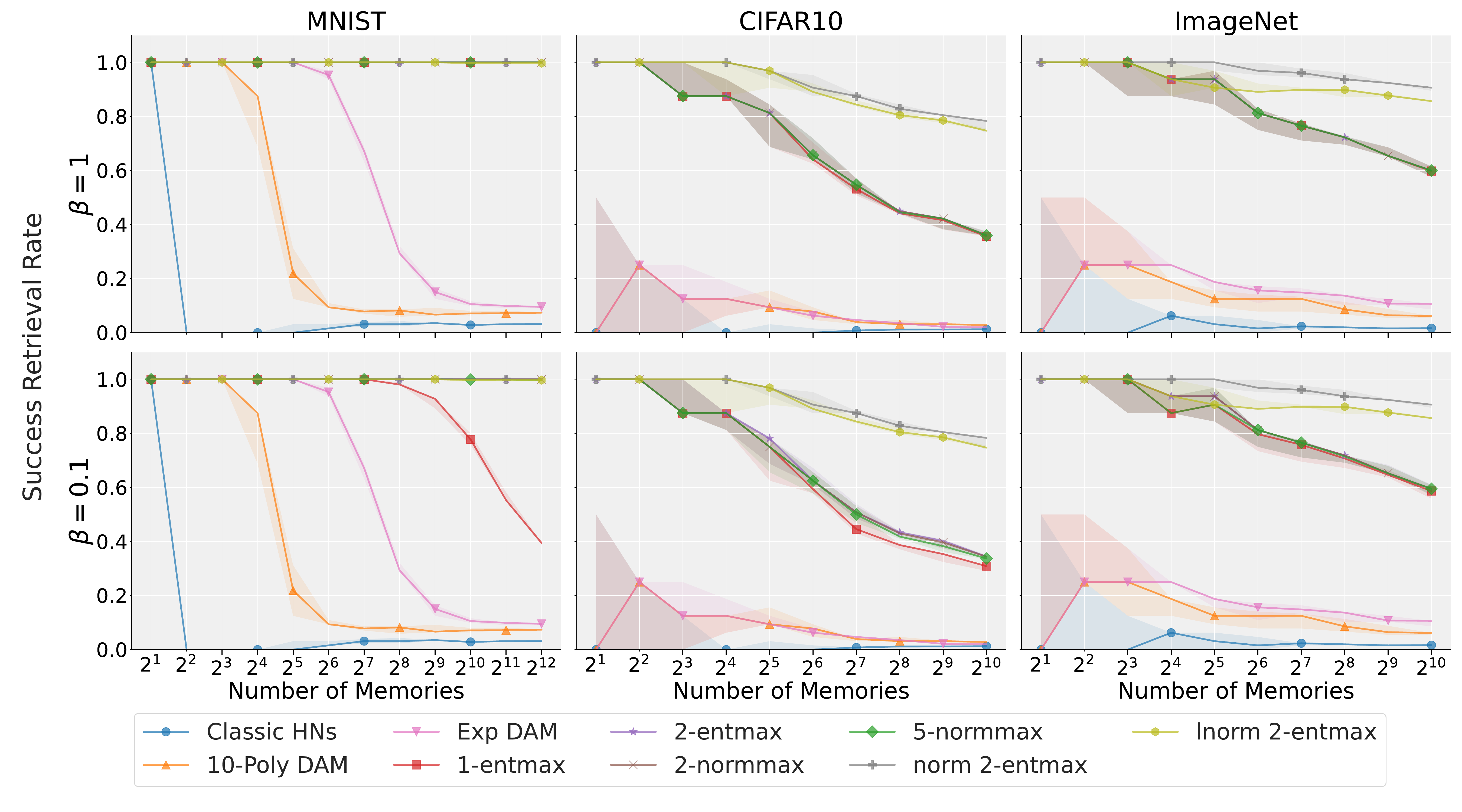}
\caption{Memory capacity for different numbers of stored memories for $\beta=0.1$ (bottom) and $\beta=1$ (top). For $\beta=1$, entmax and normmax lines intersect. Norm stands for $\ell_2$ normalization, which corresponds to $\hat{\bm{y}}_{\Psi}(\bm{z}) = \bm{z}/{\|\bm{z}\|}$, while lnorm, short for layer normalization, corresponds to $\hat{\bm{y}}_{\Psi}(\bm{z}) = (\bm{z} - \mu_{\bm{z}})/{\sqrt{\sigma_{\bm{z}}^2 + \epsilon}}$. Plotted are the medians over 5 runs with different memories and the interquartile range.}
    \label{fig:1}
\end{figure*}
We assess the retrieval capacity on three image datasets: MNIST \citep{lecun1998gradient}, CIFAR10 \citep{krizhevsky2009learning}, and Tiny ImageNet \citep{le2015tiny}. Prior to inputting the images as queries to the network, we normalize all pixel values to the interval $[-1, 1]$. The images are flattened into a single vector when fed to the Hopfield network. During the masking process, pixels outside the mask were set to 0. When introducing Gaussian noise to the images, we ensured that pixel values were clipped, preserving all values within the $[-1, 1]$ interval. A query is successfully retrieved when its cosine similarity falls above a predefined threshold of $\epsilon>0.9$. Plotted are the medians of the ratio of successfully retrieved patterns (success retrieval rate) and the interquartile range for 5 runs with different memories for the methods described in Table~\ref{tab:hopfield}.

In Figure~\ref{fig:1} and Figure~\ref{fig:3}, we can observe that the classic Hopfield networks of \citet{hopfield1982neural} and the dense associative models of \citet{krotov2016dense} and \citet{demircigil2017model} struggle to successfully retrieve patterns, most noticeable for the former, even for a low number of memories or for low levels of noise. Notably, in modern Hopfield networks with $\beta=1$ (first row of Figure~\ref{fig:1}), all variants—$\alpha$-entmax and $\gamma$-normmax—show ideal behavior on the MNIST dataset. They also demonstrate accurate behavior on the other datasets, with or without the normalization and layer normalization post-transformations. Additionally, these methods exhibit graceful degradation as the number of stored memories increases. A similar behavior is observed for \(\beta=0.1\) (second row), with performance improving as \(\alpha\) increases, for $\alpha$-entmax methods, as $\gamma$ decreases for $\gamma$-normmax. One can also see that \(2\)-entmax with \(\hat{\bm{y}}_{\Psi}(\bm{z}) = {\bm{z}}/{\|\bm{z}\|}\) (normalization) demonstrates even better performance across all datasets, indicating the positive contribution of this specific \(\hat{\bm{y}}_{\Psi}(\bm{z})\). Superior performance is also observed with \(\hat{\bm{y}}_{\Psi}(\bm{z}) = (\bm{z} - \mu_{\bm{z}})/{\sqrt{\sigma_{\bm{z}}^2 + \epsilon}}\) (layer normalization), although not as good as the former. Similar behavior can be observed in Figure~\ref{fig:3} but now in terms of the noise standard deviation. Detailed plots of the HFY networks, using \(\hat{\bm{y}}_\Omega\) as either \(\alpha\)-entmax or \(\gamma\)-normmax for different $\alpha$ and $\gamma$ values and different \(\hat{\bm{y}}_\Psi\), can be found in Appendix~\ref{app:MR}.

\begin{figure*}[t]
    \centering
    \includegraphics[width=1\textwidth]{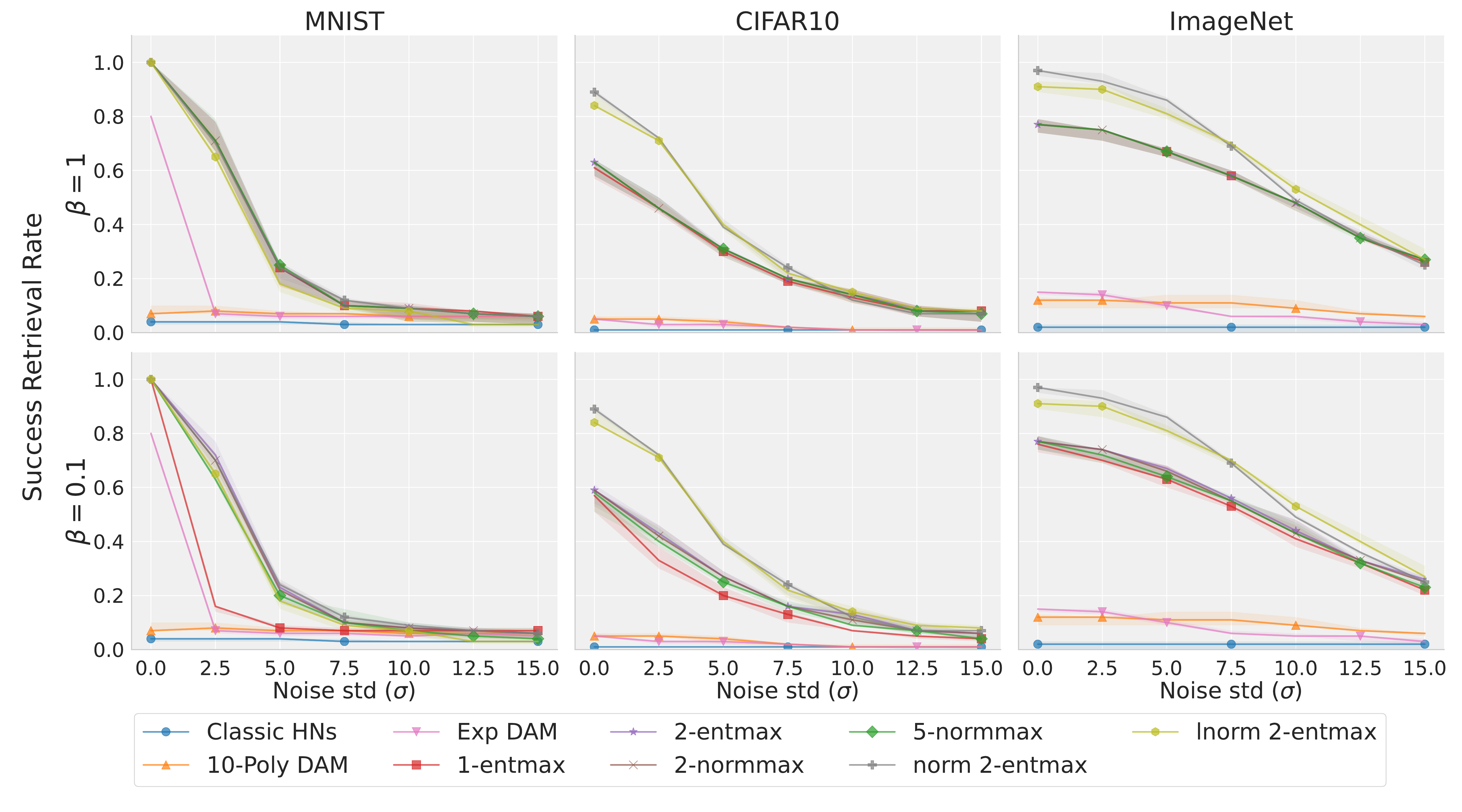}
    \caption{Memory robustness against different levels of noise for $\beta=0.1$ (bottom) and $\beta=1$ (top). For $\beta=1$ entmax and normmax lines intersect. %Norm stands for $\ell_2$ normalization, which corresponds to $\hat{\bm{y}}_{\Psi}(\bm{z}) = {\bm{z}}/{\|\bm{z}\|}$, while lnorm, short for layer normalization, corresponds to $\hat{\bm{y}}_{\Psi}(\bm{z}) = \frac{\bm{z} - \mu_{\bm{z}}}{\sqrt{\sigma_{\bm{z}}^2 + \epsilon}}$. 
    %Plotted are the medians over 5 runs with different memories and the interquartile range.
    }
    \label{fig:3}
\end{figure*}

\subsection{Sparse and structured transformations in multiple instance learning}

In multiple instance learning (MIL), instances are grouped into bags and the goal is to predict the label of each bag based on the instances it contains. If a bag contains at least one item from a given class, we consider it positive. This holds true even when we only know whether that particular instance is present in the bag, without needing to know the quantity. This framework is particularly useful in situations where annotating individual instances is challenging or impractical, while bag-level labels are easier to obtain. Instances of such scenarios include medical imaging, where a bag might represent an image, instances could manifest as patches within the image, and the label signifies the presence or absence of a disease. 
%In real-world scenarios where data may be subject to imperfections, errors, or uncertainties considering noise in the instances of a bag becomes essential. 
We also consider a extended variant, denoted $K$-MIL, where bags are considered positive if they contain $K$ or more positive instances; MIL is recovered when $K=1$. 
$K$-MIL with $K>1$ can be useful in scenarios where instance labels or uncertain, or where precision in bag labels is more important than recall. 
%Requiring a bag to have a specified threshold, $K$, of positive instances for positive bag classification, the model becomes more robust to noise. This robustness is attributed to the fact that the model can tolerate a certain level of uncertainty or errors in individual instances, thus enhancing its applicability and performance in situations where noise in the data is prevalent. Notably, 
Our $k$-subsets method is particularly suitable for this problem, due to its ability to retrieve $k$ patterns. 
%theoretically designed to outperform the methods in the probability simplex.

\citet{ramsauer2020hopfield} tackle MIL via a Hopfield pooling layer, where the query $\bm{q}$ is learned and the keys $\bm{X}$ are instance embeddings. This approach closely resembles transformer self-attention with pre- and post-operations such as layer normalization, which differs slightly from a pure Hopfield layer.
We experiment with the sparse variants of the Hopfield pooling layer introduced by \citet{ramsauer2020hopfield}, as these layers contain more parameters, making them stronger pooling approximators. We use our proposed $\alpha$-entmax and $\gamma$-normmax transformations (\S\ref{sec:probabilistic}), as well as structured variants using SparseMAP with $k$-subsets (\S\ref{sec:structured}), varying $\alpha$, $\gamma$ and $k$ in each case. 
We run these models for $K$-MIL problems in the MNIST dataset (choosing `9' as target as it can be easily misunderstood with `7' or `4'),  
and on three MIL benchmarks: the Elephant, Fox, and Tiger datasets \citep{ilse2018attention}. 
We experiment with $K \in \{2,3,5\}$. 
Further details can be found in Appendix~\ref{sec:MNIST_Experimental_Details} and \ref{sec:MIL_bench_details}.  

Table~\ref{tab:MIL} shows the results. We observe that for MNIST with $K=1$, $1$-entmax  outperforms the remaining methods. Normmax shows consistent results across datasets achieving near-optimal performance, arguably due to its ability to adapt to near-uniform metastable states of varying size. We also observe that, for $K>1$, the $k$-subsets approach achieves top performance when $k=K$, as expected. We conjecture that this is due to the ability of SparseMAP  with $k$-subsets for $k=K$ to retrieve exactly the $K$ positive instances in the bag, whereas other values of $k$ might under or over-retrieve. In the MIL benchmarks, SparseMAP pooling surpasses sparse pooling variants for 2 out of 3 datasets. % possibly due to the fact that positive bags are all filled with positive instances for these datasets. %which allows SparseMAP to help highlighting the collective certainty (represented by the value of $k$) associated with positive instances in the bags. \Saul{this might not be clear}
\begin{table*}[t]
    \caption{Results for MIL. We show accuracies for MNIST and ROC AUC for MIL benchmarks, averaged across 5 runs.}
    \vspace{-0.1cm}
    \centering
    \small
    \resizebox{\textwidth}{!}{%
    \begin{tabular}{lcccccccccccccccc}
        \toprule
        \multicolumn{1}{c}{} & \multicolumn{4}{c}{MNIST} & \multicolumn{3}{c}{MIL benchmarks} \\
        \cmidrule(lr){2-5}
        \cmidrule(lr){6-8}
        Methods & $K$=1 & $K$=2 & $K$=3 & $K$=5 & Fox & Tiger & Elephant  \\
        \midrule
        1-entmax (softmax) & $\mathbf{98.4\pm0.2}$ & $94.6\pm0.5$ & $91.1\pm0.5$ & $89.0\pm0.3$ & $66.4\pm2.0$ & $87.1\pm1.6$ & $92.6\pm0.6$\\
        1.5-entmax & $97.6\pm0.8$&$96.0\pm 0.9$ & $90.4\pm1.1$& $92.4\pm1.4$& $66.3\pm2.0$ & $87.3\pm1.5$ & $92.4\pm1.0$ \\
        2.0-entmax (sparsemax) &$97.9\pm0.2$& $96.7\pm0.5$ & $92.9\pm0.9$& $91.6\pm1.0$ & $66.1\pm0.6$ & $\mathbf{87.7\pm1.4}$ & $91.8\pm0.6$\\
        2.0-normmax &$97.9\pm0.3$& $96.6\pm0.6$& $93.9\pm0.7$& $92.4\pm0.7$& $66.1\pm2.5$ & 
        $86.4\pm0.8$ & $92.4\pm0.7$\\
        5.0-normmax &$98.2\pm0.5$& $97.2\pm0.3$ & $95.8\pm0.4$& $93.2\pm0.5$ & $66.4\pm2.3$& $85.5\pm0.6$ & $93.0\pm0.7$
        \\
        SparseMAP, $k=2$ & $97.9\pm0.3$ & $\mathbf{97.7\pm0.3}$ & $95.1\pm0.5$ & $92.6\pm1.1$ & $66.8\pm2.7$ & $85.3\pm0.5$ & $\mathbf{93.2\pm0.7}$ \\
        SparseMAP, $k=3$ &$98.0\pm0.6$ & $96.1\pm1.0$ & $\mathbf{96.5\pm0.5}$& $92.2\pm1.2$& $\mathbf{67.4\pm2.0}$ & $86.1\pm0.8$ & $92.6\pm1.7$\\
        SparseMAP, $k=5$ &$98.2\pm0.4$&$96.2\pm1.4$& $95.1\pm1.1$ & $\mathbf{95.1\pm1.5}$& $67.0\pm2.0$ & $86.3\pm0.8$ & $91.2\pm1.0$\\
        \bottomrule
    \end{tabular}
    }
    \label{tab:MIL}
\end{table*}

\subsection{Post-transformations in multiple instance learning}

In the previous experiment, we worked with extended variants of the Hopfield pooling layers from \citet{ramsauer2020hopfield}, which are designed to resemble self-attention mechanisms in transformers, incorporating distinct pre- and post-layer normalization for the queries and memories. These layers optionally normalize queries and keys with layer normalization, project them, and then layer-normalize them again each with different learnable parameters leading to different keys and values, as shown in the following pipeline, where $\bm{W}_Q, \bm{W}_K, \bm{W}_V$ are projection matrices (see \citealt[Figure~A.7]{ramsauer2020hopfield}):
\begin{itemize}
    \item \textbf{Queries:} $\bm{q} \,\, \mapsto \,\, \bm{q}' = \bm{W}_Q^\top \text{LayerNorm}(\bm{q}) \,\, \mapsto \,\, \bm{q}^{(0)} = \text{LayerNorm}(\bm{q}')$
    \item \textbf{Keys:} $\bm{x}_i \,\, \mapsto \,\, \bm{x}_i' = \bm{W}_K^\top \text{LayerNorm}(\bm{x}_i) \,\, \mapsto \,\, \bm{k}_i = \text{LayerNorm}(\bm{x}_i')$
    \item \textbf{Values:} $\bm{x}_i \,\, \mapsto \,\, \bm{x}_i' = \bm{W}_V^\top \text{LayerNorm}(\bm{x}_i) \,\, \mapsto \,\, \bm{v}_i = \text{LayerNorm}(\bm{x}_i')$
\end{itemize}
\begin{comment}
\begin{itemize}
    \item \textbf{Queries:} $
    \mathbf{q} \rightarrow \text{LayerNorm}_{\mathbf{q}}(\mathbf{q}) \rightarrow \bm{W}_{\bm{q}} \, \text{projected as } \, \mathbf{q'} \rightarrow \text{LayerNorm}_{\mathbf{q'}}(\mathbf{q'}) = \mathbf{q}^t$

    \item \textbf{Keys:} $
    \mathbf{X} \rightarrow \text{LayerNorm}_{\mathbf{X}}(\mathbf{X}) \rightarrow \bm{W}_{\bm{X}} \, \text{projected as } \, \mathbf{X'} \rightarrow \text{LayerNorm}_{\mathbf{X'}}(\mathbf{X'}) = \mathbf{X}$

    \item \textbf{Values:} $
    \mathbf{X} \rightarrow \text{LayerNorm}_{\mathbf{V}}(\mathbf{X}) \rightarrow \bm{W}_{\bm{V}} \, \text{projected as } \, \mathbf{V'} \rightarrow \text{LayerNorm}_{\mathbf{V'}}(\mathbf{V'}) = \mathbf{V}$
\end{itemize}
\end{comment}
These operations are followed by the Hopfield update $\bm{q}^{(t+1)} = \bm{V}^\top \hat{\bm{y}}_\Omega(\beta \bm{K} \bm{q}^{(t)}) $, 
where $\bm{K} = [\bm{k}_1^\top; ...; \bm{k}_N^\top] \in \mathbb{R}^{N \times D}$ and  $\bm{V} = [\bm{v}_1^\top; ...; \bm{v}_N^\top] \in \mathbb{R}^{N \times D}$. 
%However, the use of distinct transformations for both the queries and the memories leads to a mismatch with our derived theoretical framework, with layer normalization as post-transformation, as the same operation should be applied to both the queries and memories to ensure they are in the same feature space.
Despite its higher expressiveness, which motivated its use by \citet{ramsauer2020hopfield} and in our previous experiment, this approach also contrasts with ``pure'' Hopfield layers, where keys must equal the values. Therefore, we %Despite the promising results encountered for this setup (Table~\ref{tab:MIL}), we aim to 
experiment also with different post-transformations in the pure Hopfield scenario, which matches precisely the derived theoretical framework. %\Saul{Does this justify the results on pure hopfield networks? and now you see why it doesn't make sense to add post transformations to ramsauer's layers?}
\begin{table*}[t]
    \caption{Results for MIL. We show ROC AUC, averaged across 5 runs. We bold the top performing model for each dataset and underline the best $\hat{\bm{y}}_\Psi$ for each method.}
    \vspace{-0.1cm}
    \centering
    \small
    \begin{tabular}{lcccccccccccccccc}
        \toprule
        Methods & Post-Transformation & Fox & Tiger & Elephant\\
        \midrule
        \multirow{3}{*}{$1$-entmax (softmax)} & Identity & $63.6\pm1.7$ & $86.9\pm1.0$ & $91.3\pm1.0$\\
        & $\ell_2$ normalization & $\underline{64.3\pm2.4}$& $\underline{87.0\pm0.8}$ & $\underline{91.6\pm0.4}$
        \\
        &LayerNorm & $62.1\pm2.3$ & $\underline{87.0\pm0.9}$ & $91.2\pm1.0$
        \\
        \midrule
        \multirow{3}{*}{$1.5$-entmax} & Identity & $61.6\pm3.8$ & $86.7\pm0.9$ & $\underline{92.0\pm0.4}$ \\
        & $\ell_2$ normalization & $\underline{64.2\pm2.4}$& $86.7\pm0.4$ & $91.5\pm0.7$
        \\
        &LayerNorm & $63.4\pm1.6$ & $\underline{87.0\pm0.9}$ & $\underline{92.0\pm0.4}$
        \\
        \midrule
        \multirow{3}{*}{$2$-entmax (sparsemax)} & Identity & $63.7\pm1.7$ & $86.8\pm0.9$ & $91.6\pm0.5$\\
        & $\ell_2$ normalization & $\underline{63.4\pm2.7}$& $\underline{87.6\pm1.0}$ & $90.6\pm0.8$
        \\
        &LayerNorm & $\underline{63.4\pm1.6}$ & $85.0\pm1.3$ & $\underline{91.7\pm0.5}$
        \\
        \midrule
        \multirow{3}{*}{$2$-normmax} & Identity & $63.7\pm1.7$ & $86.7\pm0.9$ & $92.0\pm0.4$  \\
        & $\ell_2$ normalization &$\underline{64.2\pm2.4}$ & $\mathbf{\underline{87.7\pm0.6}}$ & $\mathbf{\underline{92.6\pm0.8}}$
        \\
        &LayerNorm & $63.4\pm1.6$ & $87.0\pm0.9$ & $91.9\pm0.4$
        \\
        \midrule
        \multirow{3}{*}{$5$-normmax} & Identity & $61.9\pm1.7$ & $86.9\pm1.0$ & $\underline{91.9\pm0.6}$ \\
        & $\ell_2$ normalization & $64.2\pm2.4$ & $\underline{87.5\pm0.7}$ & $91.3\pm0.7$
        \\
        &LayerNorm & $\underline{\mathbf{64.6\pm3.1}}$ & $87.0\pm0.9$ & $\underline{91.9\pm0.6}$
        \\
        \bottomrule
    \end{tabular}
    \label{tab:MIL1}
\end{table*}
We experiment in Table~\ref{tab:MIL1} with pure Hopfield layers using different $\hat{\bm{y}}_\Psi$ functions, namely the identity, $\ell_2$-normalization, and layer normalization (see \S\ref{sec:HFYE}). Note that, for the post-transformation identity, 1-entmax recovers \citet{ramsauer2020hopfield} without extra parametrizations and 2-entmax recovers \citet{hu2023sparse}. Identity is represented by $\hat{\bm{y}}_\Psi(\bm{z}) = \bm{z}$. For $\ell_2$-normalization, we use $\hat{\bm{y}}_\Psi(\bm{z}) = \frac{r\bm{z}}{\|\bm{z}\|}$ with $r = 1$. In the case of layer normalization, we apply $\hat{\bm{y}}_\Psi(\bm{z}) = \eta \frac{\bm{z} - \mu_{\bm{z}}}{\sqrt{\sigma_{\bm{z}}^2 + \epsilon}} + \bm{\delta}$, where $\eta$ and $\bm{\delta}$ are learnable parameters.
 The post-transformations are applied to both the initial query and memory, projecting them into the space of the Hopfield output.

Table~\ref{tab:MIL1} displays the results for the MIL benchmarks. %It is evident that the results, in average, underperform compared to \citet{ramsauer2020hopfield} Hopfield layers (refer to Table~\ref{tab:MIL}). 
We see that both $\ell_2$-normalization and layer normalization post-transformations lead to clear benefits for all methods:  across the three datasets and five models,  
$\ell_2$-normalization outperforms or matches the other post-transformations in 10 out of 15 cases, whereas layer normalization outperforms or matches the other post-transformations in 7 out of 15 entries. 

%\subsection{Extraction of Text Highlights}
\subsection{Structured Rationalizers}

Finally, we experiment with rationalizer  models in sentiment prediction tasks, where the inputs are sentences or documents in natural language and the rationales are text highlights (see Figure~\ref{fig:mismatch} for an illustration). 
These models, sometimes referred as select-predict or explain-predict models \citep{jacovi2021aligning,zhang2021explain}, consist of a rationale generator and a predictor. The generator processes the input text and extracts the rationale as a subset of words to be highlighted, and the predictor classifies the input based solely on the extracted rationale, which generally involves concealing non-rationale words through the application of a binary mask. 
\begin{figure*}[t]
    \centering
    \includegraphics[width=1\textwidth]{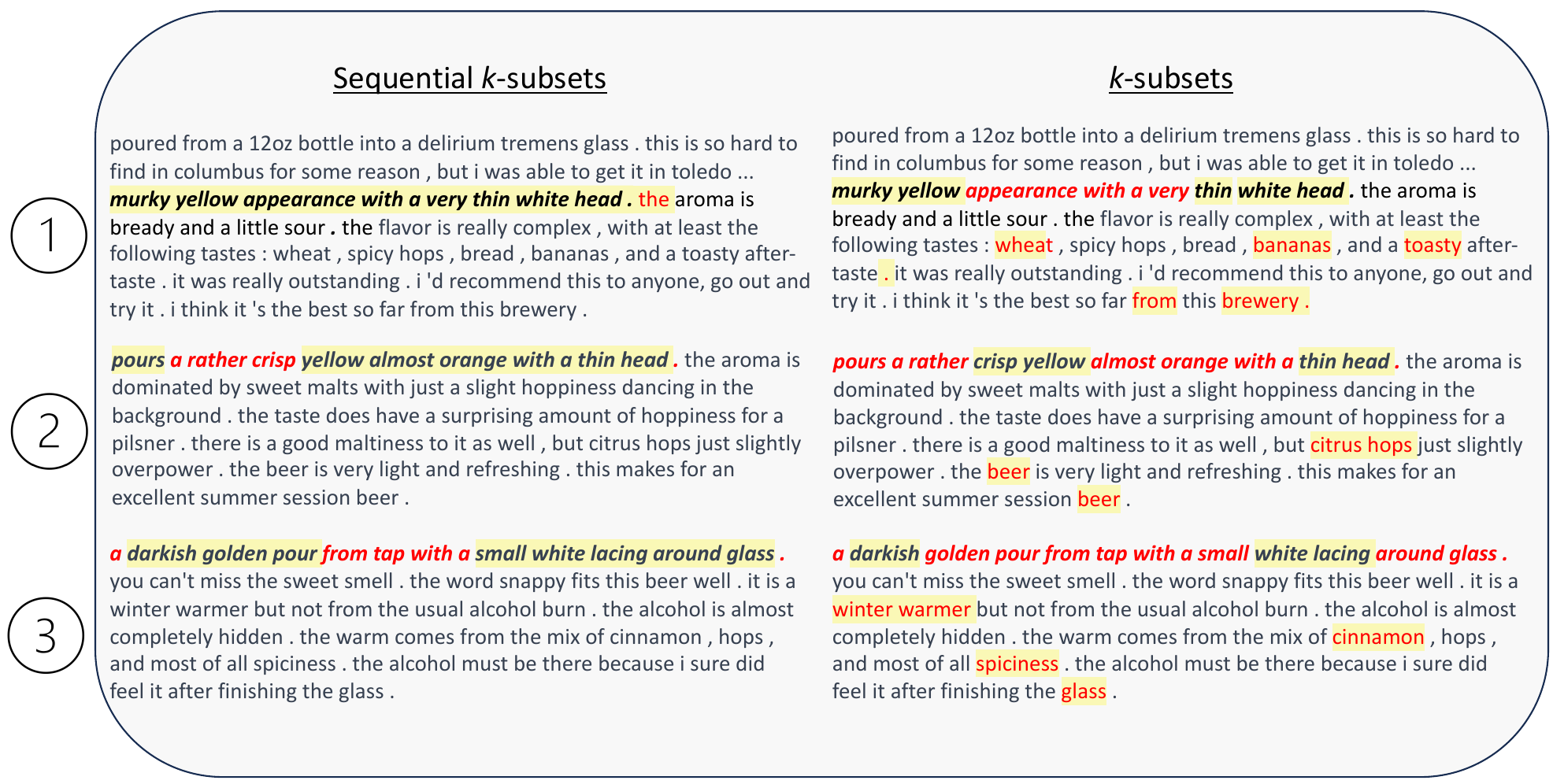}
    \caption{Examples of human rationale overlap for the aspect ``appearance''. The \hl{yellow highlight} indicates the model's rationale, while \textbf{\textit{italicized and bold font}} represents the human rationale. \textcolor{red}{Red font} identifies mismatches with human annotations. SparseMAP  with sequential $k$-subsets prefers more contiguous  rationales, which better match humans.}
    \label{fig:mismatch}
\end{figure*} 
Rationalizers are usually trained end-to-end, and the discreteness of the latent rationales is either handled with stochastic methods via score function estimators or the reparametrization trick \citep{lei2016rationalizing,bastings2019interpretable}, or with deterministic methods via structured continuous relaxations \citep{guerreiro2021spectra}. In either case, the model imposes sparsity and contiguity penalties to ensure rationales are short and tend to extract  adjacent words.

\begin{table*}[t]
\caption{Text rationalization results. We report mean and min/max $F_1$ scores across five random seeds on test sets for all datasets but Beer, where we report
MSE. All entries except SparseMAP are taken from \citet{guerreiro2021spectra}. We also report human rationale overlap (HRO) as $F_1$ score. We bold the best-performing rationalized model(s).}
\centering
\small
\resizebox{\textwidth}{!}{%
\begin{tabular}{l l c c c c c}
\toprule
Method & Rationale  & SST$\uparrow$ & AgNews$\uparrow$ & IMDB$\uparrow$ & Beer$\downarrow$  & Beer(HRO)$\uparrow$\\
\midrule
\multirow{2}{*}{SFE} 
& top-$k$       & .76 {\small(.71/.80)} & .92 {\small(.92/.92)} & .84 {\small(.72/.88)} & .018 {\small(.016/.020)} & .19 {\small(.13/.30)} \\
& contiguous    & .71 {\small(.68/.75)} & .86 {\small(.85/.86)} & .65 {\small(.57/.73)} & .020 {\small(.019/.024)} & .35 {\small(.18/.42)} \\
\midrule
\multirow{2}{*}{SFE w/Baseline} 
& top-$k$       & .78 {\small(.76/.80)} & .92 {\small(.92/.93)} & .82 {\small(.72/.88)} & .019 {\small(.017/.020)} & .17 {\small(.14/.19)}  \\
& contiguous    & .70 {\small(.64/.75)} & .86 {\small(.84/.86)} & .76 {\small(.73/.80)} & .021 {\small(.019/.025)} & .41 {\small(.37/.42)} \\
\midrule
\multirow{2}{*}{Gumbel} 
& top-$k$       & .70 {\small(.67/.72)} & .78 {\small(.73/.84)} & .74 {\small(.71/.78)} & .026 {\small(.018/.041)} & .27 {\small(.14/.39)} \\
& contiguous    & .67 {\small(.67/.68)} & .77 {\small(.74/.81)} & .72 {\small(.72/.73)} & .043 {\small(.040/.048)} & .42 {\small(.41/.42)}\\
\midrule
HardKuma 
& -             & .80 {\small(.80/.81)} & .90 {\small(.87/.88)} & .87 {\small(.90/.91)} & .019 {\small(.016/.020)} & .37 {\small(.00/.90)}  \\
\midrule
\multirow{2}{*}{Sparse Attention} 
& sparsemax     & \textbf{.82} {\small(.81/.83)} & \textbf{.93} {\small(.93/.93)} & .89 {\small(.89/.90)} & .019 {\small(.016/.021)} & .48 {\small(.41/.55)}  \\
& fusedmax      & .81 {\small(.81/.82)} & .92 {\small(.91/.92)} & .88 {\small(.87/.89)} & .018 {\small(.017/.019)} & .39 {\small(.29/.53)} \\
\midrule
SPECTRA 
& seq. $k$-subsets & .80 {\small(.79/.81)} & .92 {\small(.92/.93)} & \textbf{.90} {\small(.89/.90)} & \textbf{.017} {\small(.016/.019)} & .61 {\small(.56/.68)}  \\
\midrule
\multirow{2}{*}{SparseMAP} 
& $k$-subsets       & .81 {\small(.81/.82)} & \textbf{.93} {\small(.92/.93)} & \textbf{.90} {\small(.90/.90)} & \textbf{.017} {\small(.017/.018)} & .42 {\small(.29/.62)}   \\
& seq. $k$-subsets  & .81 {\small(.80/.83)} & \textbf{.93} {\small(.93/.93)} & \textbf{.90} {\small(.90/.90)} & .020 {\small(.018/.021)} & \textbf{.63} {\small(.49/.70)} \\
\bottomrule
\end{tabular}
}
\label{tab:spectra_extended}
\end{table*}

Our model architecture is adapted from SPECTRA \citep{guerreiro2021spectra}, but the combination of the generator and predictor departs from prior approaches \citep{lei2016rationalizing,bastings2019interpretable,guerreiro2021spectra} in which the predictor does not ``mask'' the input tokens; instead, it takes as input the pooled vector that results from the Hopfield pooling layer (either a sequential or non-sequential SparseMAP $k$-subsets layer). %This departure from the traditional binary mask approach adds flexibility to the model's interpretation process. 
%Moreover, in alignment with the architecture proposed by \citet{guerreiro2021spectra}, our predictor maintains a similar structure with a noteworthy modification. 
By integrating this Hopfield pooling layer into the predictor, we transform the sequence of word embeddings into a single vector from which the prediction is made. The rationale is formed by the pattern associations (word tokens) extracted by the Hopfield layer. 
We use the same hyperparameters as \citet{guerreiro2021spectra}. 
We use a head dimension of 200, to match the dimensions of the encoder vectors  (the size of the projection matrices associated to the static query and keys) and a head dropout of 0.5 (applied to the output of the Hopfield layer). %We used a single attention head to better match the SPECTRA model.  
We use a transition score of 0.001 and a train temperature of 0.1.

Table~\ref{tab:spectra_extended} shows the results on the downstream task (classification for SST, AgNews, IMDB; regression for BeerAdvocate) and the $F_1$ overlap with human rationales for the BeerAdvocate dataset \citep{mcauley2012learning}. 
Compared to strong baselines \citep{bastings2019interpretable,guerreiro2021spectra}, our methods achieve equal or slightly superior performance for all datasets. Moreover, our sequential $k$-subsets model outperforms the baselines in terms of overlap with human rationales, arguably due to the fact that human rationales tend to contain adjacent words, which is encouraged by our sequential model.

\section{Related Work}\label{sec:related_work}
%Recent research on modern Hopfield networks and dense associative memories includes \citep[\textit{inter alia}]{krotov2016dense,demircigil2017model,ramsauer2020hopfield,millidge2022universal,hoover2023energy}. 
Hopfield networks trace their origins to the works of \cite{amari1972learning}, \cite{nakano1972associatron},  \cite{amari1988statistical}, and \cite{hopfield1982neural}. A pivotal moment spurring increased research interest occurred in \cite{krotov2016dense}, which introduced a novel polynomial energy function, followed by exponential dense associative memories \citep{demircigil2017model}. While these models were initially designed for binary cases, the work by \cite{ramsauer2020hopfield} generalized them to continuous states, resembling attention mechanisms in transformers. These works were further extended to induce sparsity by \citet{hu2023sparse}, the first work which proposed {sparse Hopfield networks} and which derived retrieval error bounds tighter than the dense analog. 
This was followed by two lines of work which developed in parallel: \citet{wu2023stanhop} proposed  a ``generalized sparse Hopfield model'' based on $\alpha$-entmax with learnable $\alpha$, which they successfully applied to time series prediction problems; and \citet{martins2023sparse} and \citet{santos2024sparse} considered $\alpha$-entmax and $\gamma$-normmax, along with analysis and conditions for exact retrieval. 
%Additionally, \new{and concurrently to \citet{martins2023sparse}},  \citet{wu2023stanhop} proposed  a ``generalized sparse Hopfield model'' based on $\alpha$-entmax with learnable $\alpha$, which they successfully applied to time series prediction problems. 
%By establishing a connection with Fenchel-Young losses \citep{blondel2020learning}, where the energy is expressed as the difference between two Fenchel-Young losses (see \ref{eq:updates_general_hfy}), 
Our work generalizes these previous approaches as specific instances of a broader family of energy functions presented in \S\ref{sec:HFYE} which are expressed as a difference between two Fenchel-Young losses (see equation \ref{eq:updates_general_hfy}). 
Additionally, neither \citet{hu2023sparse} nor \citet{wu2023stanhop} explored the potential of achieving exact retrieval through sparse transformations. Our work addresses this gap by presenting a unified framework for sparse Hopfield networks with enhanced theoretical guarantees for retrieval and coverage (see Propositions \ref{prop:separation}--\ref{prop:storage}). Furthermore, the derived framework extends their constructions and broadens the applicability to new families, including $\gamma$-normmax. An effective bisection algorithm for $\gamma$-normmax (Algorithm \ref{algo:bisect}) is introduced. The link with Fenchel-Young losses allowed the derivation of many results, such as the margin conditions, which were found to have direct application to sparse Hopfield networks. We leveraged this framework to accommodate structure where we explored the structured margin of SparseMAP (Proposition~\ref{prop:sparsemap_margin}), key to establish exact retrieval of pattern associations (Proposition~\ref{prop:stationary_single_iteration_sparsemap}).

A related approach to structure was later explored in \cite{hu2024nonparametric}, where a particular instance of their structure involves a top-$k$ modern Hopfield network, which relates to our $k$-subsets example. Our approach simplifies the process by adjusting \(\hat{\bm{y}}_\Omega\) as part of an optimization problem. %, providing a more straightforward and intuitive solution.
Our $k$-subsets example also relates to the top-$k$ retrieval model introduced by \citet{davydov2023retrieving}. Their model diverges from ours as they employ an entropic regularizer that does not support sparsity, thereby making exact retrieval impossible. 
%Our $k$-subsets example is related to \cite{hu2024nonparametric}, which later investigated a similar approach to structure. In that instance, their structure incorporates a top-$k$ modern Hopfield network. Although this model employs a complex mathematical framework to accomplish the same goal, our method streamlines the procedure by redefining \(\hat{\bm{y}}_\Omega\) as an optimization problem, leading to a simpler and more understandable approach.
%Additionally, the top-$k$ retrieval model presented by \citet{davydov2023retrieving} also relates to our $k$-subsets example, however, they use an entropic regularizer that does not support sparsity, which makes exact retrieval impossible. 
Our sparse and structured Hopfield layers in \S\ref{sec:structured}, with sparsemax and SparseMAP, involve a quadratic regularizer, which relates to the differentiable layers of \citet{amos2017optnet}. The use of SparseMAP and its active set algorithm \citep{niculae2018sparsemap} allows to exploit the structure of the problem to ensure efficient Hopfield updates and implicit derivatives. 

\cite{millidge2022universal} introduced universal Hopfield networks, unpacking associative memory models into three operations: similarity, separation, and projection. This framework closely aligns with the general Fenchel-Young framework proposed in this paper (Proposition \ref{prop:cccp}), which can also accommodate various alternative similarity metrics. This extension differs from their work where we can incorporate an additional optional operation: the post-transformation $\hat{\bm{y}}_\Psi$ in \eqref{eq:updates_general_hfy}. This operation can implement $\ell_2$ normalization or layer normalization over the produced Hopfield result, bridging the gap with transformers. The first case aligns with the normalization approach in \citet{nguyen2019transformers}, while the second corresponds to the layer normalization used in \citet{attention}. 

%By introducing universal Hopfield networks, \cite{millidge2022universal} broke down associative memory models into three operations: projection, separation, and similarity. The effectiveness of other similarity metrics, like the Manhattan or Euclidean distance, is demonstrated by their empirical results. This framework closely resembles the broad Fenchel-Young framework (Proposition \ref{prop:cccp}) that is suggested in this paper which can also support a number of different similarity metrics. This formulation is different from their work in that we can include post-transformations, such as normalization or layer normalization, to the result of the Hopfield projection. While the second case corresponds to the layer normalization employed in \citet{attention}, the first case is consistent with the normalization approach in \citet{nguyen2019transformers}. 

Memory retrieval has garnered significant attention in computational neuroscience, based on foundational work by early researchers \citep{anderson1972recognition, tulving1973encoding, tulving1985memory}, and is a crucial paradigm for understanding how neural systems access and use stored information. However, a gap remains in machine learning approaches that effectively model memory retrieval paradigms, such as free and sequential retrieval. A pioneering study by \citet{recanatesi2015neural} used Hopfield networks to model free recall, but these models are limited to binary states. Recently, \citet{Naim2020} introduced a parameter-free, graph-based model to predict recall based on associative memory structure, but experiments are limited to word recall. Our work provides an alternative by incorporating continuous states through constrained and penalized sparse transformations.
%In computational neuroscience, memory retrieval has attracted a lot of attention mainly due to researchers' foundational work \citep{anderson1972recognition, tulving1973encoding, tulving1985memory}. This is because memory retrieval is essential to comprehending how neural systems access and use stored information. Despite these advancements, there is still a lack of machine learning techniques that can accurately simulate several memory retrieval paradigms, including sequential and free retrieval. Classical Hopfield networks were employed in a groundbreaking study by \citet{recanatesi2015neural} to simulate free recall; however, these models are only capable of representing binary states. In this work, with constrained and penalized sparse transformations, we extend this approach to incorporate continuous states.

\section{Conclusions}\label{sec:conclusions}
We presented a unified framework for Hopfield networks that accommodates not only sparse and structured Hopfield networks but also recovers many known methods, such as classic Hopfield networks, dense associative memories, and modern Hopfield networks. 
Our framework hinges on a broad family of energy functions, based on convex duality,
written as a difference of two Fenchel-Young losses, one parametrized by
a generalized negentropy function and the other can be any convex function that relates with the post-transformation. By incorporating additional operations such as $\ell_2$ normalization and layer normalization, we bridge the gap between Hopfield networks and transformer architectures, providing a theoretically grounded approach to more robust Hopfield-based attention mechanisms. A central result of our paper is the connection between the margin properties of certain Fenchel-Young losses and sparse Hopfield networks, establishing provable conditions for exact retrieval. Moreover, we extend this framework to incorporate structure via the SparseMAP transformation, allowing for the retrieval of pattern associations favored by top-$k$ or sequential top-$k$ retrieval, rather than a single pattern. Finally, we apply and validate our broad family of energies on different memory recall paradigms. We also validate the effectiveness of our approach on image retrieval, multiple instance learning, and text rationalization tasks.

\section*{Acknowledgments}

This work was supported by EU’s Horizon Europe Research and Innovation Actions (UTTER, contract 101070631), by the project DECOLLAGE (ERC-2022-CoG 101088763), by the Portuguese Recovery and Resilience Plan through project C64500888200000055 (Center for Responsible AI), by Fundação para a Ciência e Tecnologia through contract UIDB/50008/2020, 
by the Dutch Research Council (NWO) VI.Veni.212.228,
and by the Champalimaud Foundation.

%\section*{Broader Impact}

%Hopfield networks are increasingly relevant for practical applications and not
%only as theoretical models. While our work is mostly a theoretical advancement,
%promising experimental results signal potentially wider impact. Sparse HFY 
%networks are applicable in the same scenarios where modern Hopfield networks
%would be, and we do not foresee any specific societal consequences of sparse 
%transformations or exact retrieval in such cases. In the structured case,
%the practicioner has more hands-on control for encoding inductive biases
%through the design of a factor graph and the choice of its parameters
%(e.g., $k$). These choices may in turn reflect human biases and 
%have societal implications: for example, contiguous rationales might be well 
%suited for English but not for other languages. We encourage care when designing
%such methods for practical applications.

%This paper presents work whose goal is to advance the field of machine learning. There are many potential societal consequences of our work, none which we feel must be specifically highlighted here.

\newpage

\appendix
\onecolumn

%\section{Generalized Negentropies}\label{sec:generalized_negent}
%
%We recall here the definition of generalized negenetropies from \citet[\S4.1]{blondel2020learning}:
%
%\begin{definition}
%\label{def:generalized_negent}
%   A function \(\Omega : \triangle_{N} \to \mathbb{R}\)
%   is a generalized negentropy iff it satisfies the following conditions:
%   \begin{enumerate}
%      \item Zero negentropy: $\Omega(\bm{y})=0$
%      if \(\bm{y}\) is a one-hot vector (delta distribution), i.e.,
%      \(\bm{y}=\bm{e}_i\) for any \(i \in \{1,\ldots,N\}\).
%      \item Strict convexity:
%      \(\Omega\left((1-\lambda)\bm{y} + \lambda \bm{y}'\right)
%      < (1-\lambda)\Omega(\bm{y}) + \lambda \Omega(\bm{y}')
%      \) for $\lambda \in \,\, ]0,1[$ and $\bm{y}, \bm{y}' \in \triangle_N$ with $\bm{y} \ne \bm{y}'$.  
%      \item Permutation invariance:
%      \(\Omega(\bm{Py})=\Omega(\bm{y})\) for any permutation matrix \(\bm{P}\)
%      (i.e., square matrices with a single 1 in each row and each column, zero elsewhere).
%   \end{enumerate}
%\end{definition}
%
%This definition implies that $\Omega \le 0$ and that $\Omega$ is minimized when $\bm{y} = \mathbf{1}/N$ is the uniform distribution \citep[Proposition~4]{blondel2020learning}. This justifies the name ``generalized negentropies.''

\section{Bisection Algorithm for the Normmax Transformation}\label{sec:normmax}

We derive here expressions for the normmax transformation along with a bisection algorithm to compute this transformation for general $\gamma$. 

Letting $\Omega(\bm{y}) = -1 + \|\bm{y}\|_\gamma + I_{\triangle_N}(\bm{y})$ be the norm entropy, we have 
%\begin{align}
    $(\nabla \Omega^*)(\bm{\theta}) = \arg\max_{\bm{y} \in \triangle_N} \bm{\theta}^\top \bm{y} - \|\bm{y}\|_\gamma.$ 
%\end{align}
The Lagrangian function is $L(\bm{y}, \bm{\lambda}, \mu) = -\bm{\theta}^\top \bm{y} + \|\bm{y}\|_\gamma - \bm{\lambda}^\top\bm{y} + \mu (\mathbf{1}^\top \bm{y} - 1)$. 
Equating the gradient to zero and using  $\nabla \|\bm{y}\|_\gamma= \left({\bm{y}} / {\|\bm{y}\|_\gamma}\right)^{\gamma-1}$, we get:
\begin{align}\label{eq:kkt}
    \mathbf{0} = \nabla_{\bm{y}} L(\bm{y}, \bm{\lambda}, \mu) = 
    -\bm{\theta} + \left({\bm{y}}/{\|\bm{y}\|_\gamma}\right)^{\gamma-1} - \bm{\lambda} + \mu \mathbf{1}.
\end{align}
The complementarity slackness condition implies that, if $y_i > 0$, we must have $\lambda_i = 0$, therefore, we have for such $i \in \mathrm{supp}(\bm{y})$: 
\begin{align}\label{eq:kkt_mu}
    -\theta_i + \left({y_i}/{\|\bm{y}\|_\gamma}\right)^{\gamma-1} + \mu = 0 \quad \Rightarrow \quad y_i = (\theta_i - \mu)^{\frac{1}{\gamma-1}} \|\bm{y}\|_\gamma.
\end{align}
Since we must have $\sum_{i \in \mathrm{supp}(\bm{y})} y_i = 1$, we obtain $\|\bm{y}\|_\gamma^{-1} = {\sum_{i \in \mathrm{supp}(\bm{y})} (\theta_i - \mu)^{\frac{1}{\gamma-1}}}$. 
%\begin{align}
%    1 = \sum_{i \in \mathrm{supp}(\bm{y})} (\theta_i - \mu)^{\frac{1}{\gamma-1}} \|\bm{y}\|_\gamma \quad \Longrightarrow \quad \|\bm{y}\|_\gamma = \frac{1}{\sum_{i \in \mathrm{supp}(\bm{y})} (\theta_i - \mu)^{\frac{1}{\gamma-1}}}.
%\end{align}
Plugging into \eqref{eq:kkt_mu} and noting that, from \eqref{eq:kkt}, we have $\theta_i < \mu_i$ for $i \notin \mathrm{supp}(\bm{y})$, we get, for $i \in [N]$:
\begin{align}\label{eq:normmax_solution}
y_i = \frac{(\theta_i - \mu)_+^{\frac{1}{\gamma-1}}}{\sum_{j \in \mathrm{supp}(\bm{y})} (\theta_j - \mu)_+^{\frac{1}{\gamma-1}}}.
\end{align}
Moreover, since $\sum_{i \in \mathrm{supp}(\bm{y})} y_i^\gamma = \|\bm{y}\|_\gamma^\gamma$, we obtain from \eqref{eq:kkt_mu}:
\begin{align}\label{eq:normmax_normalization_condition}
    \|\bm{y}\|_\gamma^\gamma = \sum_{i \in \mathrm{supp}(\bm{y})} (\theta_i - \mu)^{\frac{\gamma}{\gamma-1}} \|\bm{y}\|_\gamma^\gamma \quad \Rightarrow \quad \sum_{i \in \mathrm{supp}(\bm{y})} (\theta_i - \mu)^{\frac{\gamma}{\gamma-1}} = 1.
\end{align}
In order to compute the solution \eqref{eq:normmax_solution} we need to find $\mu$ satisfying \eqref{eq:normmax_normalization_condition}. This can be done with a simple bisection algorithm if we find a lower and upper bound on $\mu$. 

%\andre{Writing $\frac{\gamma}{\gamma-1} := 1/(\alpha-1)$, which is equivalent to $\alpha = 2 - 1/\gamma$,  we recognize the condition \eqref{eq:normmax_normalization_condition} as the normalization condition for $\alpha$-entmax with logits $\bm{\theta} / (\alpha-1)$ (cf.~\citet[\S 3.3]{peters2019sparse}). Therefore, we have $\tilde{\bm{y}} := (\bm{\theta} - \mu \mathbf{1})_+^{\gamma / (\gamma - 1)} = \alpha\text{-entmax}(\bm{\theta} / (\alpha-1))$, and $\bm{y} := \text{normmax}(\bm{\theta}) = \tilde{\bm{y}}^{1/\gamma} / \sum_i \tilde{y}_i^{1/\gamma} = \tilde{\bm{y}}^{2-\alpha} / \sum_i \tilde{y}_i^{2-\alpha}$.}
%>>> gamma = 3.4
%>>> alpha = 2-1/gamma
%>>> normmax_bisect(x, alpha=gamma)
%tensor([[0.0000, 0.4445, 0.0000, 0.0584, 0.4971]])
%>>> p = entmax_bisect(x/(alpha-1), alpha=alpha)
%>>> p**(2-alpha) / torch.sum(p**(2-alpha))
%tensor([[0.0000, 0.4445, 0.0000, 0.0584, 0.4971]])

%Since $\Omega$ is a generalized negentropy, it is maximized by the one-hot distribution and minimized by the uniform distribution. Therefore, for any $\bm{y} \in \triangle_N$, we have 
%\begin{align}\label{eq:bound_norm_entropy}
%N^{(1-\alpha)/\alpha} \le \|\bm{y}\|_\alpha \le 1.
%\end{align}
We have, from \eqref{eq:normmax_solution}, that 
$\mu = \theta_i - (y_i / \|\bm{y}\|_\gamma)^{\gamma-1}$ for any $i \in \mathrm{supp}(\bm{y})$. 
Letting $\theta_{\max} = \max_i \theta_i$ and $y_{\max} = \max_i y_i$, we have in particular that 
$\mu = \theta_{\max} - (y_{\max} / \|\bm{y}\|_\gamma)^{\gamma-1}$. 
We also have that $y_{\max} = \|\bm{y}\|_\infty \le \|\bm{y}\|_\gamma$, which implies $y_{\max} / \|\bm{y}\|_\gamma \le 1$. 
Since $1/N \le y_{\max} \le 1$ and $\|\bm{y}\|_\gamma \le 1$ for any $\bm{y} \in \triangle_N$,  
%combining with \eqref{eq:bound_norm_entropy}, 
we also obtain 
%$y_{\max} / \|\bm{y}\|_\alpha \le N^{(\alpha-1)/\alpha}$. 
%We also have that 
$y_{\max} / \|\bm{y}\|_\gamma \ge (1/N) / 1 = N^{-1}$. 
Therefore we have
\begin{align}
\underbrace{\theta_{\max} - 1}_{\mu_{\min}} \, \le \, \mu \, \le \, \underbrace{\theta_{\max} - N^{1-\gamma}}_{\mu_{\max}}.
\end{align}
The resulting algorithm is shown as Algorithm~\ref{alg:normmax_bisection}. 

\begin{algorithm}[t]
\small
\caption{Compute $\gamma$-normmax by bisection.}\label{alg:normmax_bisection}
%\small
\begin{algorithmic}[1]
\State \textbf{Input:} Scores $\bm{\theta} = [\theta_1, ..., \theta_N]^\top \in \mathbb{R}^N$, parameter $\gamma > 1$, number of bisection iterations $T$
\State \textbf{Output:} Probability vector $\bm{y} = [y_1, ..., y_N]^\top \in \triangle_N$. 
\State Define $\theta_{\max} \leftarrow \max_i \theta_i$
\State Set  $\mu_{\min} \leftarrow \theta_{\max} - 1$ and $\mu_{\max} \leftarrow \theta_{\max} - N^{1 - \gamma}$
\For{$t \in 1, \dots, T$}

\State  Set $\mu \leftarrow (\mu_{\min} + \mu_{\max}) / 2$ and  $Z \leftarrow \sum_j (\theta_j - \mu)_+^{\frac{\gamma}{\gamma-1}}$
\State \textbf{if} {$Z<1$}~\textbf{then}~$\mu_{\max}\leftarrow\mu$~%
\textbf{else}~$\mu_{\min}\leftarrow\mu$
\EndFor 
\State Return $\bm{y} = [y_1, ..., y_N]^\top$ with 
$y_i = {(\theta_i - \mu)_+^{\frac{1}{\gamma-1}}} / {\sum_{j} (\theta_j - \mu)_+^{\frac{1}{\gamma-1}}}.$
\label{line:return_normalize}
\end{algorithmic}
\label{algo:bisect}
\end{algorithm}

\section{Proofs of Main Text}

\subsection{Proof of Proposition~\ref{prop:cccp}}
\label{app:general_cccp}
Recall that the energy is written as a difference of two Fenchel-Young losses:
\begin{align}
    E(\bm{q}) 
    &= \underbrace{-L_\Omega(\bm{X}\bm{q}, \bm{u})}_{E_{\mathrm{concave}}(\bm{q})} + \underbrace{L_\Psi(\bm{X}^\top \bm{u}, \bm{q})}_{E_{\mathrm{convex}}(\bm{q})} + \mathrm{constant}.
\end{align}
The CCCP algorithm works as follows: at the $t$\textsuperscript{th} iteration, it linearizes the concave function $E_{\mathrm{concave}}$ by using a first-order Taylor approximation around $\bm{q}^{(t)}$, $${E}_{\mathrm{concave}}(\bm{q}) \approx \tilde{E}_{\mathrm{concave}}(\bm{q}) := {E}_{\mathrm{concave}}(\bm{q}^{(t)}) + \left(\frac{\partial E_{\mathrm{concave}}(\bm{q}^{(t)})}{\partial \bm{q}}\right)^\top (\bm{q} - \bm{q}^{(t)}).$$
Then, it computes a new iterate by solving the convex optimization problem $\bm{q}^{(t+1)} := \arg\min_{\bm{q}} E_{\mathrm{convex}}(\bm{q}) + \tilde{E}_{\mathrm{concave}}(\bm{q})$, which leads to $\nabla E_\mathrm{convex}(\bm{q}^{(t+1)}) = -\nabla E_\mathrm{concave}(\bm{q}^{(t)})$. Using the fact, from Proposition~\ref{prop:properties}, that $\nabla L_\Omega(\bm{\theta}, \bm{y}) = \hat{\bm{y}}_\Omega(\bm{\theta}) - \bm{y}$ and the chain rule leads to
\begin{align}\label{eq:energy_gradients}
\nabla E_{\mathrm{concave}}(\bm{q}) &= - \nabla_{\bm{q}} L_\Omega(\bm{X}\bm{q}; \bm{u}) = \bm{X}^\top\bm{u} - \bm{X}^\top\hat{\bm{y}}_\Omega(\bm{X}\bm{q}), \nonumber\\
\nabla E_{\mathrm{convex}}(\bm{q}) &= - \bm{X}^\top\bm{u} + \nabla \Psi(\bm{q}), 
\end{align}
leading to the equation $\nabla \Psi(\bm{q}^{(t+1)}) = \bm{X}^\top\hat{\bm{y}}_\Omega(\bm{X}\bm{q}^{(t)})$. 
Using the property that \(\nabla \Psi(\bm{q}) = \bm{\eta}\) is equivalent to \(\bm{q} = \nabla \Psi^*(\bm{\eta})\), \textit{i.e.}, that $(\nabla \Psi)^{-1} = \nabla \Psi^*$, we finally obtain:
\begin{align}
    %\nabla \Psi(\bm{q}^{(t+1)}) &= \bm{X}^\top\hat{\bm{y}}_\Omega(\bm{X}\bm{q}^{(t)}) \nonumber \\ 
    \bm{q}^{(t+1)} &= \nabla \Psi^* \left ( \bm{X}^\top\hat{\bm{y}}_\Omega(\bm{X}\bm{q}^{(t)}) \right) = \hat{\bm{y}}_\Psi \left ( \bm{X}^\top\hat{\bm{y}}_\Omega(\bm{X}\bm{q}^{(t)}) \right) ,
\end{align}
which leads to the update equation \eqref{eq:updates_general_hfy}.

\subsection{Proof of Proposition~\ref{prop:normalization}}\label{sec:proof_prop_normalization}

Let $\Psi(\bm{q}) = I_S(\bm{q})$ with $S := \{ \bm{q}  \,\,:\,\, \|\bm{q} - \bm{\delta}\| \le \eta \sqrt{D} \,\, \wedge \,\,  \bm{1}^\top (\bm{q} - \bm{\delta})=0 \}.$ 
We start by showing that, 
if $f(\bm{q}) := I_F(\bm{q})$ with $F := \left\{\|\bm{q}\| \le 1 \,\, \wedge \,\, \mathbf{1}^\top \bm{q} = 0\right\}$, we have $(\nabla f^*)(\bm{z}) = \frac{\bm{z} - \mu_{\bm{z}} \mathbf{1}}{\|\bm{z} - \mu_{\bm{z}}\|}$. 
By definition, we have $f^*(\bm{z}) = \max_{\bm{q}} \bm{z}^\top \bm{q} = -\min_{\bm{q}} -\bm{z}^\top \bm{q}$ subject to $\mathbf{1}^\top \bm{q} = 0$ and $\|\bm{q}\| \le 1$. 
Introducing Lagrange multipliers $\mu$ and $\lambda \ge 0$, we obtain the Lagrangian function $L(\bm{q}, \mu, \lambda) = -\bm{z}^\top \bm{q} + \mu \mathbf{1}^\top \bm{q} + \lambda (\|\bm{q}\| - 1)$. 
We have 
\begin{align}\label{eq:proof_normalization}
    \mathbf{0} = \nabla L(\bm{q}, \mu, \lambda) = -\bm{z} + \mu \mathbf{1} + \lambda {\bm{q}}/{\|\bm{q}\|},   
\end{align}
which implies 
$0 = -\mathbf{1}^\top \bm{z} + D\mu + \lambda {\mathbf{1}^\top \bm{q}}/{\|\bm{q}\|}$. Since we must have $\mathbf{1}^\top \bm{q} = 0$, this implies that $\mu = \frac{\mathbf{1}^\top \bm{z}}{D} = \mu_{\bm{z}}$. 
Therefore, we can write \eqref{eq:proof_normalization} as 
$\bm{z} - \mu_{\bm{z}} \mathbf{1} = \lambda \frac{\bm{q}}{\|\bm{q}\|}$. 
Taking the norm in both sides, 
we obtain $|\lambda| = \|\bm{z} - \mu_{\bm{z}} \mathbf{1}\|$; since $\lambda \ge 0$, we have $\lambda = \|\bm{z} - \mu_{\bm{z}} \mathbf{1}\|$. 
Next, we observe that, while $\bm{q}$ is constrained as  $\|\bm{q}\| \le 1$, the objective $\bm{z}^\top \bm{q}$ is maximized when  $\|\bm{q}\| = 1$, and therefore we obtain $\bm{q}^\star = (\nabla f^*)(\bm{z}) = \frac{\bm{z} - \mu_{\bm{z}} \mathbf{1}}{\|\bm{z} - \mu_{\bm{z}}\|}$. 
We also obtain 
\begin{align}
f^*(\bm{z}) = \bm{z}^\top \bm{q}^\star = \frac{\|\bm{z}\|^2 - d\mu_{\bm{z}}}{\|\bm{z} - \mu_{\bm{z}}\|} = \|\bm{z} - \mu_{\bm{z}}\|.
\end{align}
Now, consider $g(\bm{q}) := I_G(\bm{q})$ with $G := \{\|\bm{q}\| \le r \,\, \wedge \,\, \mathbf{1}^\top \bm{q} = 0\}$ for some $r > 0$. We can write 
$g(\bm{q}) = f(\bm{q} / r)$, and using the linear transformation property in Table~\ref{table:convex_conjugate_properties}, we have $g^*(\bm{z}) = f^*(r \bm{z})$, and therefore 
$(\nabla g^*)(\bm{z}) = r (\nabla f^*)(r \bm{z}) = r \frac{\bm{z} - \mu_{\bm{z}} \mathbf{1}}{\|\bm{z} - \mu_{\bm{z}}\|}$. 
When $r = \eta \sqrt{D}$ this becomes $(\nabla g^*)(\bm{z}) = \eta \frac{\bm{z} - \mu_{\bm{z}} \mathbf{1}}{\sigma_{\bm{z}}}$, where $\sigma_{\bm{z}} := \sqrt{\frac{1}{D} \sum_i (z_i - \mu_{\bm{z}})^2}$. 
Finally, observe that we can write $\Psi(\bm{q}) = g(\bm{q} - \bm{\delta})$. 
From the translation property in Table~\ref{table:convex_conjugate_properties}, we then have 
$\Psi^*(\bm{z}) = \bm{\delta}^\top \bm{z} + g^*(\bm{z})$. 
This leads to 
$(\nabla \Psi^*)(\bm{z}) = \bm{\delta} + (\nabla g^*)(\bm{z}) = \mathrm{LayerNorm}(\bm{z}; \eta, \bm{\delta})$.

\subsection{Proof of Proposition~\ref{prop:bounds_cccp}}\label{sec:proof_prop_bounds_cccp}

We start by proving that $E(\bm{q}) \ge 0$.
    We show first that for any $\Omega$ satisfying conditions 1--3 above, we have
    \begin{align}\label{eq:upper_bound_fy}
    L_\Omega(\bm{\theta}; \mathbf{1}/N) \le \max_i \theta_i -\mathbf{1}^\top \bm{\theta}/N.
    \end{align}
    From the definition of $\Omega^*$ and the fact that $\Omega(\bm{y}) \ge \Omega(\mathbf{1}/N)$ for any $\bm{y} \in \triangle_{N}$,
%    \citep[Proposition~4]{blondel2020learning},
    we have that, for any $\bm{\theta}$, $\Omega^*(\bm{\theta}) = \max_{\bm{y} \in \triangle_{N}} \bm{\theta}^\top \bm{y} - \Omega(\bm{y}) \le \max_{\bm{y} \in \triangle_{N}} \bm{\theta}^\top \bm{y} - \Omega(\mathbf{1}/N) = \max_i \theta_i - \Omega(\mathbf{1}/N)$,
    which leads to \eqref{eq:upper_bound_fy}.

    Let now $k = \arg\max_i \bm{q}^\top \bm{x}_i$, i.e., $\bm{x}_k$ is the pattern most similar to the query $\bm{q}$.
    %Invoking Lemma~\ref{lemma:E2_bound} with $\bm{p} = \bm{e}_k$ and since $\Omega(\bm{e}_k) = 0$, we get $E_2(\bm{q}) \le -\bm{x}_k^\top \bm{q}$.
    %From the definition of $\Omega^*$ and the fact that $\Omega(\bm{p}) \ge \Omega(\mathbf{1}/N)$ for any $\bm{p} \in \triangle_{N}$, we have that, for any $\bm{z}$, $\Omega^*(\bm{z}) = \max_{\bm{p} \in \triangle_{N}} \bm{z}^\top \bm{p} - \Omega(\bm{p}) \le \max_{\bm{p} \in \triangle_{N}} \bm{z}^\top \bm{p} - \Omega(\mathbf{1}/N) = \max_i z_i - \Omega(\mathbf{1}/N)$, and therefore we have $E_2(\bm{q}) = -\beta^{-1}\Omega^*(\beta \bm{X}\bm{q}) \ge -\bm{x}_k^\top \bm{q} + \beta^{-1} \Omega(\mathbf{1}/N)$.
    We have
    \begin{align*}
        E(\bm{q}) &= -\beta^{-1} L_\Omega(\beta \bm{X} \bm{q}; \mathbf{1}/{N}) + \frac{1}{2} \|\bm{q} - \bm{\mu}_{\bm{X}}\|^2 + \frac{1}{2}(M^2 - \|\bm{\mu}_{\bm{X}}\|^2)\\
        &\ge
        -\beta^{-1} (\beta\max_i \bm{q}^\top \bm{x}_i - \beta \mathbf{1}^\top \bm{X}\bm{q}/N) + \frac{1}{2} \|\bm{q} - \bm{\mu}_{\bm{X}}\|^2 + \frac{1}{2}(M^2 - \|\bm{\mu}_{\bm{X}}\|^2)\\
        &= -\bm{q}^\top \bm{x}_k + \bm{q}^\top \bm{\mu}_{\bm{X}} + \frac{1}{2} \|\bm{q} - \bm{\mu}_{\bm{X}}\|^2 + \frac{1}{2}(M^2 - \|\bm{\mu}_{\bm{X}}\|^2)\\
        &= -\bm{q}^\top \bm{x}_k + \frac{1}{2} \|\bm{q}\|^2 + \frac{1}{2}\underbrace{M^2}_{\ge \|\bm{x}_k\|^2} %\\
        %&\ge -\bm{q}^\top \bm{x}_k + \frac{1}{2} \|\bm{q}\|^2 + \frac{1}{2}\|\bm{x}_k\|^2 \\
        %&
        \ge \frac{1}{2}\|\bm{x}_k - \bm{q}\|^2 \ge 0.
    \end{align*}
    The zero value of energy is attained when $\bm{X} = \mathbf{1} \bm{q}^\top$ (all patterns are equal to the query), in which case $\bm{\mu}_{\bm{X}} = \bm{q}$, $M = \|\bm{q}\| = \|\bm{\mu}_{\bm{X}}\|$,  and we get $E_{\mathrm{convex}}(\bm{q}) = E_{\mathrm{concave}}(\bm{q}) = 0$.

    Now we prove the two upper bounds.
    For that, note that, for any $\bm{y} \in \triangle_{N}$, we have $0 \le L_{\Omega}(\bm{\theta}, \bm{y}) = L_{\Omega}(\bm{\theta}, \mathbf{1}/N) - \Omega(\mathbf{1}/N) + \Omega(\bm{y}) - (\bm{y} - \mathbf{1}/N)^\top\bm{\theta} \le L_{\Omega}(\bm{\theta}, \mathbf{1}/N) - \Omega(\mathbf{1}/N) - (\bm{y} - \mathbf{1}/N)^\top\bm{\theta},$ due to the assumptions 1--3 which ensure $\Omega$ is non-positive. That is, $L_\Omega(\bm{\theta}, \mathbf{1}/N) \ge \Omega(\mathbf{1}/N) + (\bm{y} - \mathbf{1}/N)^\top\bm{\theta}.$
    Therefore,
    with $\bm{q} = \bm{X}^\top \bm{y}$,
    we get $$E_{\mathrm{concave}}(\bm{q}) \le -\beta^{-1}\Omega(\mathbf{1}/N) - \bm{y}^\top \bm{X} \bm{q} + \bm{q}^\top \bm{\mu}_{\bm{X}} = -\beta^{-1}\Omega(\mathbf{1}/N) - \|\bm{q}\|^2 + \bm{q}^\top \bm{\mu}_{\bm{X}},$$ and
    $E(\bm{q}) = E_{\mathrm{concave}}(\bm{q}) + E_{\mathrm{convex}}(\bm{q}) \le -\beta^{-1}\Omega(\mathbf{1}/N) - \|\bm{q}\|^2 + \bm{q}^\top \bm{\mu}_{\bm{X}} + \frac{1}{2} \|\bm{q} - \bm{\mu}_{\bm{X}}\|^2 + \frac{1}{2}(M^2 - \|\bm{\mu}_{\bm{X}}\|^2) = -\beta^{-1}\Omega(\mathbf{1}/N) -\frac{1}{2}\|\bm{q}\|^2 + \frac{1}{2}M^2 \le -\beta^{-1}\Omega(\mathbf{1}/N) + \frac{1}{2}M^2$.

    To show the second upper bound, use the fact that $E_\mathrm{concave}(\bm{q}) \le 0$, which leads to $E(\bm{q}) \le E_\mathrm{convex}(\bm{q}) = \frac{1}{2}\|\bm{q} - \bm{\mu}_{\bm{X}}\|^2 + \frac{1}{2}(M^2 - \|\bm{\mu}_{\bm{X}}\|^2) = \frac{1}{2}\|\bm{q}\|^2 - \bm{q}^\top\bm{\mu}_{\bm{X}} + \frac{1}{2}M^2$.
    Note that $\|\bm{q}\| = \|\bm{X}^\top \bm{y}\| \le \sum_i y_i \|x_i\| \le M$ and that, from the Cauchy-Schwarz inequality, we have $- \bm{q}^\top\bm{\mu}_{\bm{X}} \le \|\bm{\mu}_{\bm{X}}\| \|\bm{q}\| \le M^2$. Therefore, we obtain $E(\bm{q}) \le \frac{1}{2}\|\bm{q}\|^2 - \bm{q}^\top\bm{\mu}_{\bm{X}} + \frac{1}{2}M^2 \le \frac{1}{2}M^2 +M^2 + \frac{1}{2}M^2 = 2M^2$.

%\section{Proof of Proposition~\ref{prop:updates}}

%We now turn to the update rule.
%The CCCP algorithm works as follows: at the $t$\textsuperscript{th} iteration, it linearizes the concave function $E_{\mathrm{concave}}$ by using a first-order Taylor approximation around $\bm{q}^{(t)}$, $${E}_{\mathrm{concave}}(\bm{q}) \approx \tilde{E}_{\mathrm{concave}}(\bm{q}) := {E}_{\mathrm{concave}}(\bm{q}^{(t)}) + \left(\frac{\partial E_{\mathrm{concave}}(\bm{q}^{(t)})}{\partial \bm{q}}\right)^\top (\bm{q} - \bm{q}^{(t)}).$$
%Then, it computes a new iterate by solving the convex optimization problem $\bm{q}^{(t+1)} := \arg\min_{\bm{q}} E_{\mathrm{convex}}(\bm{q}) + \tilde{E}_{\mathrm{concave}}(\bm{q})$, which leads to the equation $\nabla E_\mathrm{convex}(\bm{q}^{(t+1)}) = -\nabla E_\mathrm{concave}(\bm{q}^{(t)})$. Using the Fenchel-Young losses properties from Proposition \ref{prop:properties} and the chain rule leads to
%\begin{align}\label{eq:energy_gradients}
%\nabla E_{\mathrm{concave}}(\bm{q}) &= -\beta^{-1} \nabla_{\bm{q}} L_\Omega(\beta\bm{X}\bm{q}; \mathbf{1}/N) = \bm{X}^\top (\mathbf{1}/N - \hat{\bm{y}}_\Omega(\beta\bm{X}\bm{q})) \nonumber\\
%&= \bm{\mu}_{\bm{X}} - \bm{X}^\top\hat{\bm{y}}_\Omega(\beta\bm{X}\bm{q})\nonumber\\
%\nabla E_{\mathrm{convex}}(\bm{q}) &= \bm{q} - \bm{\mu}_{\bm{X}},
%\end{align}
%giving the update equation \eqref{eq:updates_hfy}.

\subsection{Proof of Proposition~\ref{prop:separation}}\label{sec:proof_prop_separation}

A stationary point is a solution of the equation $-\nabla E_{\mathrm{concave}}(\bm{q}) = \nabla E_{\mathrm{convex}}(\bm{q})$.
Using the expression for gradients \eqref{eq:energy_gradients},  this is equivalent to $\bm{q} = \bm{X}^\top \hat{\bm{y}}_\Omega(\beta\bm{X}\bm{q})$. If $\bm{x}_i = \bm{X}^\top \bm{e}_i$ is not a convex combination of the other memory patterns,
\(\bm{x}_i\) is a stationary point iff $\hat{\bm{y}}_\Omega(\beta\bm{X}\bm{x}_i)= \bm{e}_i$.
We now use the margin property of sparse transformations \eqref{eq:margin}, according to which the latter is equivalent to $\beta\bm{x}_i^\top \bm{x}_i - \max_{j\ne i} \beta\bm{x}_i^\top \bm{x}_j \ge m$. 
Noting that the left hand side equals $\beta\Delta_i$ leads to the desired result. 

If the initial query satisfies ${\bm{q}^{(0)}}^\top (\bm{x}_i - \bm{x_j}) \ge \frac{m}{\beta}$ for all $j \ne i$, we have again from the margin property that $\hat{\bm{y}}_\Omega(\beta\bm{X}\bm{q}^{(0)})= \bm{e}_i$, which combined to the previous claim ensures convergence in one step to $\bm{x}_i$. 
Finally, note that, 
if $\bm{q}^{(0)}$ is $\epsilon$-close to $\bm{x}_i$, we have 
$\bm{q}^{(0)} = \bm{x}_i + \epsilon \bm{r}$ for some vector $\bm{r}$ with $\|\bm{r}\|=1$. 
Therefore, we have 
\begin{align*}
(\bm{q}^{(0)})^\top (\bm{x}_i - \bm{x}_j) &= (\bm{x}_i + \epsilon \bm{r})^\top (\bm{x}_i - \bm{x}_j) %\nonumber\\
%&
\ge \Delta_i + \epsilon \bm{r}^\top (\bm{x}_i - \bm{x}_j) %\nonumber\\
%&
\ge \Delta_i - \epsilon \underbrace{\|\bm{r}\|}_{=1}  \|\bm{x}_i - \bm{x}_j\|,
\end{align*}
where we invoked the Cauchy-Schwarz inequality in the last step. 
Since the patterns are normalized (with norm $M$),%
\footnote{In fact, the result still holds if patterns are not normalized but have their norm upper bounded by $M$, i.e., if they lie within a ball of radius $M$ and not necessarily on the sphere.} %
we have from the triangle inequality that $\|\bm{x}_i - \bm{x}_j\| \le \|\bm{x}_i\| + \|\bm{x}_j\| = 2M$; using the assumption that $\Delta_i \ge \frac{m}{\beta} + 2M\epsilon$, we obtain ${\bm{q}^{(0)}}^\top (\bm{x}_i - \bm{x}_j) \ge \frac{m}{\beta}$, which from the previous points ensures convergence to $\bm{x}_i$ in one iteration.

\subsection{Proof of Proposition~\ref{prop:storage}}\label{sec:proof_prop_storage}

For the first statement, we follow a similar argument as the one made by \citet{ramsauer2020hopfield} in the proof of their Theorem A.3---however their proof has a mistake, which we correct here.%
\footnote{Concretely, \citet{ramsauer2020hopfield} claim that given a separation angle $\alpha_{\min}$, we can place $N = \left({2\pi}/{\alpha_{\min}}\right)^{D-1}$ patterns equidistant on the sphere, but this is not correct.} %
Given a separation angle $\alpha_{\min}$, we lower bound the number of patterns $N$ we can place in the sphere separated by at least this angle. 
Estimating this quantity is an important open problem in combinatorics, related to determining the size of spherical codes (of which kissing numbers are a particular case; \citealt{conway2013sphere}). We invoke a lower bound due to \citet{chabauty1953resultats}, \citet{shannon1959probability}, and \citet{wyner1965capabilities} (see also \citet{jenssen2018kissing} for a tighter bound), who show that $N \ge (1 + o(1))\sqrt{2\pi D} \frac{\cos \alpha_{\min}}{(\sin \alpha_{\min})^{D-1}}$. 
For $\alpha_{\min} = \frac{\pi}{3}$, which corresponds to the kissing number problem, we obtain the bound $N \ge (1 + o(1))\sqrt{{3\pi D}/{8}}\left({2}/{\sqrt{3}}\right)^D = \mathcal{O}\left(\left({2}/{\sqrt{3}}\right)^D\right).$
In this scenario, we have $\Delta_i = M^2(1 - \cos \alpha_{\min})$ by the definition of $\Delta_i$. 
From Proposition~\ref{prop:separation}, we have exact retrieval under $\epsilon$-perturbations if $\Delta_i \ge m\beta^{-1} + 2M\epsilon$. 
Combining the two expressions, we obtain $\epsilon \le \frac{M}{2}(1 - \cos \alpha_{\min}) - \frac{m}{2\beta M}$. 
Setting $\alpha_{\min} = \frac{\pi}{3}$, we obtain
$\epsilon \le \frac{M}{2}\left(1 - \frac{1}{2}\right) - \frac{m}{2\beta M} = \frac{M}{4} - \frac{m}{2\beta M}$. 
For the right hand side to be positive, we must have $M^2 > 2m/\beta$.

Assume now patterns are placed uniformly at random in the sphere. From \citet{brauchart2018random} we have, for any $\delta>0$:  
\begin{align}
    P(N^{\frac{2}{D-1}}\alpha_{\min} \ge \delta) \ge 1 - \frac{\kappa_{D-1}}{2}\delta^{D-1}, \quad \text{with} \quad
    \kappa_{D} := \frac{1}{D\sqrt{\pi}}\frac{\Gamma((D+1)/2)}{\Gamma(D/2)}.
\end{align}
Given our failure probability $p$, we need to have 
%\begin{align}
    $P(M^2 (1 - \cos \alpha_{\min}) \ge m\beta^{-1} + 2M\epsilon) \ge 1-p$,
%\end{align}
which is equivalent to
\begin{align}
P \left\{ N^{\frac{2}{D-1}} \alpha_{\min} \ge N^{\frac{2}{D-1}} \arccos \left( 1 - \frac{m}{\beta M^2} - \frac{2\epsilon}{M} \right)\right\} \ge 1-p.
\end{align}
Therefore, we set 
$p = \frac{\kappa_{D-1}}{2} N^2 \left[ \arccos \left( 1 - \frac{m}{\beta M^2} - \frac{2\epsilon}{M} \right)\right]^{D-1}$. 
%\begin{align}
%    p = \frac{\kappa_{D-1}}{2} N^2 \left[ \arccos \left( 1 - \frac{m}{\beta M^2} - \frac{2\epsilon}{M} \right)\right]^{D-1}.
%\end{align}
Choosing $N =  \sqrt{\frac{2p}{\kappa_{D-1}}} \zeta^{\frac{D-1}{2}}$ patterns for some $\zeta > 1$, we obtain 
$1 = \left[\zeta \arccos \left( 1 - \frac{m}{\beta M^2} - \frac{2\epsilon}{M} \right)\right]^{D-1}$. 
%\begin{align}
%    1 = \left[\zeta \arccos \left( 1 - \frac{m}{\beta M^2} - \frac{2\epsilon}{M} \right)\right]^{D-1}. 
%\end{align}
%which leads to 
%\begin{equation}
%    \zeta = \frac{1}{\arccos \left( 1 - \frac{m}{M^2} - \frac{2\epsilon}{M} \right)}.
%\end{equation}
Therefore the failure rate $p$ is attainable provided the perturbation error is 
\begin{equation}
    \epsilon \le  \frac{M}{2} \left(1 - \cos \frac{1}{\zeta}\right) - \frac{m}{2\beta M}.
\end{equation}
For the right hand side to be positive, we must have
$\cos \frac{1}{\zeta} < 1 - \frac{m}{\beta M^2}$, i.e., $\zeta < \frac{1}{\arccos \left(1 -  \frac{m}{\beta M^2}\right)}$. 

\subsection{Proof of Proposition~\ref{prop:sparsemap_margin}}\label{sec:proof_prop_sparsemap_margin}

The first statement in the proposition is stated and proved by \citet{blondel2020learning} as a corollary of their Proposition 8. 
We prove here a more general version, which includes the second statement as a novel result. 
Using \citet[Proposition 8]{blondel2020learning}, we have that the structured margin of $L_\Omega$ is given by the following expression,
\begin{align*}
    m = \sup_{\bm{y} \in \mathcal{Y}, \,\, \bm{\mu} \in \mathrm{conv}(\mathcal{Y})} \frac{\Omega(\bm{y}) - \Omega(\bm{\mu})}{r^2 - \bm{\mu}^\top \bm{y}},
\end{align*}
%\begin{align*}
%    m = \sup_{\substack{\bm{y} \in \mathcal{Y}\\\bm{\mu} \in \mathrm{conv}(\mathcal{Y})}} \frac{\Omega(\bm{y}) - \Omega(\bm{\mu})}{r^2 - \bm{\mu}^\top \bm{y}},
%\end{align*}
if the supremum exists. 
For SparseMAP, using $\Omega(\bm{\mu}) = \frac{1}{2}\|\bm{\mu}_V\|^2 = \frac{1}{2}\|\bm{\mu}\|^2 - \frac{1}{2}\|\bm{\mu}_F\|^2$ for any $\bm{\mu} \in \mathrm{conv}(\mathcal{Y})$, and using the fact that $\|\bm{y}\| = r$ for any $\bm{y} \in \mathcal{Y}$, we obtain:
\begin{align*}
    m &= \sup_{\bm{y} \in \mathcal{Y}, \,\,\bm{\mu} \in \mathrm{conv}(\mathcal{Y})} \frac{\frac{1}{2}r^2 - \frac{1}{2}\|\bm{\mu}\|^2 +  \frac{1}{2}\|\bm{\mu}_F\|^2 - \frac{1}{2}r_F^2}{\bm{y}^\top (\bm{y} - \bm{\mu})} 
    \,\, \le^{(\dagger)} \sup_{\bm{y} \in \mathcal{Y}, \,\,\bm{\mu} \in \mathrm{conv}(\mathcal{Y})} \frac{\frac{1}{2}r^2 - \frac{1}{2}\|\bm{\mu}\|^2}{\bm{y}^\top (\bm{y} - \bm{\mu})}\nonumber\\
    &= 1 - \inf_{\bm{y} \in \mathcal{Y},\,\,\bm{\mu} \in \mathrm{conv}(\mathcal{Y})} \frac{\frac{1}{2}\|\bm{y} - \bm{\mu}\|^2}{\bm{y}^\top (\bm{y} - \bm{\mu})} \,\, \le^{(\ddagger)} 1,
\end{align*}
%\begin{align*}
%    m &= \sup_{\substack{\bm{y} \in \mathcal{Y}\\\bm{\mu} \in \mathrm{conv}(\mathcal{Y})}} \frac{\frac{1}{2}r^2 - \frac{1}{2}\|\bm{\mu}\|^2 +  \overbrace{\frac{1}{2}\|\bm{\mu}_F\|^2 - \frac{1}{2}r_F^2}^{\le 0}}{\bm{y}^\top (\bm{y} - \bm{\mu})} \nonumber\\
%    &
%    \le^{(\dagger)} \sup_{\substack{\bm{y} \in \mathcal{Y}\\\bm{\mu} \in \mathrm{conv}(\mathcal{Y})}} \frac{\frac{1}{2}r^2 - \frac{1}{2}\|\bm{\mu}\|^2}{\bm{y}^\top (\bm{y} - \bm{\mu})}\nonumber\\
%    &= 1 - \inf_{\substack{\bm{y} \in \mathcal{Y}\\\bm{\mu} \in \mathrm{conv}(\mathcal{Y})}} \frac{\frac{1}{2}\|\bm{y} - \bm{\mu}\|^2}{\bm{y}^\top (\bm{y} - \bm{\mu})}\nonumber\\
%    &\le^{(\ddagger)} 1,
%\end{align*}
where the inequality $^{(\dagger)}$ follows from the convexity of $\frac{1}{2}\|\cdot\|^2$, which implies that $\frac{1}{2}\|\bm{\mu}_F\|^2 \le \frac{1}{2}\|\bm{y}_F\|^2 = \frac{1}{2}r_F^2$; and the inequality $^{(\ddagger)}$ follows from the fact that both the numerator and denominator in the second term are non-negative, the latter due to the Cauchy-Schwartz inequality and the fact that $\|\bm{\mu}\| \le r$. 
This proves the second part of Proposition~\ref{prop:sparsemap_margin}. 

To prove the first part, note first that,  
if there are no higher order interactions, then $r_F = 0$ and $\bm{\mu}_F$ is an ``empty vector'', which implies that $^{(\dagger)}$ is an equality. 
We prove now that, in this case, $^{(\ddagger)}$ is also an equality, which implies that $m=1$. 
We do that by showing that, 
for any $\bm{y} \in \mathcal{Y}$, we have $\inf_{\bm{\mu} \in \mathrm{conv}(\mathcal{Y})} \frac{\frac{1}{2}\|\bm{y}-\bm{\mu}\|^2}{\bm{y}^\top (\bm{y} - \bm{\mu})} = 0$. 
Indeed, choosing $\bm{\mu}= t\bm{y}' + (1-t)\bm{y}$ for an arbitrary $\bm{y}' \in \mathcal{Y} \setminus \{\bm{y}\}$, and letting $t\rightarrow 0^+$, we obtain $\frac{\frac{1}{2}\|\bm{y}-\bm{\mu}\|^2}{\bm{y}^\top (\bm{y} - \bm{\mu})} = \frac{\frac{t}{2}\|\bm{y}-\bm{y}'\|^2}{\bm{y}^\top (\bm{y} - \bm{y}')} \rightarrow 0$.

\subsection{Proof of Proposition~\ref{prop:stationary_single_iteration_sparsemap}}\label{sec:proof_prop_stationary_single_iteration_sparsemap}

A point $\bm{q}$ is stationary iff it satisfies $\bm{q} = \bm{X}^\top \hat{\bm{y}}_\Omega(\beta \bm{X}\bm{q})$.  
Therefore,   
$\bm{X}^\top \bm{y}_i$ is guaranteed to be a stationary point if 
$\hat{\bm{y}}_\Omega(\beta \bm{X}\bm{X}^\top \bm{y}_i) = \bm{y}_i$.%
\footnote{But not necessarily ``only if''---we could have $\bm{X}^\top\bm{y}_i$ in the convex hull of the other pattern associations.} % 
By assumption, we have $\beta \Delta_i \ge \frac{1}{2}D_i^2 \ge \frac{1}{2}\|\bm{y}_i - \bm{y}_j\|^2$ for all $j$. 
Since $\Delta_i \le \bm{y}_i^\top \bm{X}\bm{X}^\top (\bm{y}_i - \bm{y}_j)$ by definition, this implies $\beta \bm{y}_i^\top \bm{X}\bm{X}^\top (\bm{y}_i - \bm{y}_j) \ge \frac{1}{2}\|\bm{y}_i - \bm{y}_j\|^2$. 
Since SparseMAP has a margin $m \le 1$, we  recognize that the latter inequality is a margin condition (Def.~\ref{def:structured_margin}), which implies zero loss, \textit{i.e.}, $\hat{\bm{y}}_\Omega(\beta \bm{X}\bm{X}^\top \bm{y}_i) = \bm{y}_i$, as desired. 

%A point $\bm{q}$ is stationary iff it satisfies $\bm{q} = \bm{X}^\top \hat{\bm{y}}_\Omega(\beta \bm{X}\bm{q})$.  
%Therefore,   
%$\bm{X}^\top \bm{y}_i$ is guaranteed to be a stationary point if%
%\footnote{But not necessarily ``only if''---we could have $\bm{X}^\top\bm{y}_i$ in the convex hull of the other pattern associations.} %
%$\hat{\bm{y}}_\Omega(\beta \bm{X}\bm{X}^\top \bm{y}_i) = \bm{y}_i$, which is equivalent to zero loss, i.e., to the existence of a margin 
%$\underbrace{\beta \bm{y}_i^\top \bm{X}\bm{X}^\top (\bm{y}_i - \bm{y}_j)}_{\ge \beta \Delta_i} \ge \frac{1}{2}\|\bm{y}_i - \bm{y}_j\|^2$ for all $j$. 
%Since we are assuming $\Delta_i \ge \frac{D_i^2}{2\beta} \ge \frac{\|\bm{y}_i - \bm{y}_j\|^2}{2\beta}$ for all $j$, we have $\beta \Delta_i \ge \frac{\|\bm{y}_i - \bm{y}_j\|^2}{2}$ for all $j$, which implies the margin condition above. 

If the initial query satisfies $\bm{q}^\top 
\bm{X}^\top(\bm{y}_i - \bm{y}_j) \ge \frac{D_i^2}{2\beta}$ for all $j \ne i$, we have again from the margin property that $\hat{\bm{y}}_\Omega(\beta\bm{X}\bm{q})= \bm{y}_i$, which ensures convergence in one step to $\bm{X}^\top\bm{y}_i$. 

If $\bm{q}$ is $\epsilon$-close to $\bm{X}^\top\bm{y}_i$, then  
$\bm{q} = \bm{X}^\top\bm{y}_i + \epsilon \bm{r}$ for some vector $\bm{r}$ with $\|\bm{r}\|=1$. 
Therefore, 
\begin{align}
\bm{q}^\top \bm{X}^\top (\bm{y}_i - \bm{y}_j) &= (\bm{X}^\top\bm{y}_i + \epsilon \bm{r})^\top \bm{X}^\top (\bm{y}_i - \bm{y}_j) \ge \Delta_i + \epsilon \bm{r}^\top \bm{X}^\top (\bm{y}_i - \bm{y}_j).  \end{align}
We now bound $-\bm{r}^\top \bm{X}^\top (\bm{y}_i - \bm{y}_j)$ in two possible ways. 
Using the Cauchy-Schwarz inequality, 
we have $-\bm{r}^\top \bm{X}^\top (\bm{y}_i - \bm{y}_j) \le \|\bm{X} \bm{r}\| \|\bm{y}_i - \bm{y}_j\| \le \sigma_{\max}(\bm{X}) D_i$, where $\sigma_{\max}(\bm{X})$ is the largest singular value of $\bm{X}$ (its  spectral norm). 
On the other hand, denoting 
$R_i := \max_j \|\bm{y}_i - \bm{y}_j\|_1$, we can also use H\"older's inequality to obtain $-\bm{r}^\top \bm{X}^\top (\bm{y}_i - \bm{y}_j) \le \|\bm{X} \bm{r}\|_\infty \|\bm{y}_i - \bm{y}_j\|_1 \le M R_i$, where we used the fact that 
$\|\bm{X} \bm{r}\|_\infty = \max_k |\bm{x}_k^\top \bm{r}| \le \|\bm{x}_k\|\|\bm{r}\| = M$. Combining the two inequalities, we obtain 
$\bm{q}^\top \bm{X}^\top (\bm{y}_i - \bm{y}_j) \ge \Delta_i - \epsilon \min \{\sigma_{\max}(\bm{X}) D_i, MR_i\}$. 
Using the assumption that $\Delta_i \ge \frac{D_i^2}{2\beta} + \epsilon\min \{\sigma_{\max}(\bm{X}) D_i, MR_i\}$, we obtain $\bm{q}^\top \bm{X}^\top (\bm{y}_i - \bm{y}_j) \ge \frac{D_i^2}{2\beta}$, which from the previous points ensures convergence to $\bm{X}^\top\bm{y}_i$ in one iteration. 
The result follows by noting that, since $\mathcal{Y} \subseteq \{0,1\}^D$, we have $R_i = D_i^2$.  

\section{Additional Experiments and Experimental Details}
\subsection{Memory Retrieval}
Further insights into Hopfield-Fenchel-Young network variants are provided in Figure \ref{fig:2}. The figure reveals that normalization excels across a variable number of memories.
\label{app:MR}
\begin{figure}[h]
    \centering
    \includegraphics[width=1\textwidth]{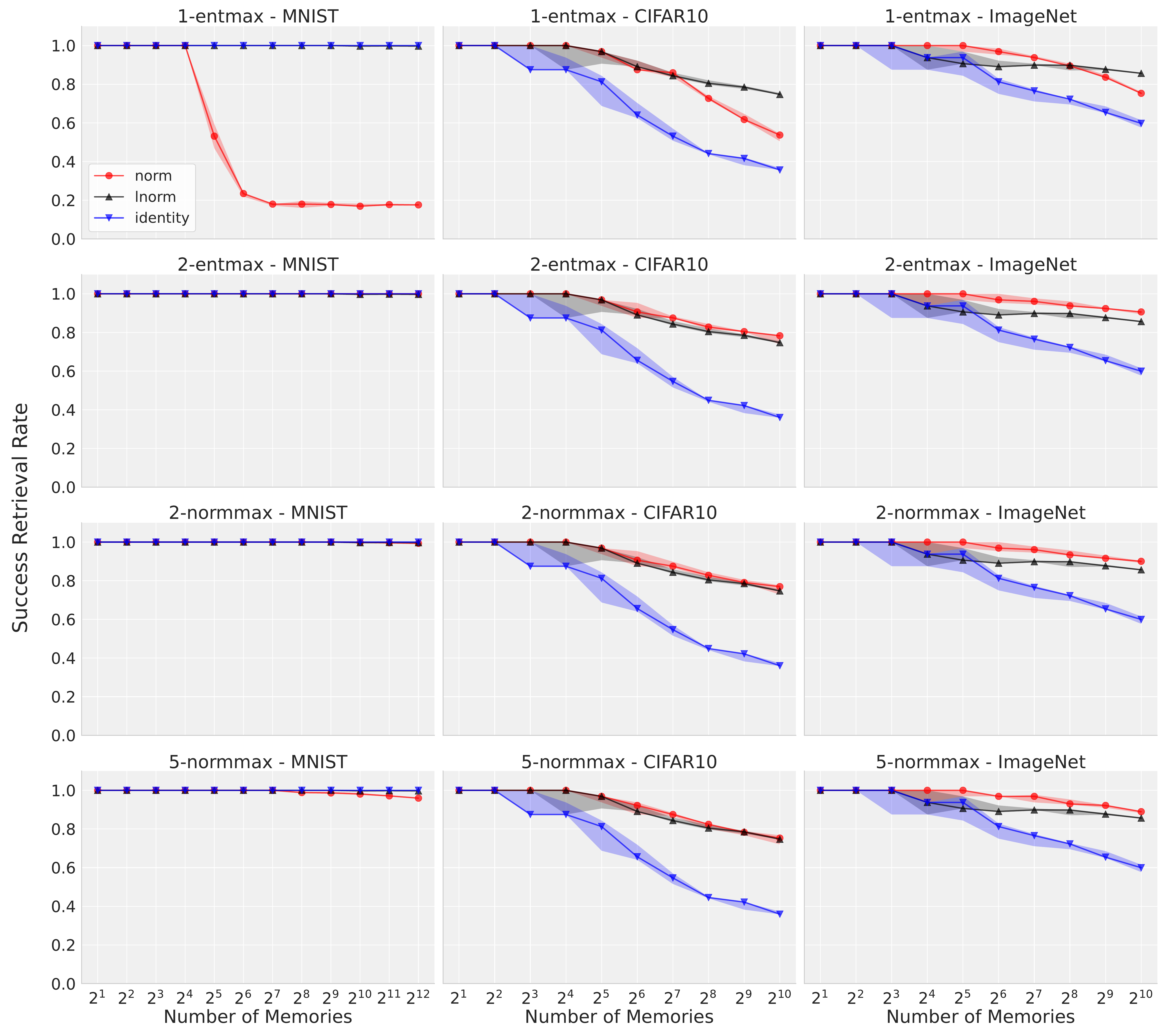}
    \caption{Memory capacity for different numbers of stored memories for $\beta = 1$ and for different $\hat{\bm{y}}_\Omega$ and $\hat{\bm{y}}_\Psi$. }
    \label{fig:2}
\end{figure}
%\begin{figure*}[h]
%    \centering
%    \includegraphics[width=1\textwidth]{figure4.pdf}
%    \caption{Memory robustness against different levels of noise for $\beta = 1$ and for different $\hat{\bm{y}}_\Omega$ and $\hat{\bm{y}}_\Psi$.}
%    \label{fig:4}
%\end{figure*}
%\section{Experimental Details}
\subsection{MNIST $K$-MIL}
\label{sec:MNIST_Experimental_Details}
For $K$-MIL, we created 4 datasets by grouping the MNIST examples into bags, for $K \in \{1, 2, 3, 5\}$. A bag is positive if it contains at least $K$ targets, where the target is the number ``9'' (we chose ``9'' as it can be easily misunderstood with ``7'' or ``4''). The embedding architecture is the same as \citet{ilse2018attention}, but instead of attention-based pooling, we use our $\alpha$-entmax pooling, with $\alpha=1$ mirroring the pooling method in \cite{ramsauer2020hopfield}, and $\alpha=2$ corresponding to the pooling in \cite{hu2023sparse}. Additionally, we incorporate $\alpha$-normmax pooling and SparseMAP pooling with $k$-subsets. Further details of the $K$-MIL datasets are shown in Table~\ref{tab:MNIST_dataset}.
\begin{table}[t]
\small
    \centering
    \caption{Dataset sample details for the MNIST $K$-MIL experiment. The size $L_i$ of the $i\textsuperscript{th}$ bag is determined through $L_i = \max\{K, L_i'\}$ where $L_i' \sim \mathcal{N}(\mu, \sigma^2)$.  %Here, $\mu$ represents the bag size mean and $\sigma$ the bag size variance. We define the bag size as the maximum between $K$ and the number generated by a normal distribution with the aforementioned parameters. 
    The number of positive instances in a bag is uniformly sampled between $K$ and $L_i$ for positive bags and between $0$ and $K-1$ for negative bags.}
    \label{tab:MNIST_dataset}
    \begin{tabular}{cccccc}
        \toprule
        {\textbf{Dataset}} & {$\boldsymbol{\mu}$} & {$\boldsymbol{\sigma}$} & {\textbf{Features}} & {\textbf{Pos. training bags}} & {\textbf{Neg. training bags}} \\
        \midrule
        MNIST, $K=1$ & 10 & 1 & 28 $\times$ 28 & 1000 & 1000 \\
        MNIST, $K=2$ & 11 & 2 & 28 $\times$ 28 & 1000 & 1000 \\
        MNIST, $K=3$ & 12 & 3 & 28 $\times$ 28 & 1000 & 1000\\
        MNIST, $K=5$ & 14 & 5 & 28 $\times$ 28 & 1000 & 1000 \\
        \bottomrule
    \end{tabular}
\end{table}

We train the models for 5 different random seeds, where the first one is used for tuning the hyperparameters. The reported test accuracies represent the average across these seeds. We use 500 bags for testing and 500 bags for validation. The hyperparameters are tuned via grid search, where the grid space is shown in Table \ref{tab:MNIST_hyperparam}. We consider a dropout hyperparameter, commonly referred to as bag dropout, to the Hopfield matrix due to the risk of overfitting (as done by \citet{ramsauer2020hopfield}). 
All models were trained for 50 epochs. We incorporated an early-stopping mechanism, with patience 5, that selects the optimal checkpoint based on performance on the validation set. 

\begin{table}[t]
\small
    \centering
    \caption{Hyperparameter space for the MNIST MIL experiment. Hidden size is the dimension of keys and queries and $\gamma$ is a parameter of the exponential learning rate scheduler \citep{Li2020An}.}
    \label{tab:MNIST_hyperparam}
    \begin{tabular}{ l  l }
        \toprule
        {\textbf{Parameter}} & {\textbf{Range}} \\
        \midrule
        learning rate & \{$10^{-5}$, $10^{-6}$\} \\
        $\gamma$ & \{0.98 , 0.96\}  \\
        hidden size & \{16, 64\} \\
        number of heads & \{8, 16\} \\
        $\beta$ & \{0.25, 0.5, 1.0, 2.0, 4.0, 8.0\} \\
        bag dropout & \{0.0, 0.75\} \\
        \bottomrule
    \end{tabular}
\end{table}

\subsection{MIL benchmarks}
\label{sec:MIL_bench_details}
The MIL benchmark datasets (Fox, Tiger and Elephant) comprise preprocessed and segmented color images sourced from the Corel dataset \cite{ilse2018attention}. Each image is composed of distinct segments or blobs, each defined by descriptors such as color, texture, and shape. The datasets include 100 positive and 100 negative example images, with the negative ones randomly selected from a pool of photos featuring various other animals.

The $\mathrm{HopfieldPooling}$ layers ($\alpha$-entmax; $\alpha$-normmax; SparseMAP, $k$-subsets) take as input a collection of embedded instances, along with a trainable yet constant query. This query pattern is used for the purpose of averaging class-indicative instances, thereby facilitating the compression of bags of variable sizes into a consistent representation. This compression is important for effectively discriminating between different bags. To fine-tune the model, a manual hyperparameter search was conducted on a validation set.

In our approach to tasks involving Elephant, Fox and Tiger, we followed a similar architecture as \citep{ramsauer2020hopfield}:
\begin{enumerate}
    \item The first two layers are fully connected linear embedding layers with ReLU activation.
    \item The output of the second layer serves as the input for the $\mathrm{HopfieldPooling}$  layer, where the pooling operation is executed.
    \item Subsequently, we employ a single layer as the final linear output layer for classification with a sigmoid as the classifier.
\end{enumerate}
\begin{table}[t]
\small
    \centering
    \caption{Hyperparameter space for the MIL benchmark experiments. Hidden size is the space in which keys and queries are associated and $\gamma$ is a parameter of the exponential learning rate scheduler.}
    \label{tab:fox_hyperparam}
    \begin{tabular}{ l  l }
        \toprule
        {\textbf{Parameter}} & {\textbf{Range}} \\
        \midrule
        learning rate & \{$10^{-3}$, $10^{-5}$\} \\
        $\gamma$ & \{0.98 , 0.96\}  \\
        embedding dimensions & \{32 , 128\}  \\
        embedding layers & \{2\}  \\
        hidden size & \{32, 64\} \\
        number of heads & \{12\} \\
        $\beta$ & \{0.1, 1, 10\} \\
        bag dropout & \{0.0, 0.75\} \\
        \bottomrule
    \end{tabular}
\end{table}
During the hyperparameter search, various configurations were tested, including different hidden layer widths and learning rates. Particular attention was given to the hyperparameters of the $\mathrm{HopfieldPooling}$ layers, such as the number of heads, head dimension, and the inverse temperature $\beta$. To avoid overfitting, bag dropout (dropout at the attention weights) was implemented as the chosen regularization technique. All hyperparameters tested are shown in Table \ref{tab:fox_hyperparam}.

We trained for 50 epochs with early stopping with patience 5, using the Adam optimizer \cite{loshchilov2017decoupled} with exponential learning rate decay. Model validation was conducted through a 10-fold nested cross-validation, repeated five times with different data splits where the first seed is used for hyperparameter tuning. The reported test ROC AUC scores represent the average across these repetitions.

\vskip 0.2in
\setlength{\bibsep}{3.2pt}
\bibliography{bib}

\end{document}